\documentclass{article}

\usepackage{microtype}
\usepackage{graphicx}
\usepackage{booktabs} 

\usepackage{hyperref}
\usepackage{url}
\usepackage{booktabs}       
\usepackage{amsfonts,amsmath}       
\usepackage{nicefrac}       
\usepackage{microtype}      

\usepackage{more_style}
\usepackage{comment}
\usepackage{subcaption}

\renewcommand{\algorithmiccomment}[1]{#1}

\usepackage[accepted]{icml2021}

\icmltitlerunning{REPAINT: Knowledge Transfer in Deep Reinforcement Learning}

\begin{document}

\twocolumn[
\icmltitle{REPAINT: Knowledge Transfer in Deep Reinforcement Learning}



\icmlsetsymbol{equal}{*}

\begin{icmlauthorlist}
\icmlauthor{Yunzhe Tao}{1}
\icmlauthor{Sahika Genc}{1}
\icmlauthor{Jonathan Chung}{1}
\icmlauthor{Tao Sun}{1}
\icmlauthor{Sunil Mallya}{1}
\end{icmlauthorlist}

\icmlaffiliation{1}{AI Labs, Amazon Web Services, Seattle, WA 98121, USA}

\icmlcorrespondingauthor{Yunzhe Tao}{yunzhe.tao@gmail.com}

\icmlkeywords{Reinforcement Learning, transfer learning, MuJoCo, DeepRacer, StarCraft}

\vskip 0.3in
]



\printAffiliationsAndNotice{}  

\begin{abstract}
Accelerating learning processes for complex tasks by leveraging previously learned tasks has been one of the most challenging problems in reinforcement learning, especially when the similarity between source and target tasks is low. This work proposes REPresentation And INstance Transfer (REPAINT) algorithm for knowledge transfer in deep reinforcement learning. REPAINT not only transfers the representation of a pre-trained teacher policy in the on-policy learning, but also uses an advantage-based experience selection approach to transfer useful samples collected following the teacher policy in the off-policy learning. Our experimental results on several benchmark tasks show that REPAINT significantly reduces the total training time in generic cases of task similarity. In particular, when the source tasks are dissimilar to, or sub-tasks of, the target tasks, REPAINT outperforms other baselines in both training-time reduction and asymptotic performance of return scores.
\end{abstract}

\section{Introduction}
\label{sec:intro}
In the past few years, deep reinforcement learning (RL) has become more ubiquitous in solving sequential decision-making problems for many real-world applications, such as game playing \citep{openai2019dota,silver2016mastering}, robotics \citep{kober2013reinforcement,openai2018learning}, and autonomous driving \citep{sallab2017deep}. However, most RL methods train an agent from scratch, typically requiring a huge amount of time and computing resources. Accelerating the learning processes for complex tasks has been one of the most challenging problems in RL \citep{kaelbling1996reinforcement, sutton2018reinforcement}. The computational cost of learning grows as the task complexity increases in the real-world applications. Therefore, it is desirable for a learning algorithm to leverage knowledge acquired in one task to improve performance on other tasks.

Transfer learning has achieved significant success in computer vision, natural language processing, and other knowledge engineering areas \citep{pan2009survey}. In transfer learning, the source (teacher) and target (student) tasks are not necessarily drawn from the same distribution \citep{taylor2008transferring}. The unseen target task may be a simple task which is similar to the previously trained tasks, or a complex task with traits borrowed from significantly different source tasks. Despite the prevalence of direct weight transfer, knowledge transfer from pre-trained agents for RL tasks has not been gaining much attention until recently \citep{barreto2019transfer,ma2018universal,schmitt2018kickstarting,lazaric2012transfer,taylor2009transfer}. However, many transfer RL algorithms are designed to select similar tasks or samples from a set of source tasks, or learn representations of source tasks. Hence they perform well only when the target tasks are similar to the source tasks, but are usually not helpful when the task similarity is low or the target tasks are much more complex than the source tasks.



In this work, we propose an algorithm, i.e., REPresentation And INstance Transfer (REPAINT), to address the aforementioned problem. The algorithm introduces an off-policy instance transfer learning and combines it with an on-policy representation transfer. The main contributions of this paper are as follows. 
(1) We develop an advantage-based experience selection approach in the off-policy instance transfer, which helps improve the sample efficiency by only transferring the useful instances. 
(2) The REPAINT algorithm is simple to implement and can be naturally extended to any policy gradient-based RL algorithms. In addition, we also provide two variants of REPAINT for actor-critic RL and an extension to Q-learning. 
(3) We clarify that our REPAINT algorithm exploits the (semantic) relatedness between the source samples and target tasks, instead of the task/sample similarities that most transfer RL methods exploit.
(4) On several transfer learning tasks, we empirically demonstrate that REPAINT significantly reduces the training time needed to achieve certain performance level in generic cases of task similarity. Moreover, when the source tasks are dissimilar to, or sub-tasks of, complex target tasks, REPAINT greatly outperforms other baseline methods in both training-time reduction and asymptotic return scores.

\section{Related Work: Transfer Learning in RL}
This section only introduces the related work on transfer learning for RL. We will discuss the connection between our proposed algorithm and other related work in Section~\ref{sec:algorithm}. In transfer learning for RL, most algorithms either assume specific forms of reward functions or perform well only when the teacher and student tasks are similar. Additionally, very few algorithms are designated to actor-critic RL.

Transfer learning algorithms in RL can be characterized by the definition of transferred knowledge, which contains the \textit{parameters} of the RL algorithm, the \textit{representation} of the trained policy, and the \textit{instances} collected from the environment \citep{lazaric2012transfer}. When the teacher and student tasks share the same state-action space and they are considered similar \citep{ferns2004metrics,phillips2006knowledge}, \textit{parameter transfer} is the most straightforward approach, namely, one can initialize the policy or value network in the student tasks by that from teacher tasks \citep{mehta2008transfer,rajendran2015attend}. Parameter transfer with different state-action variables is more complex, where the crucial aspect is to find a suitable mapping from the teacher state-action space to the student state-action space \citep{gupta2017learning,talvitie2007experts,taylor2008autonomous}.

Many transfer learning algorithms fall into the category of \textit{representation transfer}, where the algorithm learns a specific representation of the task or the solution, and the transfer algorithm performs an abstraction process to fit it into the student task. \citet{konidaris2012transfer} uses the \textit{reward shaping} approach to learn a portable shaping function for knowledge transfer, while some other works use neural networks for feature abstraction \citep{duan2016rl,parisotto2015actor,zhang2018decoupling}. \textit{Policy distillation} \citep{rusu2015policy}, or its variants, is another popular choice for learning the teacher task representation, where the student policy aims to mimic the behavior of pre-trained teacher policies during its own learning process \citep{schmitt2018kickstarting,yin2017knowledge}. Recently, \textit{successor representation} has been widely used in transfer RL, in which the rewards are assumed to share some common features, 
so that the value function can be simply written as a linear combination of the \textit{successor features} (SFs) \citep{barreto2017successor,madarasz2019better}. \citet{barreto2019transfer} extends the method of using SFs and generalised policy improvement in Q-learning \citep{sutton2018reinforcement} to more general environments. \citet{borsa2018universal}, \citet{ma2018universal}, and \citet{schaul2015universal} learn a universal SF approximator for transfer.

The basic idea of \textit{instance transfer} algorithms is that the transfer of teacher samples may improve the learning on student tasks. \citet{lazaric2008transfer} and \citet{tirinzoni2018importance} selectively transfer samples on the basis of the compliance between tasks in a model-free algorithm, while \citet{taylor2008transferring} studies how a model-based algorithm can benefit from samples coming from the teacher task. 

In this work, we propose a representation-instance transfer algorithm to handle the generic cases of task similarity in RL. The algorithm is also naturally fitted for actor-critic framework and can be easily extended to other RL algorithms.


\section{Background: Actor-Critic RL}\label{sec:back}
A general RL agent interacting with environment can be modeled in a Markov decision process (MDP), which is defined by a tuple $M = (S, A, p, r, \gamma)$, where $S$ and $A$ are sets of states and actions, respectively. The state transfer function $p(\cdot|s, a)$ maps a state and action pair to a probability distribution over states. 
$r: S\times A\times S \to \mathbb{R}$ denotes the reward function that determines a reward received by the agent for a transition from $(s,a)$ to $s'$. The discount factor, $\gamma\in[0,1]$, provides means to obtain a long-term objective. Specifically, the goal of an RL agent is to learn a policy $\pi$ that maps a state to a probability distribution over actions at each time step $t$, so that $a_t \sim \pi(\cdot | s_t)$ maximizes the accumulated discounted return $\sum_{t\ge0} \gamma^t r(s_t, a_t, s_{t+1})$.

To address this problem, a popular choice to adopt is the model-free actor-critic architecture, e.g., \citet{konda2000actor,degris12:offpolicyAC,mnih2016asynchronous,schulman2015trust,schulman2017proximal}, where the critic estimates the value function and the actor updates the policy distribution in the direction suggested by the critic. 
The actor-critic methods usually rely on the \textit{advantage function}, which is computed by
$    A^\pi(s,a) = Q^\pi(s, a)-V^\pi(s)\,,$
where $Q^\pi(s, a):=\mathbb{E}_{a_i\sim\pi(\cdot|s_i)}\left[ \sum_{i\ge t} \gamma^{i-t} r(s_i, a_i, s_{i+1}) | s, a \right]$ is the Q (action value) function and $V^\pi(s):=\mathbb{E}_\pi \left[ \sum_{i\ge t} \gamma^{i-t} r_{i+1} | s_t=s \right]$ is the state value function.

Intuitively, the advantage can be taken as the extra reward that could be obtained by taking a particular action $a$. 
In deep RL, the critic and actor functions are usually parameterized by neural networks. Then the policy gradient methods can be used to update the actor network. For example, in the clipped proximal policy optimization (Clipped PPO) \citep{schulman2017proximal}, the policy's objective function is defined to be the minimum between the standard surrogate objective and an $\epsilon$ clipped objective:
\begin{equation}\label{eq:clipped-ppo}
\footnotesize
    L_\text{clip}(\theta)=\hat{\mathbb{E}}_t \left[ \min\left(\ell_\theta(s_t,a_t)\cdot \hat{A}_t, \text{clip}_\epsilon\left(\ell_\theta(s_t,a_t)\right)\cdot \hat{A}_t\right)\right]\,,
\end{equation}
where the policy $\pi$ is parameterized by $\theta$, $\hat{A}_t$ are the advantage estimates, and $\ell_\theta(\cdot,\cdot)$ is the likelihood ratio that
$$
    \ell_\theta(s_t,a_t) = \frac{\pi_\theta(a_t|s_t)}{\pi_{\theta_\text{old}}(a_t|s_t)}\,.
$$
Additionally, the function $\text{clip}_\epsilon$ truncates $\ell_\theta(\cdot,\cdot)$ to the range of $(1-\epsilon, 1+\epsilon)$.

\section{The REPAINT Algorithm}\label{sec:algorithm}
We now describe our knowledge transfer algorithm, i.e., REPAINT, for actor-critic RL framework, which is provided in Algorithm~\ref{alg:repaint}. 
Without loss of generality, we demonstrate the policy update using Clipped PPO, and use a single teacher policy in the knowledge transfer. In practice, it can be directly applied to any policy gradient-based RL algorithms, and it is straightforward to have multiple teacher policies in transfer. More discussion can be found later in this section.

In REPAINT with actor-critic RL, the critic update uses traditional supervised regression, which is exactly the same as Clipped PPO. However, there are two core concepts underlying our actor update, i.e., on-policy representation transfer learning and off-policy instance transfer learning. The on-policy representation transfer employs a policy distillation approach \citep{schmitt2018kickstarting}. 
In the off-policy instance transfer, we update the actor by an off-policy objective $L_\text{ins}$ and using an \emph{advantage-based experience selection} on $\tilde{\mathcal{S}}$, the replay buffer for teacher instances. The proposed experience selection approach is used to select samples that have high \emph{semantic relatedness}, instead of high similarity, to the target task. We defer the discussion to Section~\ref{sec:relatedness}.


\subsection{On-policy Representation Transfer: Kickstarting} 
In order to boost the initial performance of an agent, we use the policy distillation approach adopted in a kickstarting training pipeline \citep{schmitt2018kickstarting,rusu2015policy} for on-policy representation transfer. 
The main idea is to employ an auxiliary loss function which encourages the student policy to be close to the teacher policy on the trajectories sampled by the student. Given a teacher policy $\pi_\text{teacher}$, we introduce the auxiliary loss as
$
    L_\text{aux}(\theta) = H\left(\pi_\text{teacher}(a|s) \| \pi_\theta(a|s)\right),
$
where $H(\cdot\|\cdot)$ is the cross-entropy. 
Then the policy distillation adds the above loss to the Clipped PPO objective function, i.e., \eqref{eq:clipped-ppo}, weighted at optimization iteration $k$ by the scaling $\beta_k\ge 0$:
\begin{equation}\label{eq:rl-loss}
    L_\text{rep}^k(\theta) = L_\text{clip}(\theta) - \beta_k L_\text{aux}(\theta)\,.
\end{equation}
In our experiments, the weighting parameter $\beta_k$ is relatively large at early iterations, and vanishes as $k$ increases, which is expected to improve the initial performance of the agent while keeping it focused on the current task in later epochs.


\begin{algorithm}[!t]
  \caption{REPAINT with Clipped PPO}
  \label{alg:repaint}
\begin{algorithmic}
  \STATE Initialize $\nu$, $\theta$, and load teacher policy $\pi_{\text{teacher}}(\cdot)$
  \STATE Set hyper-parameters $\zeta$, $\alpha_1$, $\alpha_2$, and $\beta_k$ in \eqref{eq:rl-loss}
  \FOR{iteration $k=1,2,\ldots$}
  \STATE Set $\theta_\text{old} \leftarrow \theta$
  \STATE Collect samples $\mathcal{S}=\{(s,a,s',r)\}$ using $\pi_{\theta_\text{old}}(\cdot)$
  \STATE Collect samples $\tilde{\mathcal{S}}=\{(\tilde{s},\tilde{a},\tilde{s}',\tilde{r})\}$ using $\pi_{\text{teacher}}(\cdot)$
  \STATE Fit state-value network $V_\nu$ using only $\mathcal{S}$ to update $\nu$
  \STATE Compute advantage estimates $\hat{A}_1,\ldots,\hat{A}_T$ for $\mathcal{S}$ and $\hat{A}'_1,\ldots,\hat{A}'_{T'}$ for $\tilde{\mathcal{S}}$ 
  \FOR[\quad\quad// \emph{experience selection}]{t=1,\ldots,$T'$}
    \IF{$\hat{A}'_t<\zeta$} 
        \STATE Remove $\hat{A}'_t$ and the corresponding transition $(\tilde{s}_t,\tilde{a}_t,\tilde{s}_{t+1},\tilde{r}_t)$ from $\tilde{\mathcal{S}}$
    \ENDIF
  \ENDFOR
  \STATE Compute sample gradient of $L_\text{rep}^k(\theta)$ in \eqref{eq:rl-loss} using $\mathcal{S}$
  \STATE Compute sample gradient of $L_\text{ins}(\theta)$ in \eqref{eq:obj-ins} using $\tilde{\mathcal{S}}$
  \STATE Update policy network by $$\theta \leftarrow \theta + \alpha_1\nabla_\theta L_\text{rep}^k(\theta) + \alpha_2\nabla_\theta L_\text{ins}(\theta)$$ 
  \ENDFOR
\end{algorithmic}
\end{algorithm}

\subsection{Off-policy Instance Transfer: Advantage-based Experience Selection} 
Note that the kickstarting aims to replicate the behavior of teacher policy in the early training stage, so that it can improve the agent's initial performance. However, when the target task is very different from the source task, kickstarting usually does not lead to much improvement. To address this, we now propose the off-policy instance transfer with an approach called advantage-based experience selection.

In the off-policy instance transfer, we form a replay buffer $\tilde{\mathcal{S}}$ by collecting training samples following the teacher policy $\pi_\text{teacher}$, but compute the rewards using current reward function from the target task. Since the samples are obtained from a different distribution, we do not use those samples to update the state value (critic) network. In order to improve the sample efficiency, when updating the policy (actor) network, we select the transitions based on the advantage values and only use the samples that have advantages greater than a given threshold $\zeta$. Moreover, since the teacher policy has been used in collecting roll-outs, we compute the objective without the auxiliary cross-entropy loss, but replace $\pi_{\theta_\text{old}}$ with $\pi_\text{teacher}$ in \eqref{eq:clipped-ppo} for off-policy learning, which leads to the following objective function:
\begin{equation}\label{eq:obj-ins}
\footnotesize
    L_\text{ins}(\theta)=\hat{\mathbb{E}}_t \left[ \min\left(\rho_\theta(s_t,a_t)\cdot \hat{A}'_t, \text{clip}_\epsilon\left(\rho_\theta(s_t,a_t)\right)\cdot \hat{A}'_t\right)\right]\,,
\end{equation}
where $\rho_\theta$ now is given by
$
    \rho_\theta(s_t,a_t)=\frac{\pi_\theta(a_t|s_t)}{\pi_\text{teacher}(a_t|s_t)}\,.
$
The idea of advantage-based experience selection is simple but effective. As mentioned before, the advantage can be viewed as the extra reward that could be obtained by taking a particular action. Therefore, since the advantages are computed under the reward function of the target task, the state-action transitions with high advantage values can be viewed as ``good'' samples for transfer, no matter how different the source and target tasks are. By retaining only the good samples from replay buffer $\tilde{\mathcal{S}}$, the agent can focus on learning useful behavior for current task, which as a result can improve the sample efficiency in knowledge transfer.

\vspace{-0.6em}
\paragraph{Related work on experience selection.}
We note that although the experience selection approach has not been used in knowledge transfer before, it is related to the prioritized experience replay (PER) \citep{schaul2015prioritized}, which prioritizes the transitions in replay buffer by the temporal-difference (TD) error, and utilizes importance sampling for the off-policy evaluation. In comparison, our method uses experience selection in the actor-critic framework, where the replay buffer is erased after each training iteration. Unlike the stochastic prioritization which imposes the probability of sampling, our method directly filters out most of the transitions from replay buffer and equally prioritizes the remaining data, which further improves the sample efficiency. In addition, our method performs policy update without importance sampling. All transitions in a batch have the same weight, since the importance can be reflected by their advantage values. The empirical comparison among PER and several experience selection approaches can be found in Section~\ref{sec:exp-mujoco}.

As another related work, Self-Imitation Learning (SIL) \citep{oh2018self} provides an off-policy actor-critic algorithm that learns to reproduce the agent's past good decisions. It selects past experience based on the gap between the off-policy Monte-Carlo return and the agent's value estimate, and only uses samples with positive gaps to update both the actor and the critic. Similar approach can also be seen in Q-filter \citep{nair2018overcoming}. In comparison, other than that REPAINT uses an advantage-based experience selection instead of the return gap, our method introduces a more general threshold $\zeta$, while SIL fixes it to be zero. Moreover, the positive-gap samples are also used to update the critic in SIL, but we only fit the state values by on-policy data for following the correct distribution. Motivated by the justification that SIL is equivalent to a lower-bound soft-Q-learning, we also present the extension of REPAINT with Q-learning in this paper, which is given in Section~\ref{sec:alg}.

Again, we want to remark that our proposed advantage-based experience selection has a different formulation with the existing methods. Moreover, to the best of our knowledge, it is the first time that the advantage-based filtering has been applied to the knowledge transfer for RL. 

\subsection{Discussion and Extension}
In regards to the proposed REPAINT algorithm, we want to remark that the policy distillation weight $\beta_k$ and advantage filtering threshold $\zeta$ are task specific. They are dependent with the one-step rewards. To this end, one can consider to normalize reward functions in practice, so that the one-step rewards are in the same scale. In general, larger $\beta_k$ encourages the agent to better match the teacher policy, and larger $\zeta$ leads to that fewer samples are kept for policy update, which in result makes current learning concentrate more on the high-advantage experience. The empirical investigations on these two parameters can be found later.

So far, we have demonstrated the REPAINT algorithm using only a single teacher policy in the knowledge transfer and the objective function from Clipped PPO. Indeed, it is straightforward to using multiple teacher policies, or even using different teacher policies for representation transfer and instance transfer. In addition, REPAINT can be directly applied to any policy gradient-based RL algorithms, such as A2C \cite{sutton2000policy}, A3C \cite{mnih2016asynchronous}, TRPO \cite{schulman2015trust} and REINFORCE \cite{williams1992simple}. In Section~\ref{sec:alg}, we will provide the discussion in more details. We also present an extension of the REPAINT algorithm with Q-learning there.

\section{Theoretical Justification and Analysis Related to REPAINT}
This section introduces some theoretical results and justifications that are related to the proposed algorithm, from which we hope to clarify REPAINT in greater depth. The detailed discussion of Theorem~\ref{thm:convergence} and Theorem~\ref{thm:rate} can be found in Sections~\ref{sec:conv-appendix} and \ref{sec:rate-appendix} from appendix, respectively.

\subsection{Convergence Results}
We first discuss the convergence of REPAINT without any experience selection approach, and then consider how the experience selection from teacher policy impacts the policy update. 
To simplify the illustration without loss of generality, we consider the objective for the actor-critic to be 
$J_\text{RL}(\theta)=\hat{\mathbb{E}}_t[Q^{\pi_\theta}(s,a)\log\nabla\pi_\theta(a|s)]$
instead of $L_\text{clip}(\theta)$, and modify the objectives for representation transfer and instance transfer accordingly (denoted by $J_\text{rep}$ and $J_\text{ins}$).
The convergence of representation transfer can be easily obtained, as it is equivalent to the convergence of other actor-critic methods. 
Our instance transfer learning fits into the framework of the off-policy actor-critic \citep{degris12:offpolicyAC,zhang2019generalized}. 
Following 
\citet{holzleitner2020convergence}, under certain commonly used assumptions, we can prove the convergence of the off-policy instance transfer. 

\begin{theorem}\label{thm:convergence}
Suppose the critic is updated by TD residuals, and the actor $\pi_\theta$ is updated based on the objective $J_\text{ins}(\theta)$. Fix a starting point $(\theta_0, \nu_0)$ and construct the associated neighborhood $W_0\times U_0$ as in Section~\ref{sec:conv-appendix}. Assume the loss functions satisfy Assumptions~\ref{ass:smooth1}-\ref{ass:smooth3}, and the learning rates for critic and actor satisfy Assumption~\ref{ass:lr}. Then $(\theta_n, \nu_n)$ converges to a local optimum almost surely as $n\to\infty$ via online stochastic gradient descent (SGD).
\end{theorem}
In addition, we also want to show the convergence rate for REPAINT, if a good approximation for the critic is available. Again, we assume that the the experience selection approach is not used. We also assume the learning rates $\alpha_1$ and $\alpha_2$ are iteration-dependent (denoted by $\alpha_{1,k}$ and $\alpha_{2,k}$).
Let $K_{\epsilon}$ be the smallest number of updates $k$ required to attain a function gradient smaller than $\epsilon$, i.e.,
$$K_{\epsilon} = \min \{k : \underset{0\leq m\leq k}{\inf} \mathcal{F}(\theta_m) < \epsilon \},$$
with $A_k := \alpha_{2, k}/\alpha_{1, k}$ and
\begin{equation*}
\begin{aligned}
    \mathcal{F}(\theta_m) = &\norm{\nabla J_\text{rep} (\theta_m)}^2 + A_k \norm{\nabla J_\text{ins} (\theta_m)}^2 \\ 
    &+ (1+A_k) \nabla J_\text{rep}^T (\theta_m) \nabla J_\text{ins} (\theta_m).
\end{aligned}
\end{equation*}

Note that the hyper-parameter $A_k$ can be determined by how much one wants to learn from instance transfer against the representation transfer. If $A$ is set to be 1, then we can get $\mathcal{F}(\theta_m)=\|\nabla J_\text{rep} (\theta_m) + \nabla J_\text{ins} (\theta_m)\|^2$.

\begin{theorem}\label{thm:rate}
Suppose the learning rate for representation transfer satisfies $\alpha_{1, k} = k^{-a}$ for $a>0$, and the critic update satisfies Assumption~\ref{assm:criticrate}. 
When the critic bias converges to zero as $\mathcal{O}(k^{-b})$ for some $b\in (0, 1]$, we can find an integer $T(b,k)$ such that $T(b,k)$ critic updates occur per actor update.
Then the actor sequence 
satisfies
\begin{equation*}
    K_{\epsilon} \leq \mathcal{O} (\epsilon^{-1/l}),\text{ where } l = \min\{a, 1-a, b\}.
\end{equation*}
\end{theorem}

Despite greatly improving the sample efficiency, the experience selection introduces bias by changing the distribution in an uncontrolled fashion. In practice, to mitigate it, we can adopt REPAINT in the early training stage, and then reduce to traditional actor-critic algorithms. As a consequence, the agent first learns useful teacher behavior to achieve good initial performance, and focuses on the target task afterwards.

\subsection{Semantic Relatedness \textit{vs.} Task Similarity}\label{sec:relatedness}
Most transfer learning algorithms for RL are built based on the \textit{similarity} between source and target tasks, and perform well only when the tasks are similar. 
If samples from two MDPs are given, some metrics can be defined to compute or learn the task similarity, e.g., the Kantorovich distance-based metric \citep{song2016measuring}, the restricted Boltzmann machine distance measure \citep{ammar2014automated}, the policy overlap \citep{carroll2005task}, and the task compliance and sample relevance \citep{lazaric2008transfer}. In general, the similarity is usually unknown before getting any samples, unless other information is given. For example, the methods using successor features \citep{barreto2019transfer,borsa2018universal,schaul2015universal} assume that the reward functions among tasks are a linear combination of some common features, 
namely, $r(s,a,s')=\sum_i w_i \phi_i(s,a,s')$ with fixed $\phi_i$'s. Then the similarity can be characterized by the distance of the weight vectors. 

In this paper, we aim to show that REPAINT handles generic cases of task similarity. Therefore, we do not use any similarity information during transfer. 
Instead, REPAINT falls into exploiting the \textit{relatedness} for knowledge transfer, a concept that has been used in some other machine learning areas, e.g., multi-task learning \citep{Caruana:1997} and meta-learning \citep{achille2019task2vec}.
The difference between relatedness and similarity in transfer learning can be analogous to the difference between semantic relatedness and lexicographical similarity in languages. More specifically, no matter how different the source and target tasks are, we can always select related samples from the source tasks that are useful for learning the target tasks. In REPAINT, the relatedness is just defined by the advantage values of source samples under target reward function and state values. Moreover, the cross-entropy weights $\beta_k$ and the experience selection threshold $\zeta$ are used to control the contribution from the source task. We will compare REPAINT against a similarity-based transfer learning algorithm \cite{lazaric2008transfer} in Section~\ref{sec:exp-mujoco}.


\section{Experiments}

According to \citet{taylor2009transfer}, the \emph{performance} of transfer learning (TL) can be measured by: (1) the improvement of the agent's initial performance when learning from a pre-trained policy; (2) the improvement of final performance and total accumulated reward after transfer; and (3) the reduction of training convergence time or the learning time needed by the agent to achieve a specified performance level. In this paper, we are particularly interested in the last metric, i.e., the training time reduction, since one cannot always expect the return score improvement, especially when the source tasks are very different from the target tasks.

In this section, we conduct experiments for answering following questions. (1) When the source (teacher) tasks are similar to the target (student) tasks, it is expected that most TL methods perform well. Can REPAINT also achieve good TL performance? (2) When the task similarity is low, can REPAINT still reduce the training time of the target tasks? (3) When the source tasks are only sub-tasks of the complex target tasks, is REPAINT still helpful? (4) Are both on-policy representation transfer and off-policy instance transfer necessary for REPAINT? (5) How do other experience selection (prioritization) approaches perform on REPAINT? (6) How do the hyper-parameters $\beta_k$ and $\zeta$ change the TL performance? Is REPAINT robust to them?

A metric to quantify the task similarity level is required for answering those questions. For simplicity, we assume in the experiments that the state and action spaces stay the same between teacher and student tasks, and the reward functions have the form of linear combination of some common features. Then we use the cosine distance function to define the task similarity, namely, the similarity between two tasks with reward functions $r_1(s,a,s')=\bm{\phi}(s,a,s')^\top \bm{w}_1$ and $r_2(s,a,s')=\bm{\phi}(s,a,s')^\top \bm{w}_2$ can be computed as
\begin{equation}\label{eq:similarity}
    \text{sim}(r_1,r_2) = \frac{\bm{w}_1 \cdot \bm{w}_2}{\|\bm{w}_1\|\|\bm{w}_2\|}\,.
\end{equation}
We say the two tasks are \textit{similar} if $\text{sim}(r_1,r_2)>0$. Otherwise ($\le 0$), they are considered to be \textit{different} (\textit{dissimilar}). In addition, if some feature weights are zero in a reward function, the corresponding task can be viewed as a \textit{sub-task} of other tasks that have non-zero feature weights.

\subsection{Experimental Setup}
To assess the REPAINT algorithm, we use three platforms across multiple benchmark tasks with increasing complexity for experiments, i.e., Reacher and Ant environments in MuJoCo simulator \citep{todorov2016mujoco}, single-car and multi-car racings in AWS DeepRacer simulator \citep{balaji2019deepracer}, and \textit{BuildMarines} and \textit{FindAndDefeatZerglings} mini-games in StarCraft II environments \cite{vinyals2017starcraft}. More detailed descriptions of the environments are given in Section~\ref{sec:env}. The first four questions mentioned above will be answered across all environments. We use simpler environments, i.e., MuJoCo-Reacher and DeepRacer single-car, to answer the last two questions, as it is easier to interpret the results without extra complexity and environment noise.

In order to compare the performance of REPAINT with other methods and demonstrate that REPAINT improves sample efficiency during transfer, we should guarantee that REPAINT does not use more samples for transfer in each iteration. Therefore, we employ an alternating REPAINT with Clipped PPO in experiments, where we adopt on-policy representation transfer and off-policy instance transfer alternately on odd and even numbered iterations. The algorithm is presented in Section~\ref{sec:alg} (Algorithm~\ref{alg:transfer-PPO}). The study of different alternating ratios has also been provided in Section~\ref{sec:exp-ratio}. In addition, one can find in Section~\ref{sec:env} the hyper-parameters we used for reproducing our results.



\begin{figure}[!t]
    \centering
    \includegraphics[width=.9\linewidth]{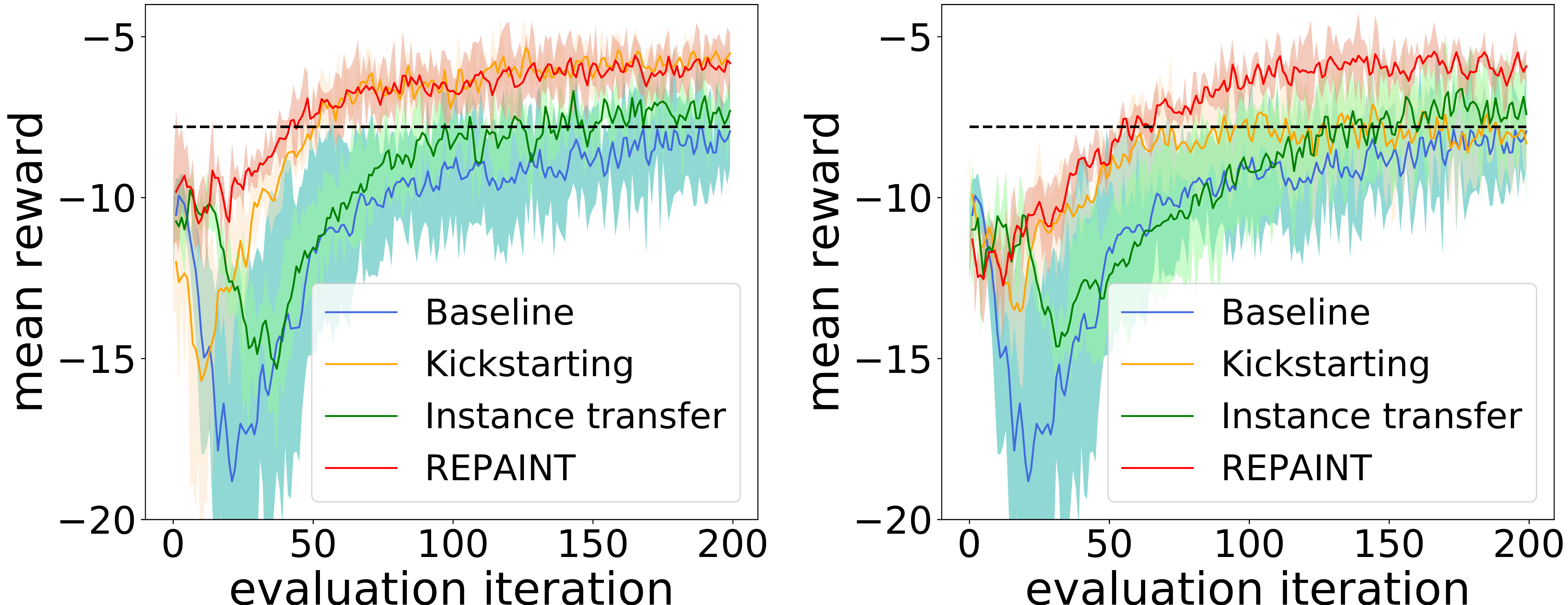}
    \caption{Evaluation performance for MuJoCo-Reacher, averaged across five runs. We consider both teacher task is similar to (left) and different from (right) the target task.}
    \label{fig:reacher1}
\end{figure}

\begin{figure}[!t]
    \centering
    \includegraphics[width=.9\linewidth]{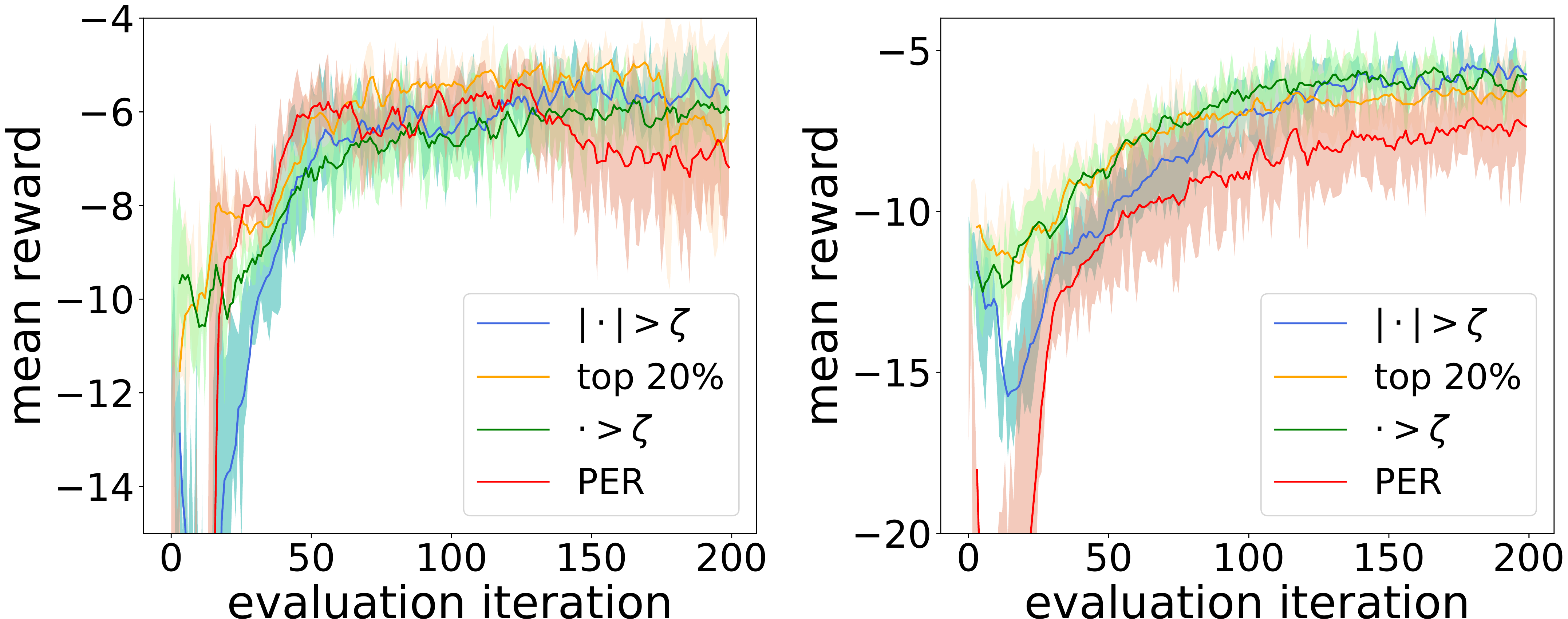}
    \caption{Comparison of different experience selection (prioritization) approaches. Left: similar tasks. Right: dissimilar tasks.}
    \label{fig:reacher2}
\end{figure}

\subsection{Continuous Action Control in MuJoCo}\label{sec:exp-mujoco}
\paragraph{MuJoCo-Reacher.} In the target task, the agent is rewarded by getting close to the goal point with less movement. As an ablation study, we first compare REPAINT against training with only kickstarting or instance transfer and with no prior knowledge (baseline), based on two teacher tasks. The first teacher policy is trained with similar reward function but a higher weight on the movement penalty, where we set it to be 3 as an example, so that the cosine similarity is positive. Another teacher policy is trained in a dissimilar task, where the agent is penalized when it is close to the goal. In this case, the cosine similarity is zero. After each training iteration, we evaluate the policy for another 20 episodes. The evaluation performance is presented in Figure~\ref{fig:reacher1}\footnote{The black dashed lines in here and other figures indicate the best return score of training from scratch (baseline), by which one can see how fast each method achieves a certain performance level. More quantitative results can be found in Table~\ref{table:summary}.}. REPAINT outperforms baseline algorithm and instance transfer in both cases of task similarity, regarding the training time reduction, asymptotic performance, and the initial performance boosting. Although kickstarting can improve the initial performance, it has no performance gain in convergence when the teacher behavior is opposed to the expected target behavior (see right sub-figure). In contrast, although the instance transfer does not boost the initial performance, it surpasses the baseline performance asymptotically in both cases. 

We also compare the performance of several experience selection rules in REPAINT, including high absolute values ($|\cdot|>\zeta$), top 20\% of transitions in ranking (top 20\%), our proposed rule ($\cdot >\zeta$), and PER \citep{schaul2015prioritized}. For PER, we used the advantage estimates to compute the prioritization instead of TD errors for a fair comparison, and slightly tuned the hyper-parameters. From Figure~\ref{fig:reacher2}, we can observe that the proposed selection rule and top 20\% rule perform better than others on the initial performance, where only the most related samples are selected for policy update. Moreover, PER does not work as well as other approaches, especially when the task similarity is low, since it includes low-advantage teacher samples which have no merits for the student policy to learn. Therefore, we suggest to use the proposed selection rule with a threshold $\zeta$ or the ranking-based rule with a percentage threshold. 

\begin{figure}[!t]
    \centering
    \includegraphics[width=.9\linewidth]{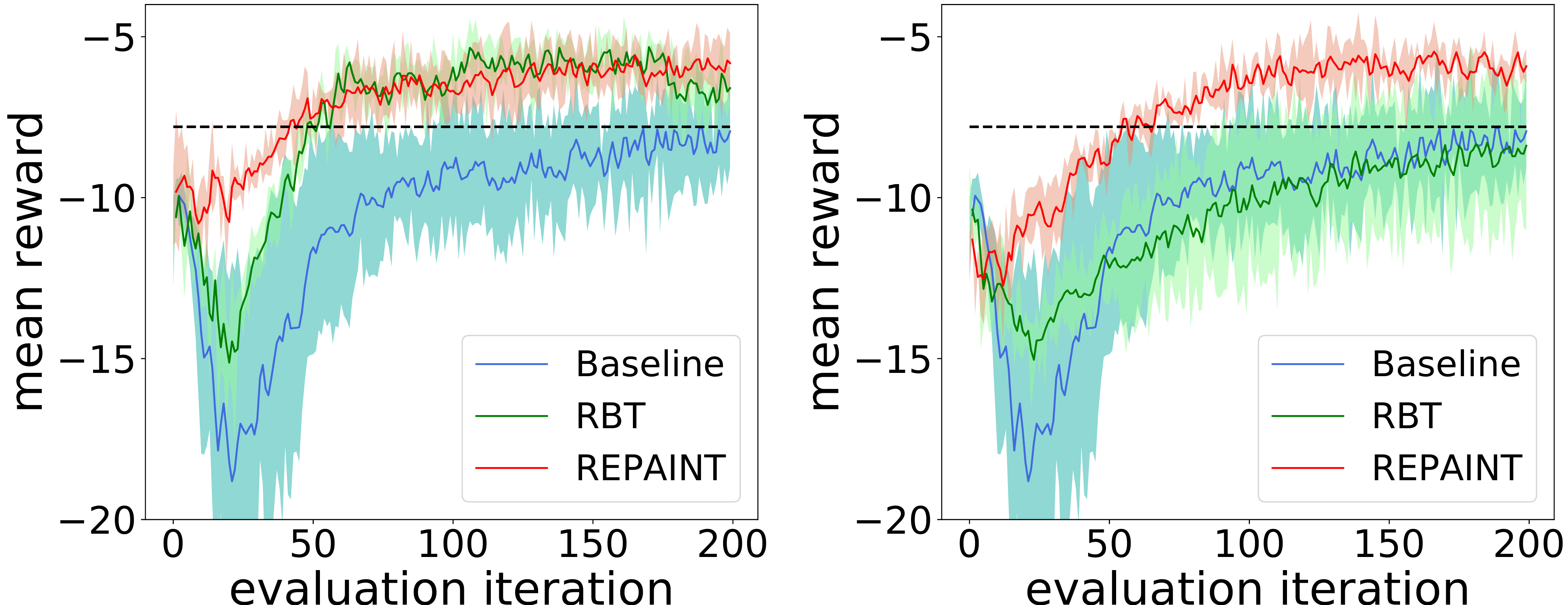}
    \caption{Comparison of relevance-based transfer (RBT) and REPAINT. Left: similar tasks. Right: dissimilar tasks.}
    \label{fig:reacher3}
\end{figure}

In order to showcase that our REPAINT algorithm exploits the (semantic) relatedness between source samples and target tasks, while most of other TL algorithms exploit the sample similarities, we make a comparison with an existing method here. \citet{lazaric2008transfer} defined some metrics of task compliance and sample relevance and proposed an instance transfer algorithm based on that. We call it relevance-based transfer (RBT). We incorporate kickstarting with RBT, and compare its TL performance with REPAINT in Figure~\ref{fig:reacher3}. When a similar task is used in knowledge transfer, RBT works well. However, when the target task is very different from the source task, although RBT attempts to transfer the most similar samples, it has no performance gain over the baseline training. The performance of REPAINT is significantly better than RBT in this case.


\begin{figure}[!t]
    \centering
    \includegraphics[width=.9\linewidth]{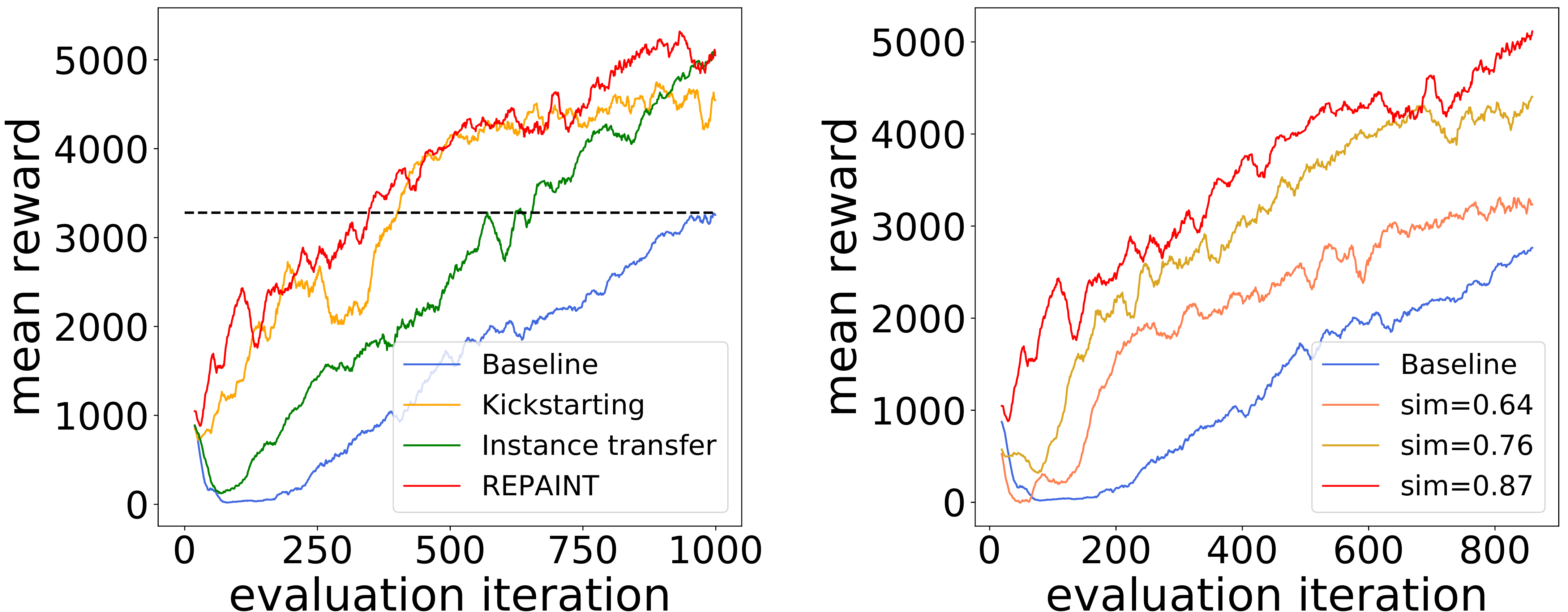}
    \caption{Evaluation performance for MuJoCo-Ant, averaged across three runs. Left: Model performance with same teacher that is pre-trained from a similar task. Right: REPAINT performance with different teacher policies (cosine similarities between source and target tasks are given). The plots are smoothed for visibility.}
    \label{fig:ant}
\end{figure}

\vspace{-0.8em}
\paragraph{MuJoCo-Ant.} 
In this target task, the agent is rewarded by survival and moving forward, and penalized by control and contact cost. The teacher policies are all trained from similar tasks, where the reward functions have higher weights on moving forward. 
We train each model for 1000 iterations, and evaluate each iteration for another 5 episodes. The results are shown in Figure~\ref{fig:ant}\footnote{We noticed that in MuJoCo-Ant, the variances of evaluation across different runs are very large. Hence we only show the mean values and omit the error bars in Figure~\ref{fig:ant}.}.
From the left sub-figure, we can again observe that, when task similarity is positive, training with REPAINT or kickstarting significantly improves the initial performance and reduces the learning cost of achieving a certain performance level. 
We also evaluate the transfer performance from the similar source tasks with different levels of cosine similarities. We set the forward coefficients in the reward function of teacher tasks to be 3, 5, and 10, corresponding to the cosine similarities of 0.87, 0.76, and 0.64, respectively. The results in the right sub-figure indicate that task similarity impacts the overall training performance, even when they are all related. Pre-trained teacher policies from more similar tasks can better contribute to the transfer performance. In addition, we present more results in Section~\ref{sec:ant-threshold}, showing that REPAINT is robust to the threshold parameter.


\begin{figure}[!t]
    \begin{subfigure}{.9\linewidth}
        \centering
        \includegraphics[width=\linewidth]{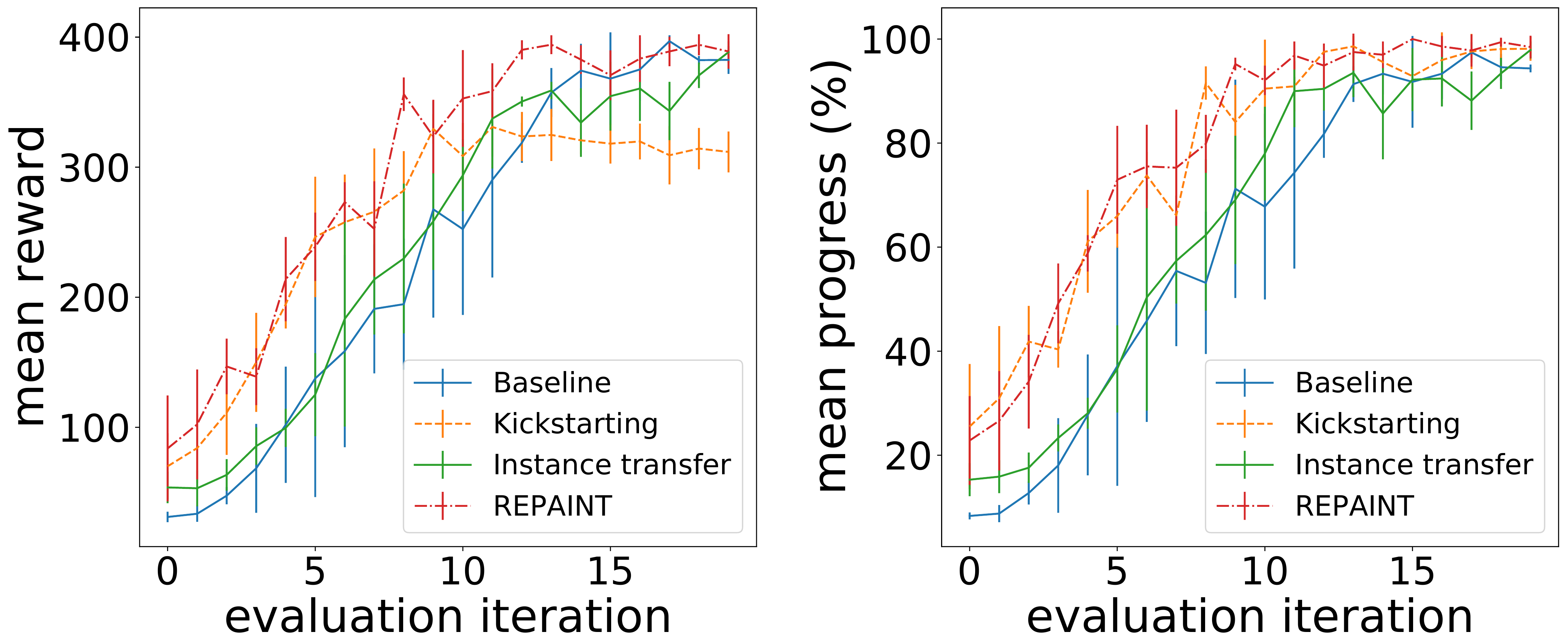}
        \caption{Outer-lane reward task with inner-lane teacher}
    \end{subfigure}
    \begin{subfigure}{.9\linewidth}
        \centering
        \includegraphics[width=\linewidth]{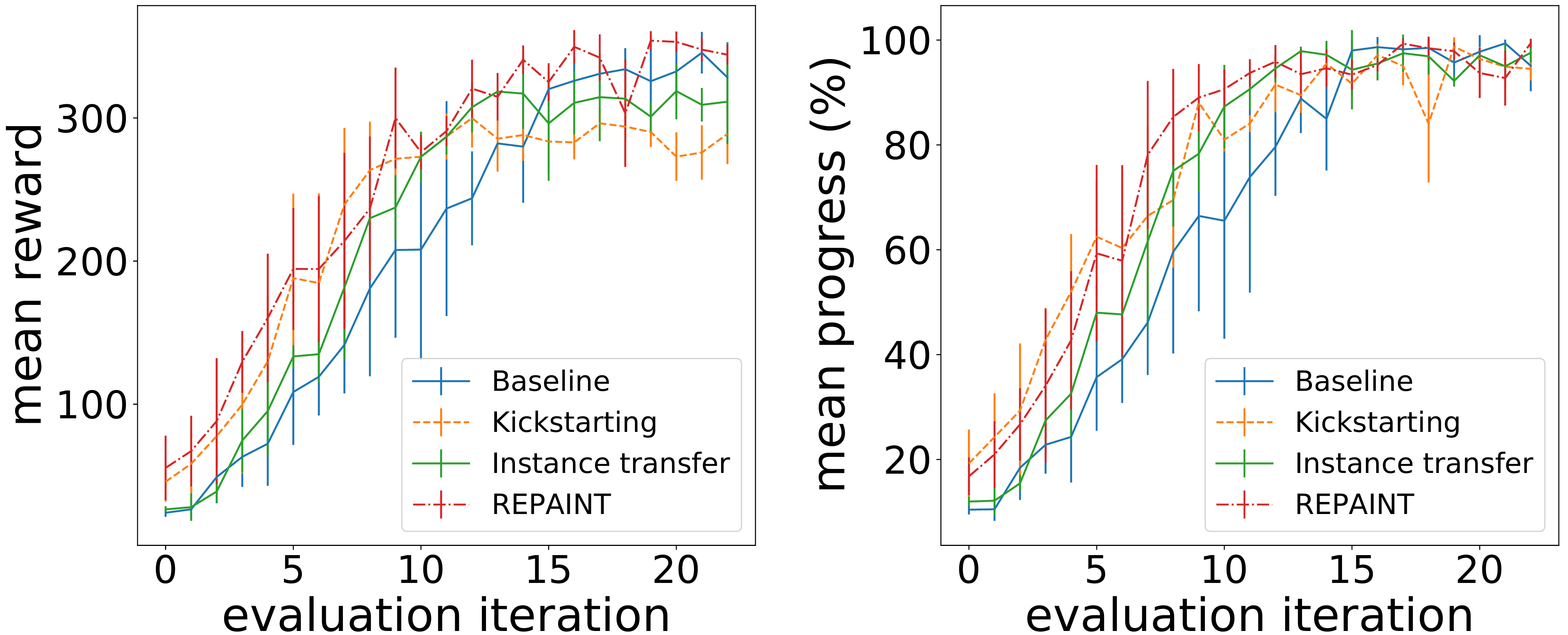}
        \caption{Inner-lane reward task with outer-lane teacher}
    \end{subfigure}
    \caption{Evaluation performance for DeepRacer single-car time-trial race, including mean accumulated rewards and mean progress (lap completion percentage), averaged across five runs.}
    \label{fig:single-car}
\end{figure}

\subsection{Autonomous Racing in AWS DeepRacer}
\paragraph{Single-car time trial.} In this experiment, we use two different reward functions, one of which rewards the agent when it is in the inner lane and penalizes when in the outer lane, and the other reward function does the opposite. When we use one reward in the student task, we provide the teacher policy that is trained with the other reward. Therefore, the cosine similarity of teacher and target tasks is negative. 

We evaluate the policy for 5 episodes after each iteration. The evaluation performance is presented in Figure~\ref{fig:single-car}, where both average return and progress (percentage of a lap the agent accomplished when it went out of track) are given. Although upon convergence, all models can finish a lap without going off-track, REPAINT and kickstarting again significantly boost the initial performance. However, when the teacher task is very different from the target task, training with kickstarting cannot improve the final performance via transfer. In contrast, instance transfer can still reduce the training convergence time with a final performance better than kickstarting (though with small margins in this example). Due to the page limit, we present the study on the effect of different cross-entropy weights $\beta_k$ and instance filtering thresholds $\zeta$ in Section~\ref{sec:single-car-exp} of appendix.

We also want to compare the REPAINT algorithm, which is a \textit{representation-instance} transfer algorithm, with a widely-used \textit{parameter} transfer approach, i.e., warm-start. In warm-start, the agent initializes with parameters from the teacher policy and conducts the RL algorithm after that. When the target task is similar to the teacher task, it usually works well. But here we compare the two algorithms in the DeepRacer single-car experiment, where the two tasks are significantly different. Figure~\ref{fig:trace} visualizes the trajectories of the agent on the track during evaluations. Each model is trained for two hours and evaluated for another 20 episodes. From both cases, we can see that although the two reward functions encode totally different behaviors, REPAINT can still focus on current task while learning from the teacher policy. This again indicates the effectiveness of the advantage-based experience selection in the instance transfer
. In comparison, training with warm-start cannot get rid of the unexpected behavior at convergence due to the reason that it may be stuck at some local optima. Therefore, initialization with previously trained policies can sometimes jump-start the training with good initial performance, but the method contributes to the final performance only when two tasks are highly similar.

\begin{figure}[!t]
    \centering
    \begin{subfigure}{\linewidth}
        \centering
        \includegraphics[width=.45\linewidth]{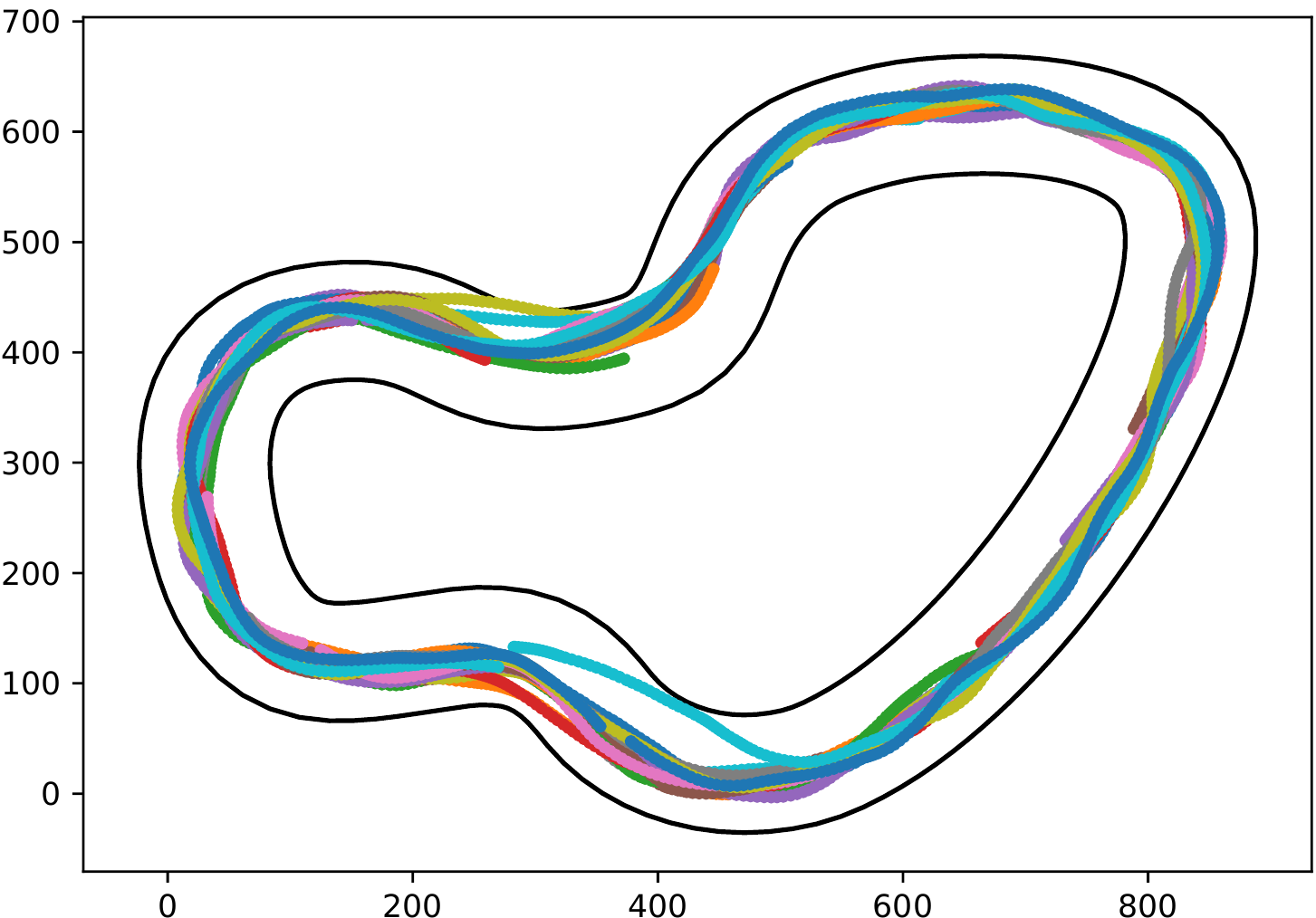}
        \includegraphics[width=.45\linewidth]{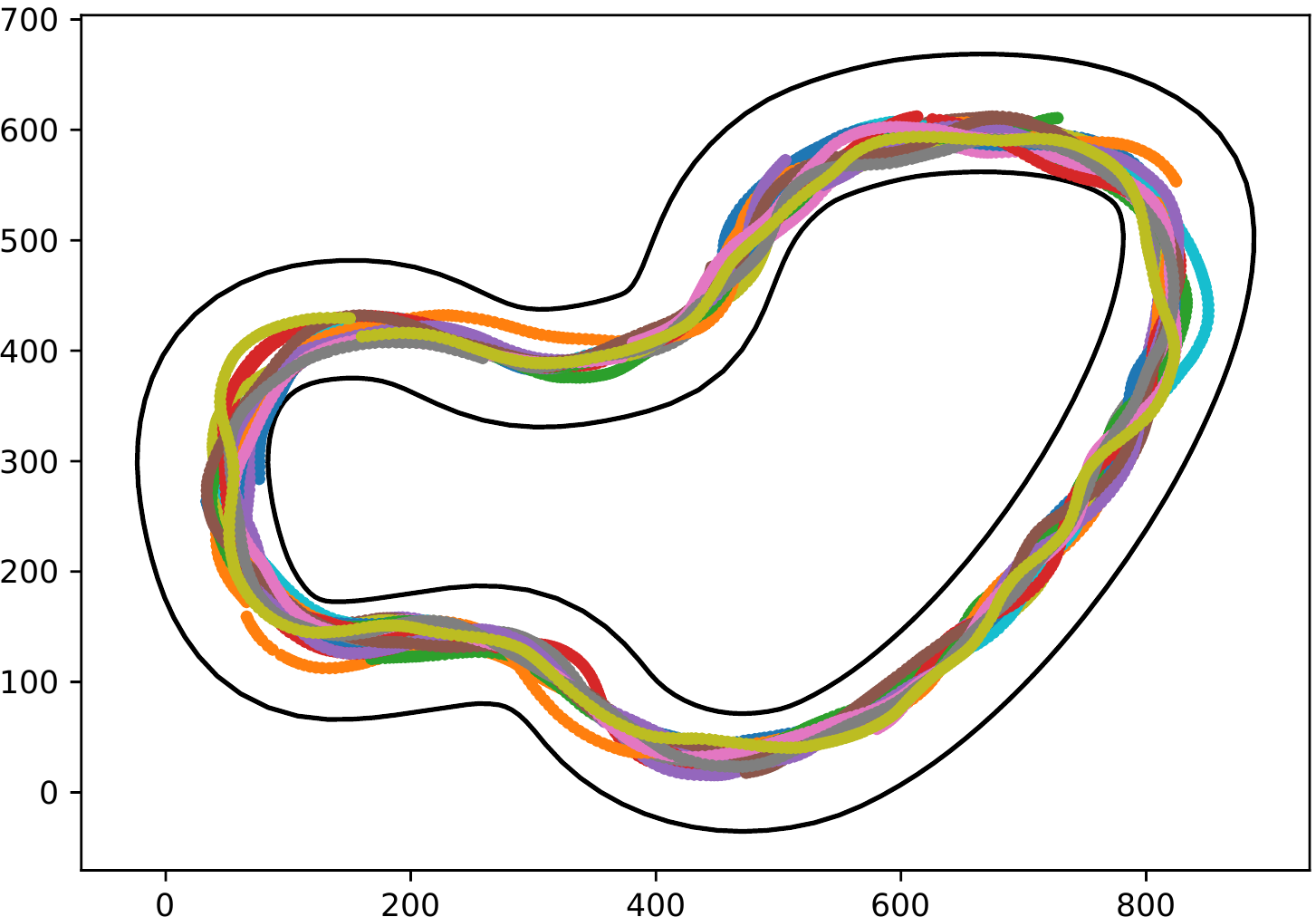}
        \caption{Outer-lane task: REPAINT \textit{vs}. warm-start}
    \end{subfigure}
    \begin{subfigure}{\linewidth}
        \centering
        \includegraphics[width=.45\linewidth]{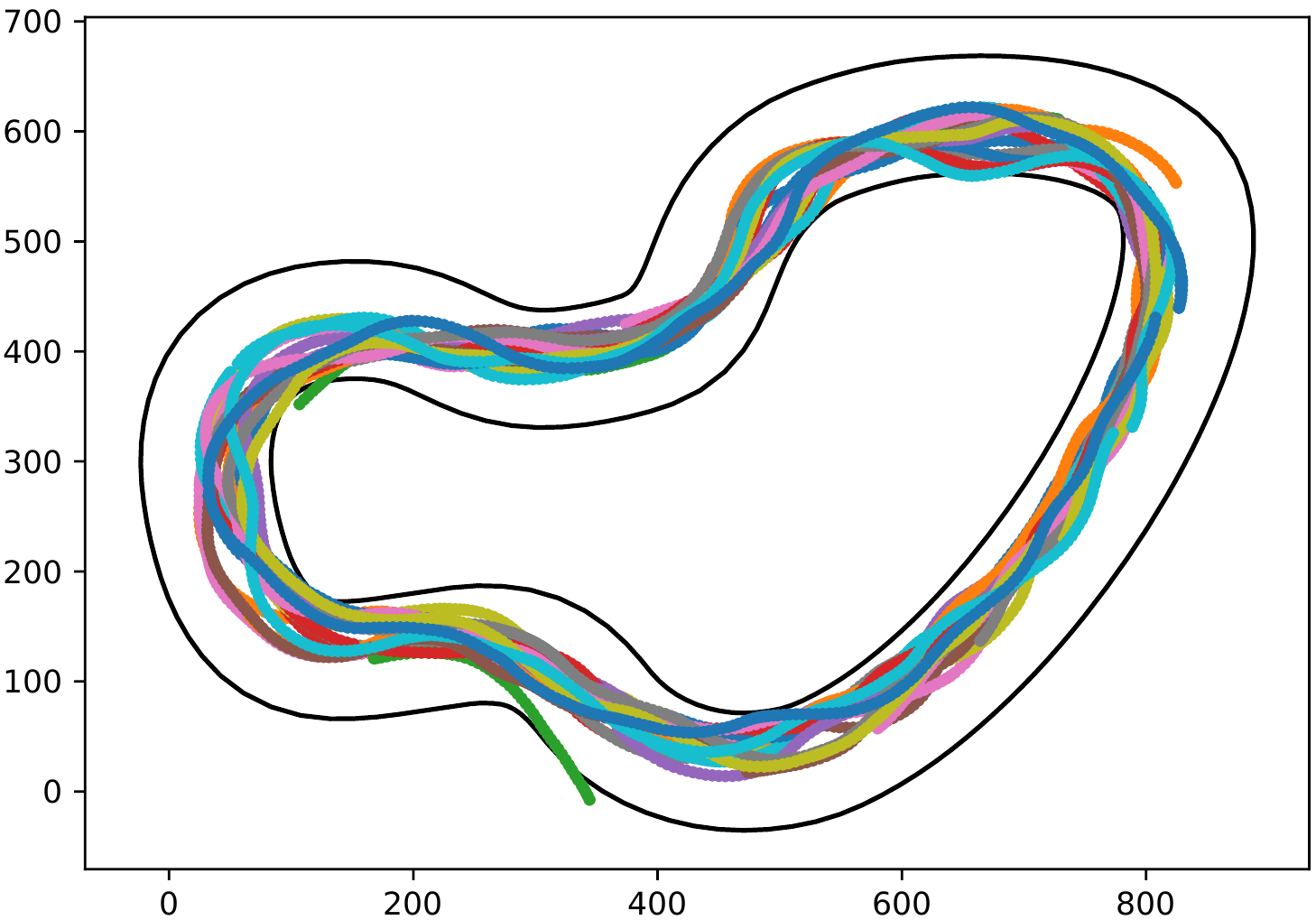}
        \includegraphics[width=.45\linewidth]{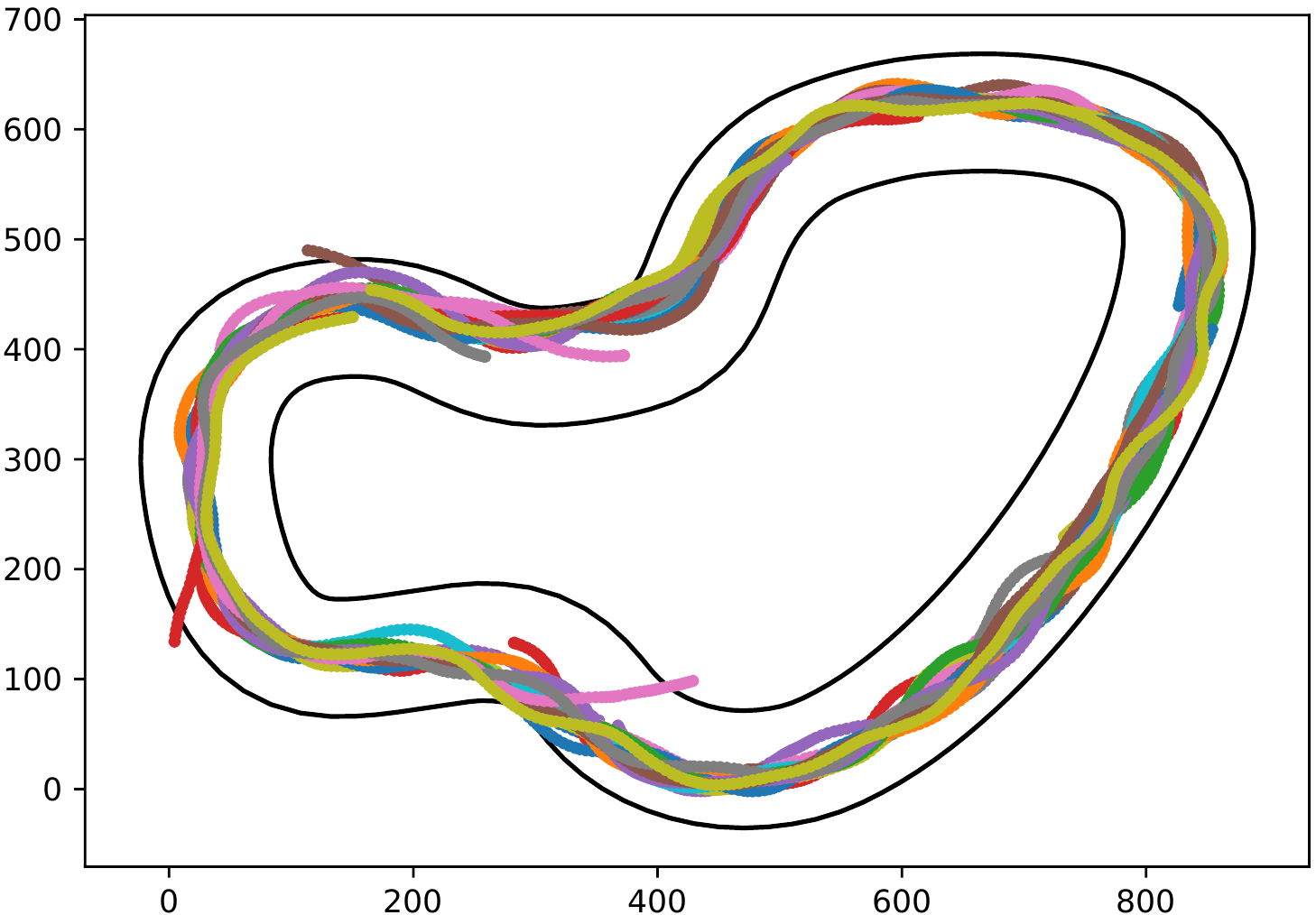}
        \caption{Inner-lane task: REPAINT \textit{vs}. warm-start}
    \end{subfigure}
    \caption{Trajectories of policy evaluations. In each of (a) and (b), evaluation of the models trained from REPAINT is visualized on the left and that trained with warm-start is on the right.}
    \label{fig:trace}
\end{figure}

\vspace{-0.6em}
\paragraph{Racing against bot cars.} The REPAINT algorithm is still helpful when the RL agent needs to learn multiple skills in a task. In the multi-car racing, the agent has to keep on the track while avoiding crashes with bot cars in order to obtain high rewards. 
We first train a teacher policy which is good at object avoidance, namely, the agent is rewarded when it keeps away from all bot cars, and gets a penalty when the agent is too close to a bot car and heads towards to it. 
Then in target tasks, we use two different reward functions to assess the models. First, we use an \textit{advanced reward} where other than keeping on track and object avoidance, it also penalizes the agent when it detects some bot car from the camera and is in the same lane with the bot. The evaluation performance is shown in Figure~\ref{fig:multi-car} (left). Since the environment has high randomness, such as agent and bot car initial locations and bot car lane changing, we only report average results. One can observe that REPAINT outperforms other baselines regarding the training time needed for certain performance level and the asymptotic performance.
Another target task with a \textit{progress-based reward} is also investigated, where the agent is only rewarded based on its completion progress, but gets large penalty when it goes off-track or crashes with bot cars. Since maximizing the completion progress involves bot car avoidance, the teacher task can be seen as either a different task or a sub-task. The results are provided in Figure~\ref{fig:multi-car} (right). When the target task is complex and the reward is simple (sparse) as in this case, it is sometimes difficult for the agent to learn a good policy as it lacks guidance from the reward on its actions. 
From the sub-figure, we can again see that training with REPAINT not only largely reduces the convergence time, but also improves the asymptotic performance compared to other models.

\begin{figure}[!t]
    \centering
    \includegraphics[width=.9\linewidth]{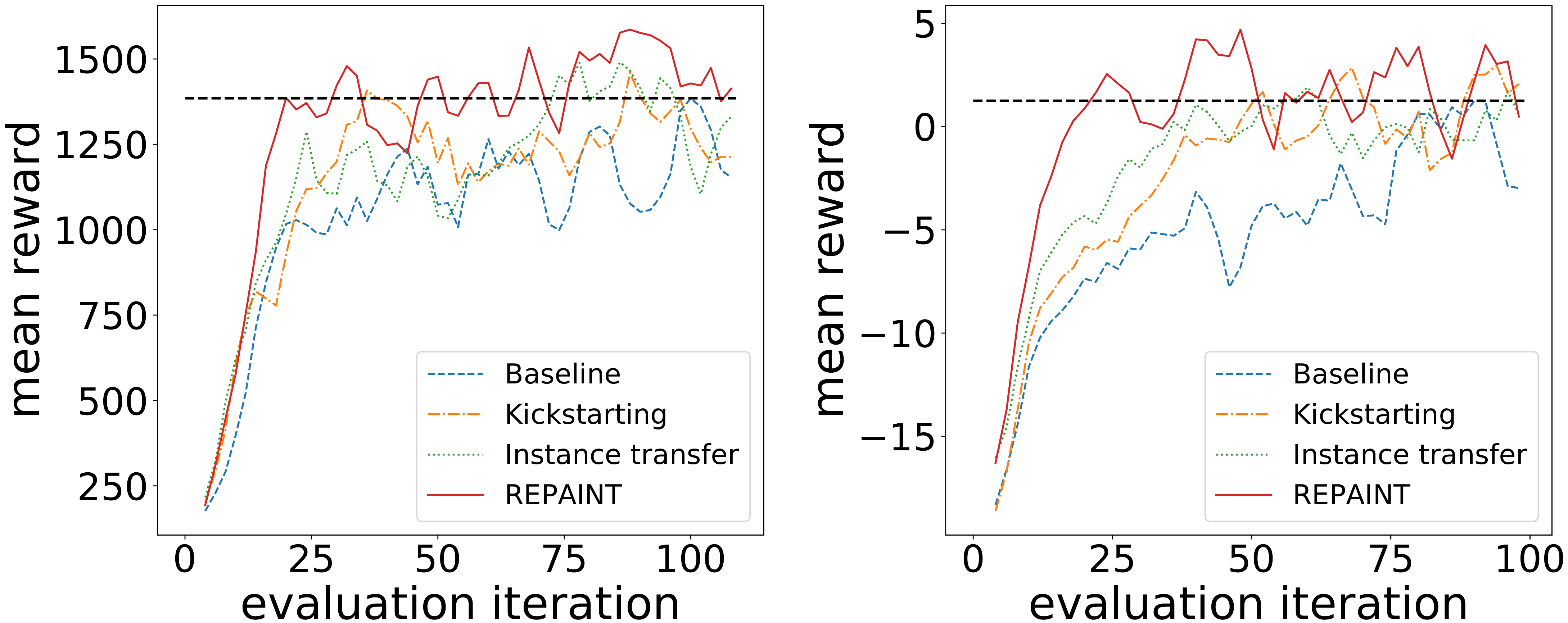}
    \caption{Evaluation performance for DeepRacer multi-car racing against bot cars, averaged across three runs. Left: Task with advanced reward. Right: Task with progress-based reward. The plots are smoothed for visibility.}
    \label{fig:multi-car}
\end{figure}

\begin{table*}[t]
\caption{Summary of the experimental results.}
\label{table:summary}
\begin{center}
\begin{footnotesize}
\begin{tabular}{cccccccccc}
\toprule
\multirow{2}{*}{Env.} &  Teacher & Target & \multirow{2}{*}{$K_\text{Baseline}$ } & $K_\text{KS}$ & $K_\text{IT}$ & $K_\text{REPAINT}$ & \multicolumn{3}{c}{Best scores} \\
& type & score & & (pct. reduced) & (pct. reduced) & (pct. reduced) & KS & IT & REPAINT \\
\midrule
 \multirow{2}{*}{Reacher} &  similar & \multirow{2}{*}{-7.4} & \multirow{2}{*}{173} & 51 (71$\%$) & 97 (44$\%$) & 42 (76$\%$) & -5.3 & -5.9 & -5.4 \\
 & different &  &  & 73 (58$\%$) & 127 (27$\%$) & 51 (71$\%$) & -6.9 & -6.4 & -5.2 \\
 \midrule
 Ant &  similar & 3685 & 997 & 363 (64$\%$) & 623 (38$\%$) & 334 (66$\%$) & 5464 & 5172 & 5540 \\
 \midrule
 \multirow{2}{*}{Single-car} & different & 394 & 18 & Not achieved & Not achieved & 13 (28$\%$) & 331 & 388 & 396  \\
 &  different & 345 & 22 & Not achieved & Not achieved & 15 (32$\%$) & 300 & 319 & 354 \\
 \midrule
 \multirow{2}{*}{Multi-car} & sub-task & 1481 & 100 & 34 (66$\%$) & 75 (25$\%$) & 29 (71$\%$) & 1542 & 1610 & 1623 \\
 &  diff/sub-task & 2.7 & 77 & 66 (14$\%$) & 53 (31$\%$) & 25 (68$\%$) & 4.9 & 4.2 & 6.1 \\
 \midrule
 StarCraft II & sub-task & 112 & 95 & 92 (3\%) & 24 (75\%) & 6 (94\%) & 125 & 312 & 276 \\
\bottomrule
\end{tabular}
\end{footnotesize}
\end{center}
\end{table*}

\subsection{More Complex Tasks: StarCraft II Environments}
At last, we also conduct the ablation study on a more complex transfer learning task using StarCraft II Learning Environments \citep{vinyals2017starcraft}. The teacher policy is trained on the \textit{BuildMarines} mini-game, where the agent is given a limited base and is tasked to maximize the number of marines trained. Then the target task builds upon \textit{BuildMarines} to include the \textit{FindAndDefeatZerglings} mini-game, denoted as \textit{BuildMarines+FindAndDefeatZerglings} (\textit{BM+FDZ}).
That is, on top of learning how to build marines, the agent must learn to use the built marines to explore the whole map, and try to find and defeat Zerglings that are randomly scattered across the map. Note that the map of \textit{BM+FDZ} is larger than that of \textit{BuildMarines}, so that although the state and action spaces are the same, the initial feature maps (observations) between the two tasks are not identical. Therefore, the knowledge transfer between the two tasks is not straightforward. Other details of the environments and reward functions can be found in Section~\ref{sec:env}.

Figure~\ref{fig:sc} shows the evaluation results for REPAINT and other methods. We want to first remark that it is a known issue that when \textit{BuildMarines} mini-game is involved, the RL performance has large variance, which can also be observed in Figure 6 of \citet{vinyals2017starcraft}. From the figure, we can again see that when the source task is a sub-task of a complex target task, the kickstarting method cannot transfer well. In contrast, the instance transfer and the REPAINT algorithm quickly transfer the knowledge of building marines to the target task by selecting the samples with high semantic relatedness. Moreover, they have achieved significantly higher rewards than training from scratch or kickstarting.

\begin{figure}[!t]
    \centering
    \includegraphics[width=.57\linewidth]{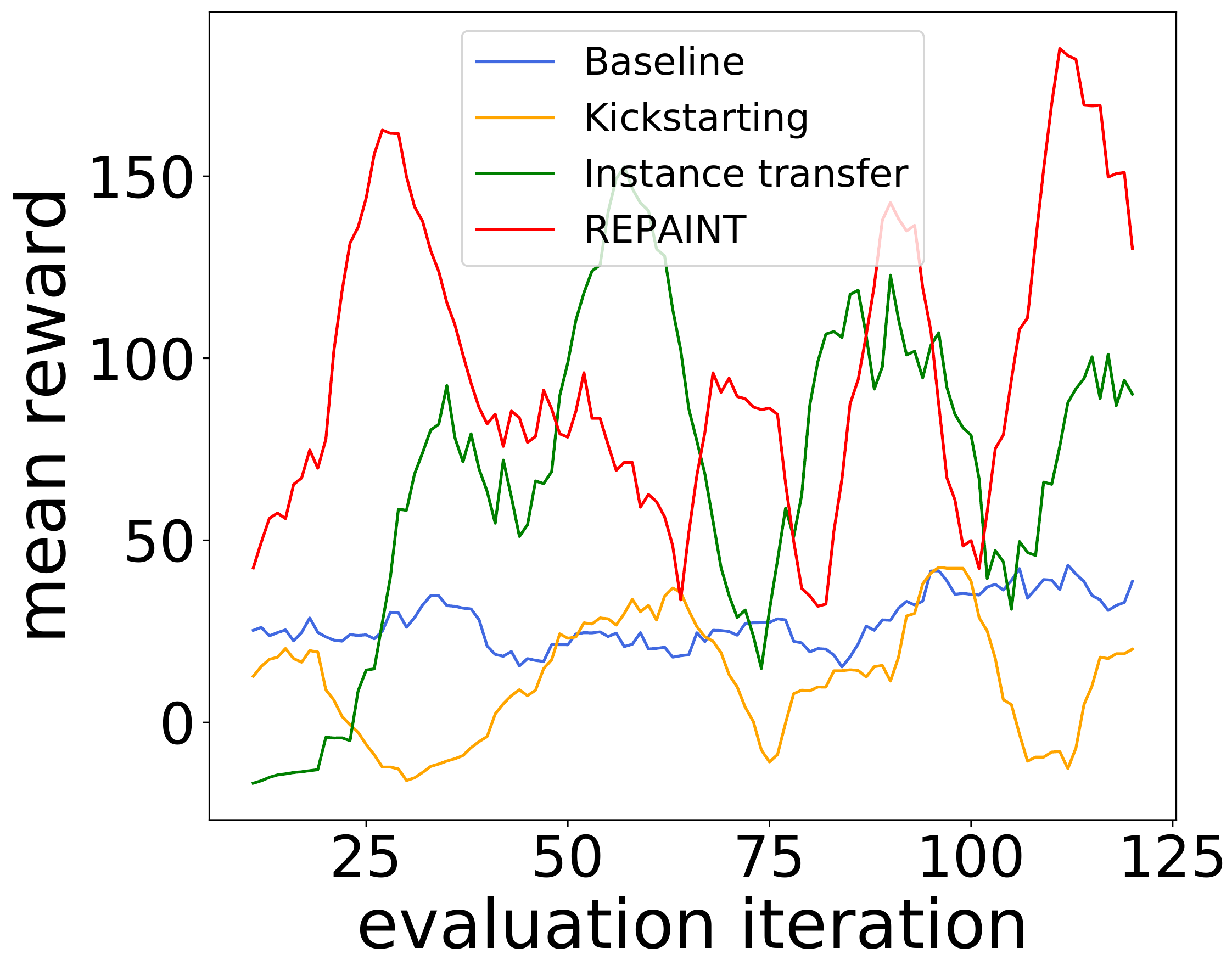}
    \caption{Evaluation performance for StarCraft II environments. The source task is a sub-task of the target task but uses a different map. The plots are smoothed for visibility.}
    \label{fig:sc}
\end{figure}

\section{Summary and Future Work}
In this work, we have proposed a knowledge transfer algorithm for RL. The REPAINT algorithm performs an on-policy representation transfer for the pre-trained teacher policies and an off-policy instance transfer for the samples collected following teacher policies. We develop an advantage-based experience selection approach in the instance transfer for selecting the samples that have high semantic relatedness to the target tasks. The idea of experience selection based on sample relatedness is simple but effective, and this is the first time that it is applied to knowledge transfer in RL.


We provide a summary of the experimental results in Table~\ref{table:summary}. The teacher type indicates whether the teacher task is a sub-task of or similar to the target task, based on the cosine similarity \eqref{eq:similarity}. 
The target score is the best performance that the baseline model can achieve. Then we provide the number of training iterations needed by each model to reach the target score. The models include training with baseline, kickstarting (KS), instance transfer (IT), and REPAINT. 
In Section~\ref{sec:moresum}, we also provide the data of wall-clock time. Although we are more interested in the training time reduction from TL, we present the best scores that each knowledge transfer model can achieve in the table. 
The kickstarting model performs well only when the tasks are similar or the target task is simple. While the instance transfer model can transfer related samples, it does not boost the initial performance. In contrast, the superior performance for REPAINT is observed regardless of the task similarity. It significantly reduces the training time for each target task, and also improves the final return scores in most tasks.

In future work, we aim to study how REPAINT can automatically learn the task similarity, and spontaneously determine the best $\beta_k$ and $\zeta$ values in the training based on the similarity. Our preliminary results in Section~\ref{sec:single-car-exp} indicate that when the task similarity is low, larger $\beta_k$ values may reduce the asymptotic performance of the agent. Moreover, we are also interested in the dependency of transfer performance on the neural network architectures. We provide some preliminary experimental results in Section~\ref{sec:more-exp-nn}.

\section*{Acknowledgements}
The authors would like to thank the anonymous reviewers for the valuable comments.

\bibliography{nips2020_refs}
\bibliographystyle{icml2021}

\clearpage
\appendix



\newtheorem{innercustomthm}{Theorem}
\newtheorem{innercustomass}{Assumption}
\newtheorem{innercustomlem}{Lemma}
\newenvironment{customthm}[1]
  {\renewcommand\theinnercustomthm{#1}\innercustomthm}
  {\endinnercustomthm}
\newenvironment{customass}[1]
  {\renewcommand\theinnercustomass{#1}\innercustomass}
  {\endinnercustomass}
  \newenvironment{customlem}[1]
  {\renewcommand\theinnercustomlem{#1}\innercustomlem}
  {\endinnercustomlem}




\section{Algorithms}\label{sec:alg}
In order to conduct a fair comparison with the baseline algorithms regarding the reduction of number of training iterations, and demonstrate that REPAINT improves the sample efficiency in knowledge transfer, we use an alternating variant of REPAINT with Clipped PPO in the experiments. The algorithm is provided in Algorithm~\ref{alg:transfer-PPO}, where we adopt on-policy representation transfer and off-policy instance transfer alternately, so that the REPAINT performs policy update one time per iteration using less samples. Indeed, Algorithm~\ref{alg:transfer-PPO} can be easily extended with different alternating ratios other than 1:1 alternating. The corresponding results and discussion can be found in Section~\ref{sec:exp-ratio}. 

Note that in Algorithm~\ref{alg:repaint} and Algorithm~\ref{alg:transfer-PPO}, we can use different learning rates $\alpha_1$ and $\alpha_2$ to control the update from representation transfer and instance transfer, respectively. 
Moreover, it is straightforward to using multiple and different teacher policies in each transfer step, and our algorithm can be directly applied to any advantage-based policy gradient RL algorithms. Assume there are $m$ previously trained teacher policies $\pi_1,\ldots,\pi_m$. In the instance transfer, we can form the replay buffer $\tilde{\mathcal{S}}$ by collecting samples from all teacher policies. Then in the representation transfer, the objective function can be written in a more general way:
\begin{equation}\label{eq:rl-loss-general}
    L_\text{rep}^k(\theta) = L_\text{clip}(\theta) - \sum_{i=1}^m \beta^k_i H\left(\pi_i(a|s) \| \pi_\theta(a|s)\right)\,,
\end{equation}
where we can impose different weighting parameters for different teacher policies.

In addition, the first term in \eqref{eq:rl-loss-general}, i.e., $L_\text{clip}(\theta)$, can be naturally replaced by the objective of other RL algorithms, e.g., Advantage Actor-Critic (A2C) \cite{sutton2000policy}:
$$
     L_\text{A2C}(\theta) = \hat{\mathbb{E}}_t \left[\log \pi_\theta(a|s) \hat{A}_t\right]\,,
$$
and Trust Region Policy Optimization (TRPO) \citep{schulman2015trust}:
$$
     L_\text{TRPO}(\theta) = \hat{\mathbb{E}}_t \left[\frac{\pi_\theta(a|s)}{\pi_{\theta_\text{old}}(a|s)} \hat{A}_t - \beta \text{KL}[\pi_{\theta_\text{old}}(\cdot|s), \pi_\theta(\cdot|s)]\right]
$$
for some coefficient $\beta$ of the maximum KL divergence computed over states. 

REPAINT can also be adapted to other policy-gradient-based algorithms that are not based on advantage values. To this end, one can define a different metric for relatedness. For example, we can use REPAINT with REINFORCE \cite{williams1992simple}, by defining the relatedness metric to be $\hat{R} - b$, where $\hat{R}$ is the off-policy return and $b$ is the baseline function in REINFORCE, which can be state-dependent. Then the experience selection approach can be built based on the new relatedness metric.

\begin{algorithm}[!t]
  \caption{Alternating REPAINT with Clipped PPO}
  \label{alg:transfer-PPO}
\begin{algorithmic}
  \STATE Initialize parameters $\nu$, $\theta$
  \STATE Load teacher policy $\pi_{\text{teacher}}(\cdot)$
  \STATE Set hyper-parameters $\zeta$, $\alpha_1$, $\alpha_2$, and $\beta_k$'s in \eqref{eq:rl-loss}
  \FOR{iteration $k=1,2,\ldots$}
  \IF[\quad\quad// \emph{representation transfer}]{$k$ is odd}
  \STATE Collect samples $\mathcal{S}=\{(s,a,s',r)\}$ using $\pi_{\theta_\text{old}}(\cdot)$
  \STATE Fit state-value network $V_\nu$ using $\mathcal{S}$ to update $\nu$
  \STATE Compute advantage estimates $\hat{A}_1,\ldots,\hat{A}_T$
  \STATE Compute sample gradient of $L_\text{rep}^k(\theta)$ in \eqref{eq:rl-loss}
  \STATE Update policy network by $\theta \leftarrow \theta + \alpha_1\nabla_\theta L_\text{rep}^k(\theta)$
  \ELSE[\quad\quad// \emph{instance transfer}]
  \STATE Collect samples $\tilde{\mathcal{S}}=\{(\tilde{s},\tilde{a},\tilde{s}',\tilde{r})\}$ using $\pi_{\text{teacher}}(\cdot)$ 
  \STATE Compute advantage estimates $\hat{A}'_1,\ldots,\hat{A}'_{T'}$
  \FOR[\quad\quad// \emph{experience selection}]{t=1,\ldots,$T'$}
    \IF{$\hat{A}'_t<\zeta$} 
        \STATE Remove $\hat{A}'_t$ and the corresponding transition $(\tilde{s}_t,\tilde{a}_t,\tilde{s}_{t+1},\tilde{r}_t)$ from $\tilde{\mathcal{S}}$
    \ENDIF
  \ENDFOR
  \STATE Compute sample gradient of $L_\text{ins}(\theta)$ in \eqref{eq:obj-ins}
  \STATE Update policy network by $\theta \leftarrow \theta + \alpha_2\nabla_\theta L_\text{ins}(\theta)$ 
  \ENDIF
  \ENDFOR
\end{algorithmic}
\end{algorithm}

\begin{figure*}[!t]
    \centering
    \begin{subfigure}{.23\textwidth}
        \centering
        \includegraphics[width=\linewidth]{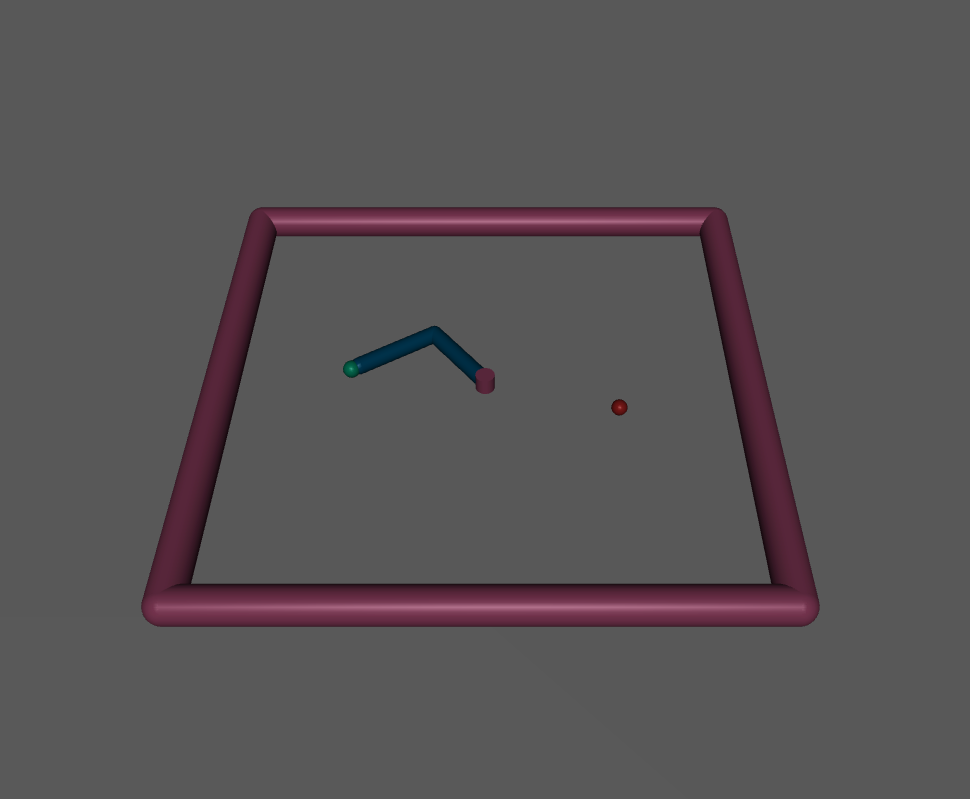}
        \caption{MuJoCo-Reacher}
    \end{subfigure}
    \begin{subfigure}{.232\textwidth}
        \centering
        \includegraphics[width=\linewidth]{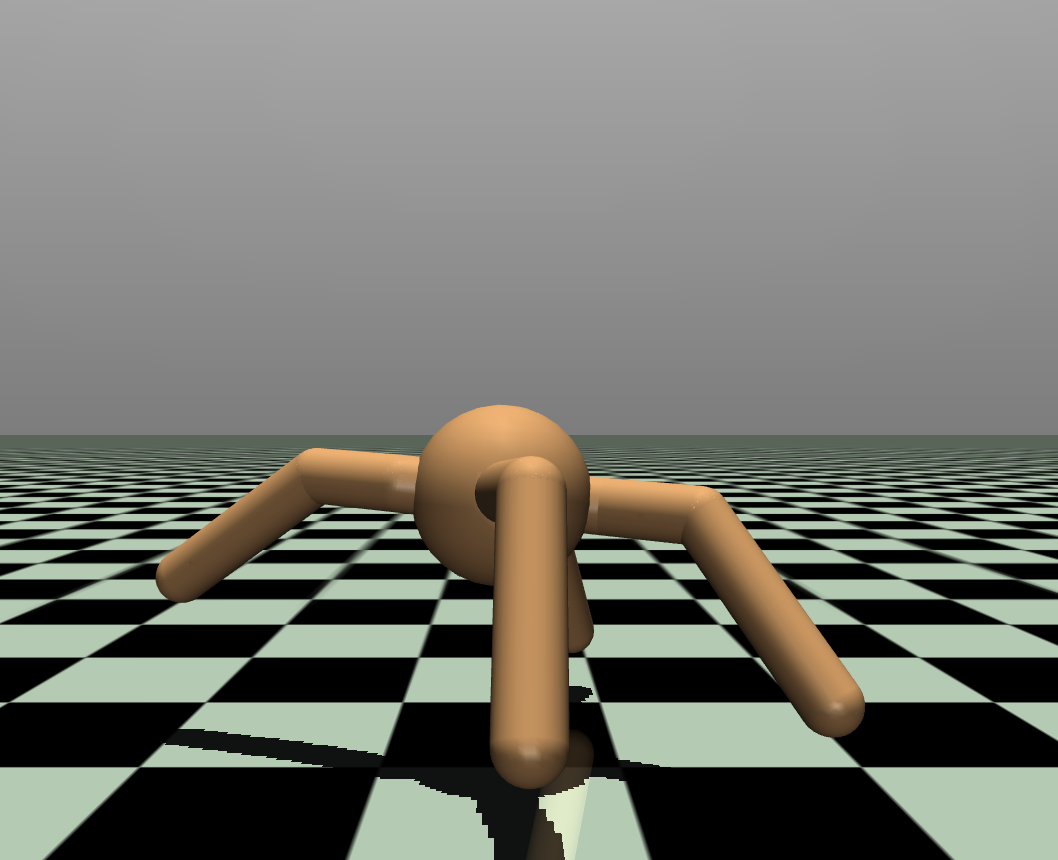}
        \caption{MuJoCo-Ant}
    \end{subfigure}
    \begin{subfigure}{.234\textwidth}
        \centering
        \includegraphics[width=\linewidth]{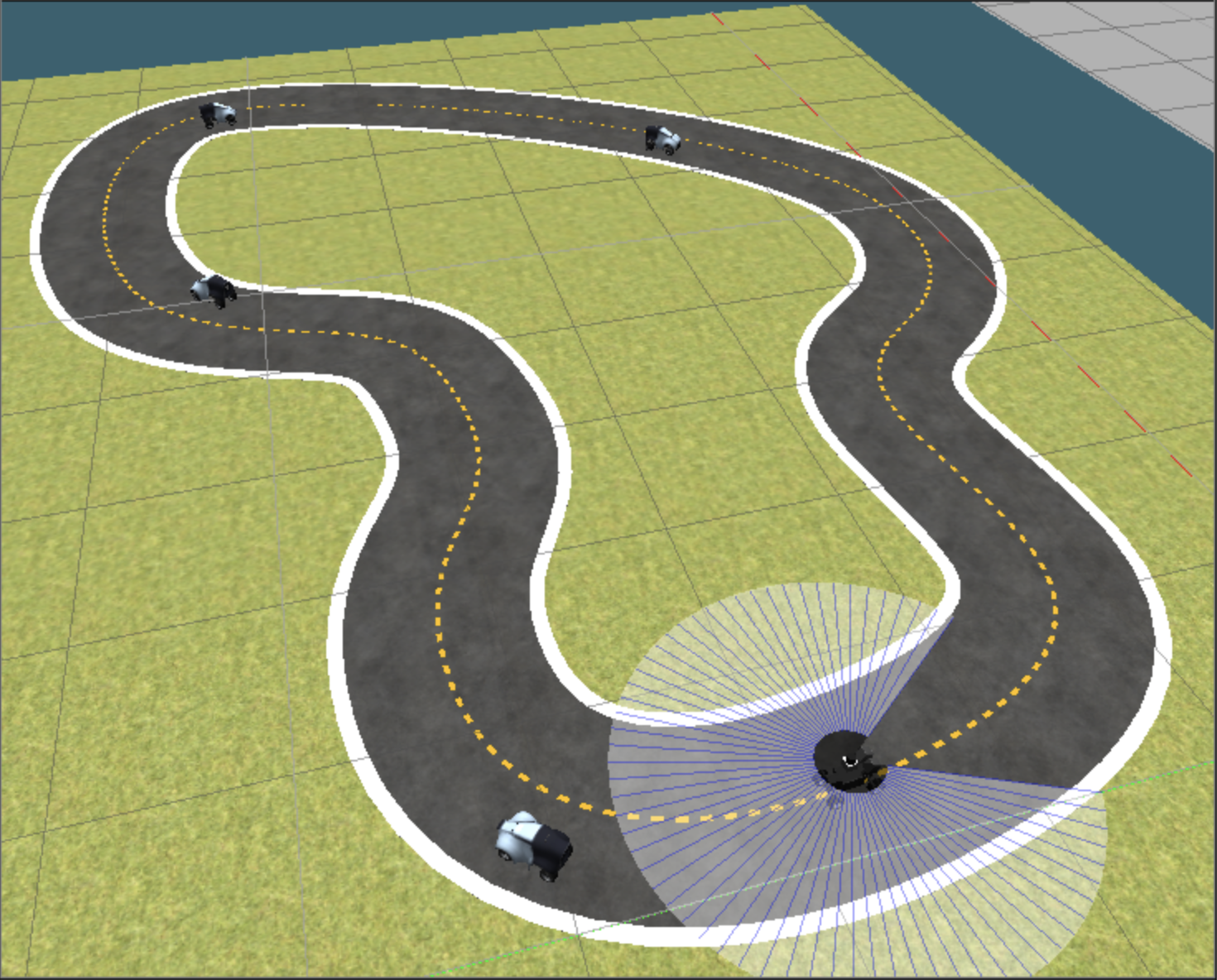}
        \caption{DeepRacer multi-car racing}
    \end{subfigure}
    \begin{subfigure}{.264\textwidth}
        \centering
        \includegraphics[width=\linewidth]{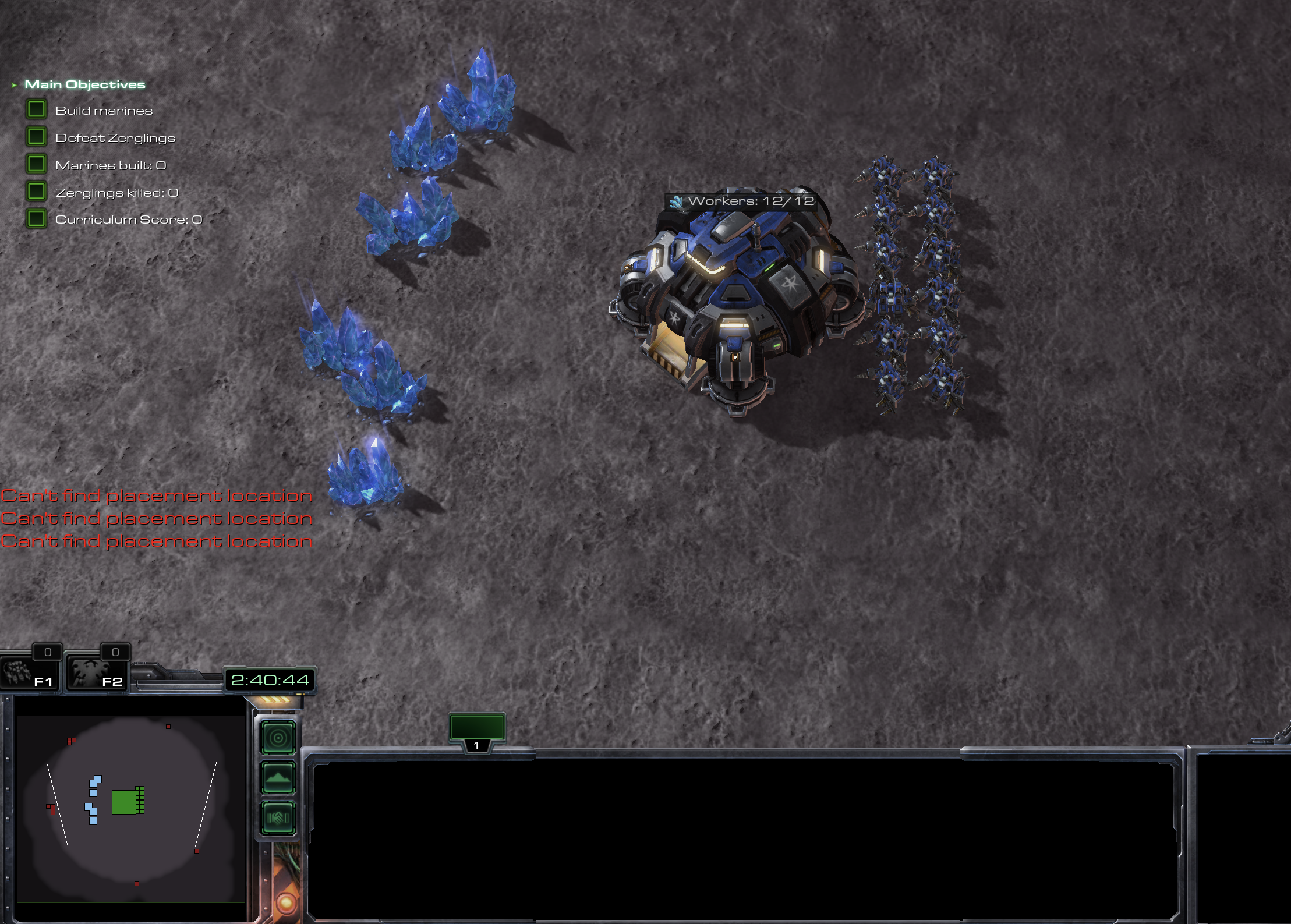}
        \caption{StarCraft II}
    \end{subfigure}
    \caption{The simulation environments used in the experiments. Note that in DeepRacer multi-car racing, the racing car is equipped by a stereo camera and a Lidar, as shown in (c). However, in the single-car time trial, the racing car only has a monocular camera. In addition, for the StarCraft II environment, 40 Zerglings are randomly generated in the unrevealed areas around the map.}
    \label{fig:deepracer}
\end{figure*}

We now discuss how the REPAINT algorithm can be extended to Q-learning. Since Q-learning is an off-policy algorithm, it is easy to notice that the kickstarting approach cannot be directly used. Although some other representation transfer approaches are suitable for Q-learning, e.g., using a neural network for feature abstraction, we skip the discussion as it is not our goal in this paper. Instead, we will focus on how to extend our experience selection approach to the instance transfer of Q-learning.

In the instance transfer for Q-learning, the Q-value network, parameterized by $\phi$, is updated by minimizing the following loss function:
$$
L(\phi) = \frac{1}{2}\sum_i \|Q_\phi(s_i,a_i)-y_i\|^2\,,
$$
where the samples are drawn from the replay buffer that is collected following some teacher policy, and the target values $y_i$'s are defined by
$$
y_i = r(s_i,a_i) +\gamma \max_{a_i'} Q_\phi(s_i',a_i')\,,
$$
with $\gamma$ the discount factor. Now given some threshold $\zeta\ge 0$, we can select the samples that satisfy the following condition to update $\phi$:
$$
y_i - Q_\phi(s_i,a_i) > \zeta\,.
$$
Similar to REPAINT with actor-critic methods, the threshold $\zeta$ here is task specific, but it needs more careful treatment in Q-learning. Since we aim to obtain an optimal Q-function $Q^*(s,a)$, we should use a $\zeta$ such that $Q^*(s,a)\ge y > Q_\phi(s,a)+\zeta$. For actor-critic methods, we can empirically show that REPAINT is robust to the advantage threshold. However, for Q-learning, we usually need to set $\zeta$ to be very small and use the experience selection only in the early training stage. Motivated by \citet{oh2018self}, we can also set $\zeta=0$. Then the convergence of Q-value follows $Q^*(s,a)\ge y > Q_\phi(s,a)$. By filtering out the samples such that $Q_\phi(s_i,a_i) \ge y_i$, one can expect the instance transfer to improve the sample efficiency and reduce the total training time for complex target tasks.

In practice, to trade off the exploitation with exploration in Q-learning, we can collect some samples following the $\epsilon$-greedy policy from the online Q-value network, and add those samples to the replay buffer as well.

\section{Details of Experimental Setup}\label{sec:env}
\subsection{Environments}
We now provide the details of our experimental setup. The graphical illustration of the environments used is presented in Figure~\ref{fig:deepracer}. First of all, MuJoCo is a well-known physics simulator for evaluating agents on continuous motor control tasks with contact dynamics, hence we omit the further description of MuJoCo in this paper.

\vspace{-0.6em}
\paragraph{DeepRacer simulator.}
In AWS DeepRacer simulator\footnote{\url{https://github.com/awslabs/amazon-sagemaker-examples/tree/master/reinforcement_learning/rl_deepracer_robomaker_coach_gazebo}}, the RL agent, i.e., an autonomous car, learns to drive by interacting with its environment, e.g., the track with moving bot cars, by taking an action in a given state to maximize the expected reward. Figure~\ref{fig:deepracer}(c) presents the environmental setting for racing against moving bot cars, where four bot cars are generated randomly on the track and the RL agent learns to finish the lap with overtaking bot cars. Another racing mode we used in this paper is the single-car time-trial race, where the goal is to finish a lap in the shortest time.

In single-car racing, we only install a front-facing camera on the RL agent, which obtains an RGB image with size $120\times160\times3$. The image is then transformed to gray scale and fed into an input embedder. For simplicity, the input embedder is set to be a three-layer convolutional neural network (CNN) \citep{goodfellow2016deep}. For the RL agent in racing against bot cars, we use a stereo camera and a Lidar as sensors. The stereo camera obtains two images simultaneously, transformed to gray scale, and concatenates the two images as the input, which leads to a $120\times160\times2$ input tensor. The input embedder for stereo camera is also a three-layer CNN by default. The stereo camera is used to detect bot cars in the front of learner car, while the Lidar is used to detect any car behind. The backward-facing Lidar has an angle range of 300 degree and a 64 dimensional signal. Each laser can detect a distance from 12cm to 1 meter. The input embedder for Lidar sensor is set to be a two-layer dense network. In both environments, the output has two heads, V head for state value function output and policy head for the policy function output, each of which is set to be a two-layer dense networks but with different output dimensions. The action space consists of a combination of five different steering angles and two different throttle degrees, which forms a 10-action discrete space. In the evaluation of DeepRacer experiments, the generalization around nearby states and actions is also considered \citep{balaji2019deepracer}, where we add small noises to the observations and actions.

\vspace{-0.6em}
\paragraph{StarCraft II learning environments (SC2LE).}
The \textit{BuildMarines} mini-game is shown in Figure~\ref{fig:deepracer}(d), where it limits the possible actions that the agent can take to either of selecting points, building workers, building supply depots, building barracks, and training marines. For \textit{BuildMarines+FindAndDefeatZerglings} (\textit{BM+FDZ}), we extend the action space to allow the agent to select the army and to attack with the army. As mentioned before, we keep the state and action spaces the same between source and target tasks in the experiments. Therefore, we provide the two army-related actions in \textit{BuildMarines} but they are always unavailable.

\begin{table}[!t]
\caption{StarCraft II \textit{BM+FDZ} reward scheme.}
\label{table:SC_reward}
\begin{center}
\begin{tabular}{c|c}
Condition & Reward \\ \hline
Performing an unavailable action & -0.01 \\
A marine dying & -4 \\
Selecting an unavailable point  & -0.01 \\
Selecting an available point & 0.01 \\
Training an SCV & 0.5 \\
Building barrack & 0.2 \\
\end{tabular}
\end{center}
\end{table}

The default observations provided in the SC2LE are used and we follow a similar network architecture to the baseline presented in \citet{vinyals2017starcraft}.
Spatial features, including screen features (size = $84 \times 84 \times 9$) and  mini-map features (size = $84 \times 84 \times 3$), are each fed through input embedders consisting of CNN with two layers. 
Non-spatial features, including the measurements (size = 5; e.g., mineral count, food count, army count) and the one-hot encoded vector of available actions (size = 7, e.g., build worker, select screen) are fed into an input embedder consisting of linear layers with a $tanh$ activation.

The action space consists of one discrete action to determine the command to take (i.e., build supply depot, build barrack, train SCV, train marine, select point on screen, attack point on screen, select army) and two actions to indicate where to commence the action on the screen (spatial action).
For example, with a command to build a barrack, the spatial action determines where the barrack will be built, and to attack the spatial action determines where the marines will attack.
It is important to note that not all commands rely on a corresponding spatial action.
For example, when issuing a command to train a marine, the spatial action is ignored.

In addition, several rules are implemented to ensure that the mini-game progresses as expected.
Firstly, workers (SCVs) cannot attack so that the agent will not find and attempt to defeat Zerglings with the workers.
Secondly, Zerglings cannot enter the base so that Zerglings do not overrun the base before marines are built. 

The \textit{BM+FDZ} agent is rewarded for each marine built and each Zergling killed.
Specifically, a +5 reward is imposed when a marine is built and +10 reward when a Zergling is killed.
Table \ref{table:SC_reward} shows small rewards and penalties that are given to facilitate the agent to achieve these goals.

\begin{table}[!t]
\caption{Hyper-parameters used in the MuJoCo simulations.}
\label{table:hyp1}
\begin{center}
\begin{tabular}{c|c}
Hyperparameter & Value \\ \hline
Num. of rollout steps & 2048 \\
Num. training epochs & 10 \\
Discount ($\gamma$) & 0.99 \\
Learning rate & 3e-4 \\
GAE parameter ($\lambda$) & 0.95 \\
Beta entropy & 0.0001 \\
Cross-entropy weight ($\beta_0$) & 0.2 \\
Reacher - Advantage Threshold ($\zeta$) & 0.8 \\
Reacher - Num. REPAINT iterations & 15 \\
Ant - Num. REPAINT iterations & 50 \\
\end{tabular}
\end{center}
\end{table}

\begin{table}[!t]
\caption{Hyper-parameters used in the DeepRacer simulations.}
\label{table:hyp2}
\begin{center}
\begin{tabular}{c|c}
Hyperparameter & Value \\ \hline
Num. of rollout episodes & 20 \\
Num. of rollout episodes when using $\pi_\text{teacher}$ & 2 \\
Num. training epochs & 8 \\
Discount ($\gamma$) & 0.999 \\
Learning rate & 3e-4 \\
GAE parameter ($\lambda$) & 0.95 \\
Beta entropy & 0.001 \\
Cross-entropy weight ($\beta_0$) & 0.2 \\
Advantage Threshold ($\zeta$) & 0.2 \\
Single-car - Num. REPAINT iterations & 4 \\
Multi-car - Num. REPAINT iterations & 20 \\
\end{tabular}
\end{center}
\end{table}

\begin{table}[!t]
\caption{Hyper-parameters used in the StarCraft II simulations.}
\label{table:hyp3}
\begin{center}
\begin{tabular}{c|c}
Hyperparameter & Value \\ \hline
Num. of rollout episodes & 2 \\
Num. of rollout episodes when using $\pi_\text{teacher}$ & 2 \\
Num. training epochs & 6 \\
Discount ($\gamma$) & 0.99 \\
Learning rate & 3e-5 \\
GAE parameter ($\lambda$) & 0.95 \\
Beta entropy & 0.01 \\
Cross-entropy weight ($\beta_0$) & 0.1 \\
Advantage Threshold ($\zeta$) & 0.2 \\
Num. REPAINT iterations & 25 \\
\end{tabular}
\end{center}
\end{table}

\subsection{Hyper-parameters}
We have implemented our algorithms based on Intel Coach\footnote{\url{https://github.com/NervanaSystems/coach}}. The MuJoCo environments are from OpenAI Gym\footnote{\url{https://gym.openai.com/envs/\#mujoco}}. The StarCraft II learning environments are from DeepMind's PySC2\footnote{\url{https://github.com/deepmind/pysc2}}. Regarding the advantage estimates, we use the generalized advantage estimator (GAE) \cite{schulman2015high}. If not specified explicitly in the paper, we always use Adam as the optimizer with minibatch size as 64, clipping parameter $\epsilon$ as 0.2, and $\beta_{k+1}=0.95\beta_k$ throughout the experiments. The other hyper-parameters are presented in Tables \ref{table:hyp1}-\ref{table:hyp3}.

\section{Extensive Experimental Results}\label{sec:more-exp}
\subsection{Discussion on Alternating Ratios}\label{sec:exp-ratio}
In Algorithm~\ref{alg:transfer-PPO}, we alternate representation transfer and instance transfer after each iteration. Here, we aim to illustrate the effect of using different alternating ratios by the MuJoCo-Reacher environment. We compare the 1:1 alternating with a 2:1 ratio, namely, two on-policy representation transfer (kickstarting) iterations before and after an off-policy instance transfer iteration. The evaluation performance is shown in Figure~\ref{fig:reacher_append}. When the teacher task is similar to the target task, adopting more kickstarted training iterations leads to faster convergence, due to the policy distillation term in the loss function. On the other hand, when the task similarity is low, instance transfer contributes more to the knowledge transfer due to the advantage-based experience selection. Therefore, we suggest to set the alternating ratio in Algorithm~\ref{alg:transfer-PPO}, or the $\alpha_1$ and $\alpha_2$ parameters in Algorithm~\ref{alg:repaint} and Algorithm~\ref{alg:transfer-PPO}, according to the task similarity between source and target tasks. However, the task similarity is usually unknown in most of the real-world applications, or the similarities are mixed when using multiple teacher policies. It is interesting to automatically learn the task similarity and determine the best ratio/parameters before actually starting the transfer learning. We leave the investigation of this topic as a future work.

\begin{figure}[!t]
    \centering
    \begin{subfigure}{.45\linewidth}
        \centering
        \includegraphics[width=\linewidth]{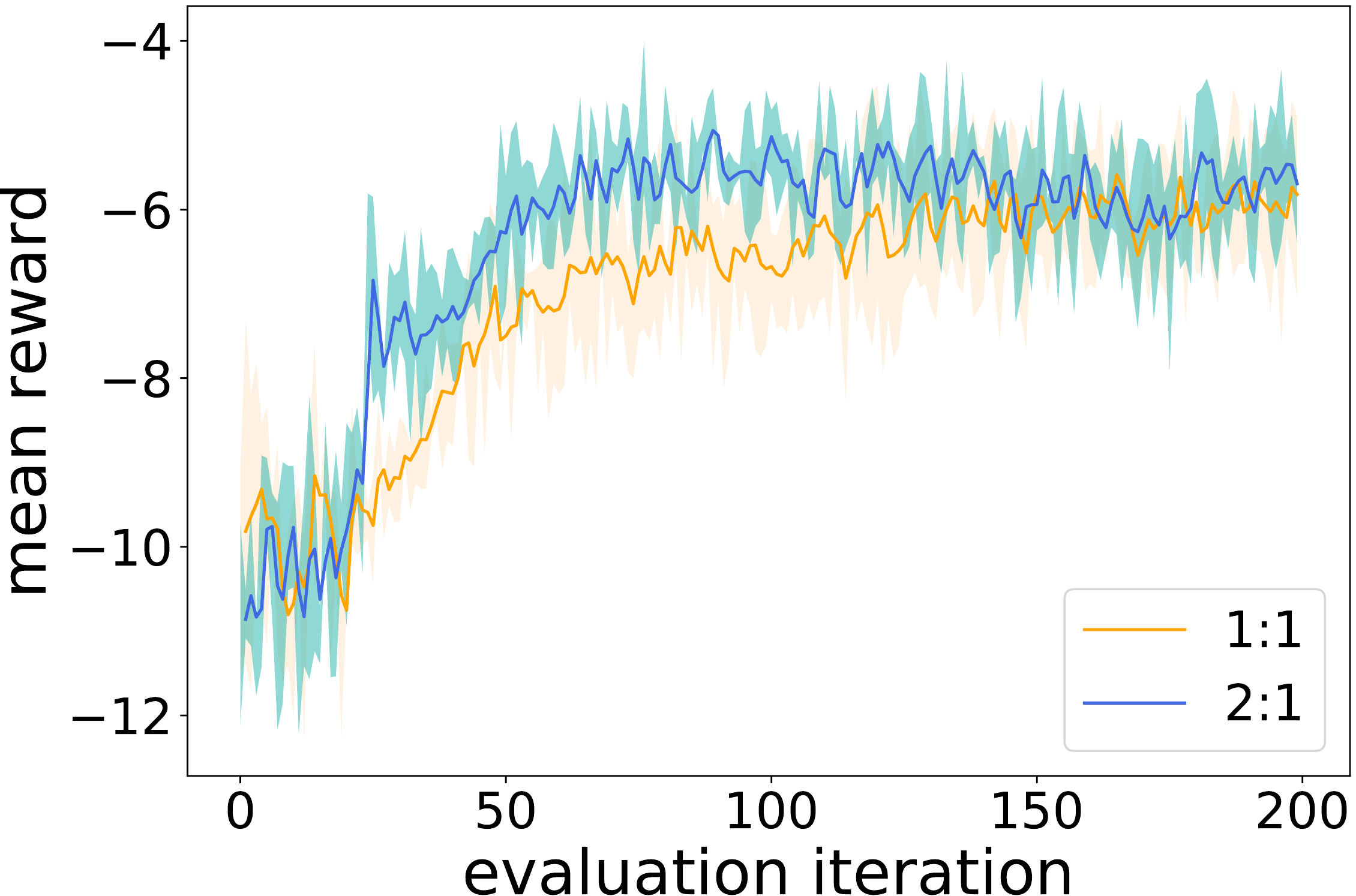}
    \end{subfigure}
    \hspace{.3em}
    \begin{subfigure}{.45\linewidth}
        \centering
        \includegraphics[width=\linewidth]{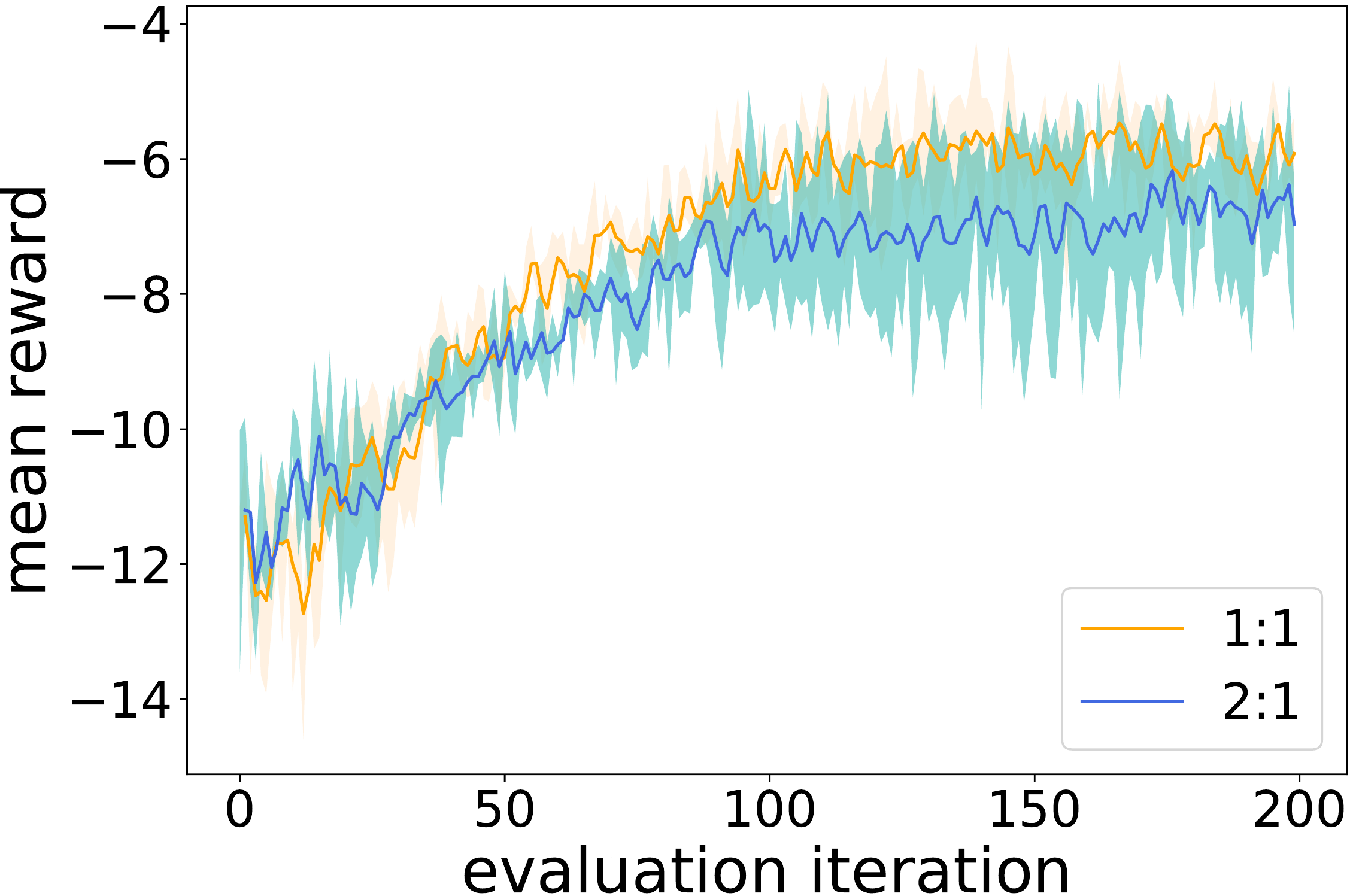}
    \end{subfigure}
    \caption{Evaluation performance for MuJoCo-Reacher, averaged across five runs. Left: Teacher task is similar to the target task. Right: Teacher task is dissimilar to the target task.}
    \label{fig:reacher_append}
\end{figure}

\begin{figure}[!t]
    \centering
    \includegraphics[width=.6\linewidth]{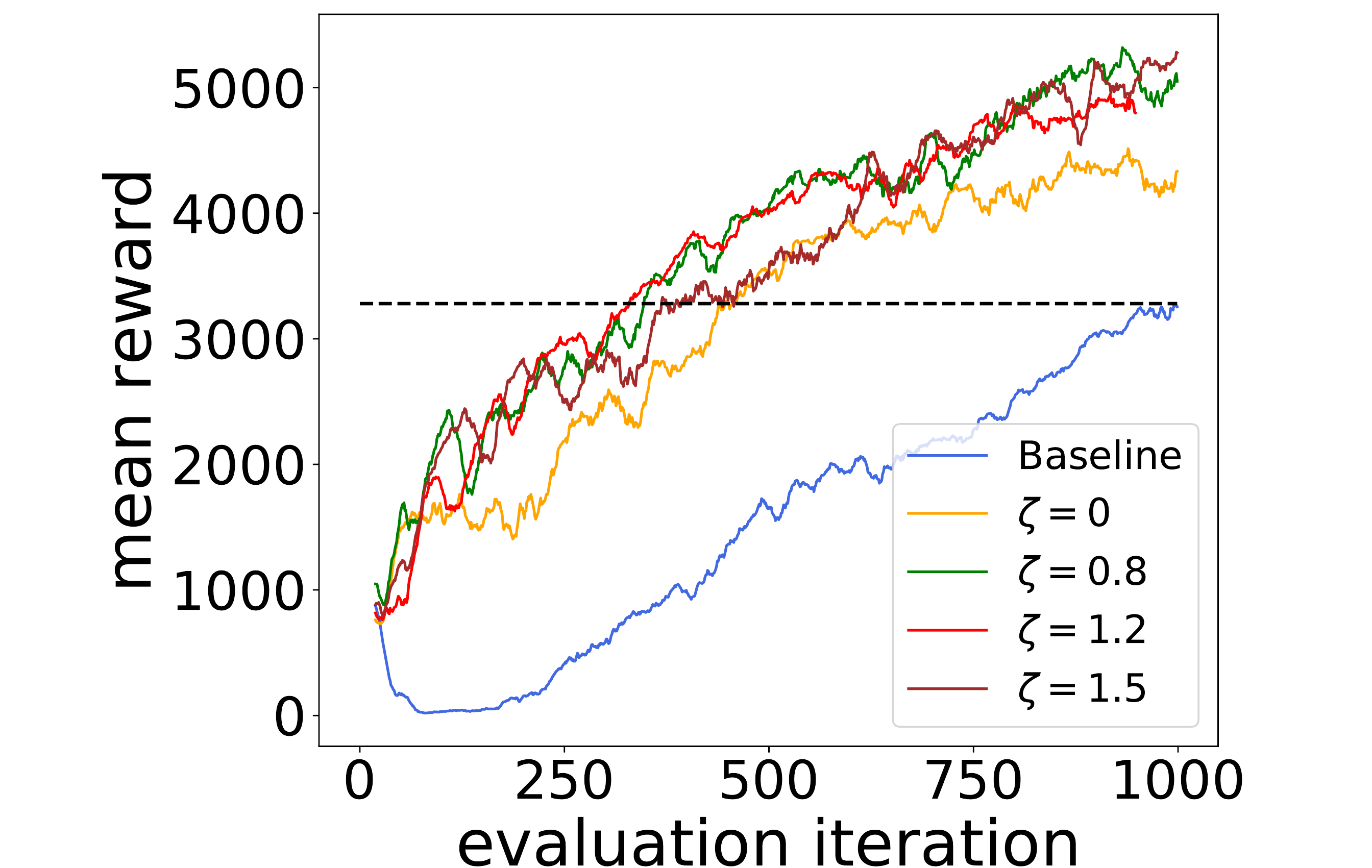}
    \caption{REPAINT performance for MuJoCo-Ant with different advantage thresholds in experience selection, averaged across three runs. Same teacher policy is used for all thresholds.}
    \label{fig:ant-zeta}
\end{figure}

\subsection{Advantage Threshold Robustness in MuJoCo-Ant}\label{sec:ant-threshold}
Figure~\ref{fig:ant-zeta} indicates that our REPAINT algorithm is robust to the threshold parameter when $\zeta>0$. Similar learning progresses are observed from training with different $\zeta$ values.

\subsection{More Results on DeepRacer Single-car Time Trial}\label{sec:single-car-exp}
In the DeepRacer single-car time-trial task, we also study the effect of different cross-entropy weights $\beta_k$ and instance filtering thresholds $\zeta$, as mentioned in the paper. We first present the results of instance transfer learning with different $\zeta$ values in Figure~\ref{fig:single-car-zeta}, where we can again see that our proposed advantage-based experience replay is robust to the threshold parameter.

\begin{figure}[!t]
    \centering
    \begin{subfigure}{.45\linewidth}
        \centering
        \includegraphics[width=\linewidth]{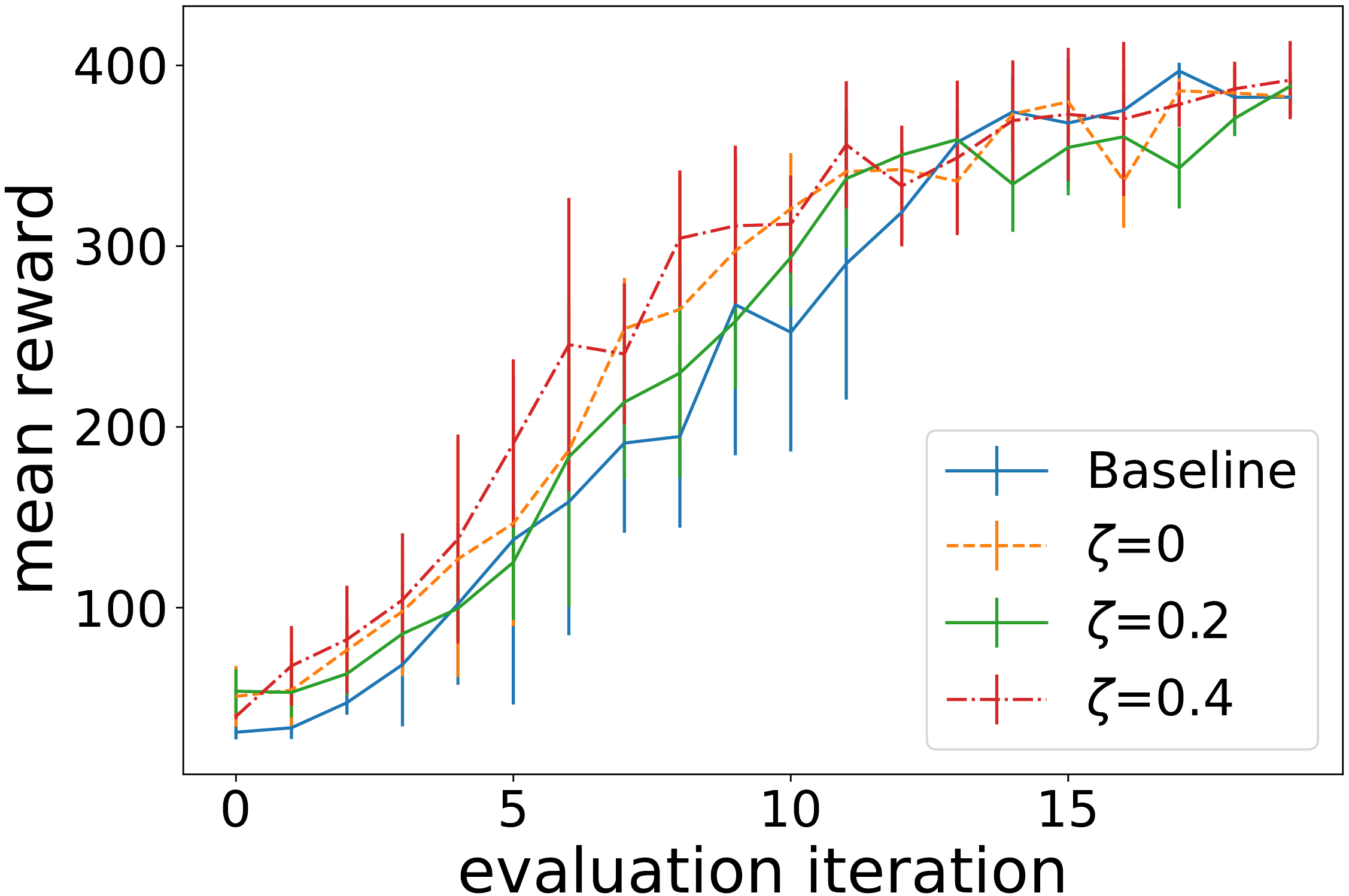}
        \caption{Outer-lane task with inner-lane teacher}
    \end{subfigure}
    \hspace{.3em}
    \begin{subfigure}{.45\linewidth}
        \centering
        \includegraphics[width=\linewidth]{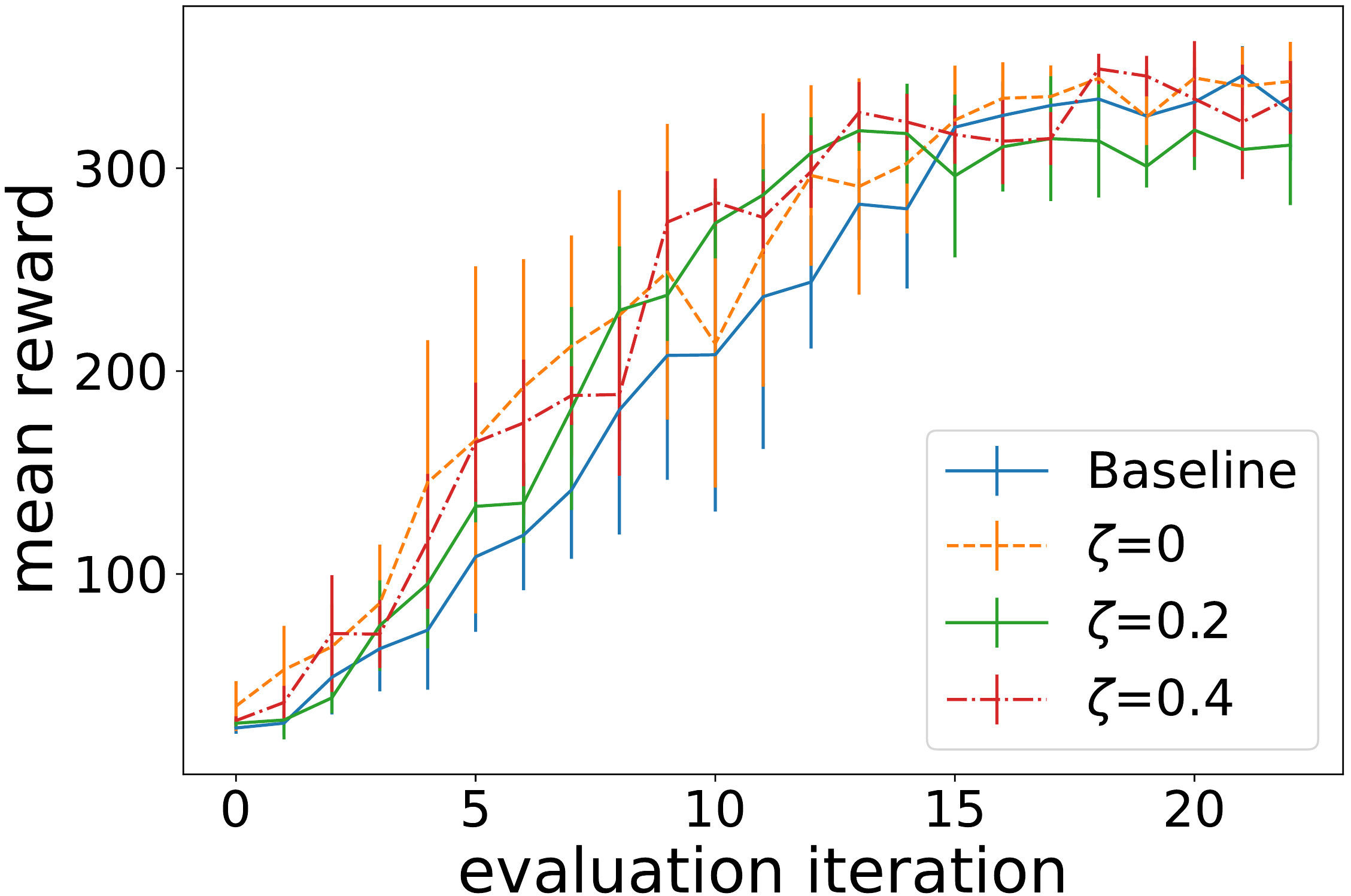}
        \caption{Inner-lane task with outer-lane teacher}
    \end{subfigure}
    \caption{Evaluation performance with respect to different $\zeta$'s, averaged across five runs.}
    \label{fig:single-car-zeta}
\end{figure}

\begin{figure}[!t]
    \centering
    \begin{subfigure}{.45\linewidth}
        \centering
        \includegraphics[width=\linewidth]{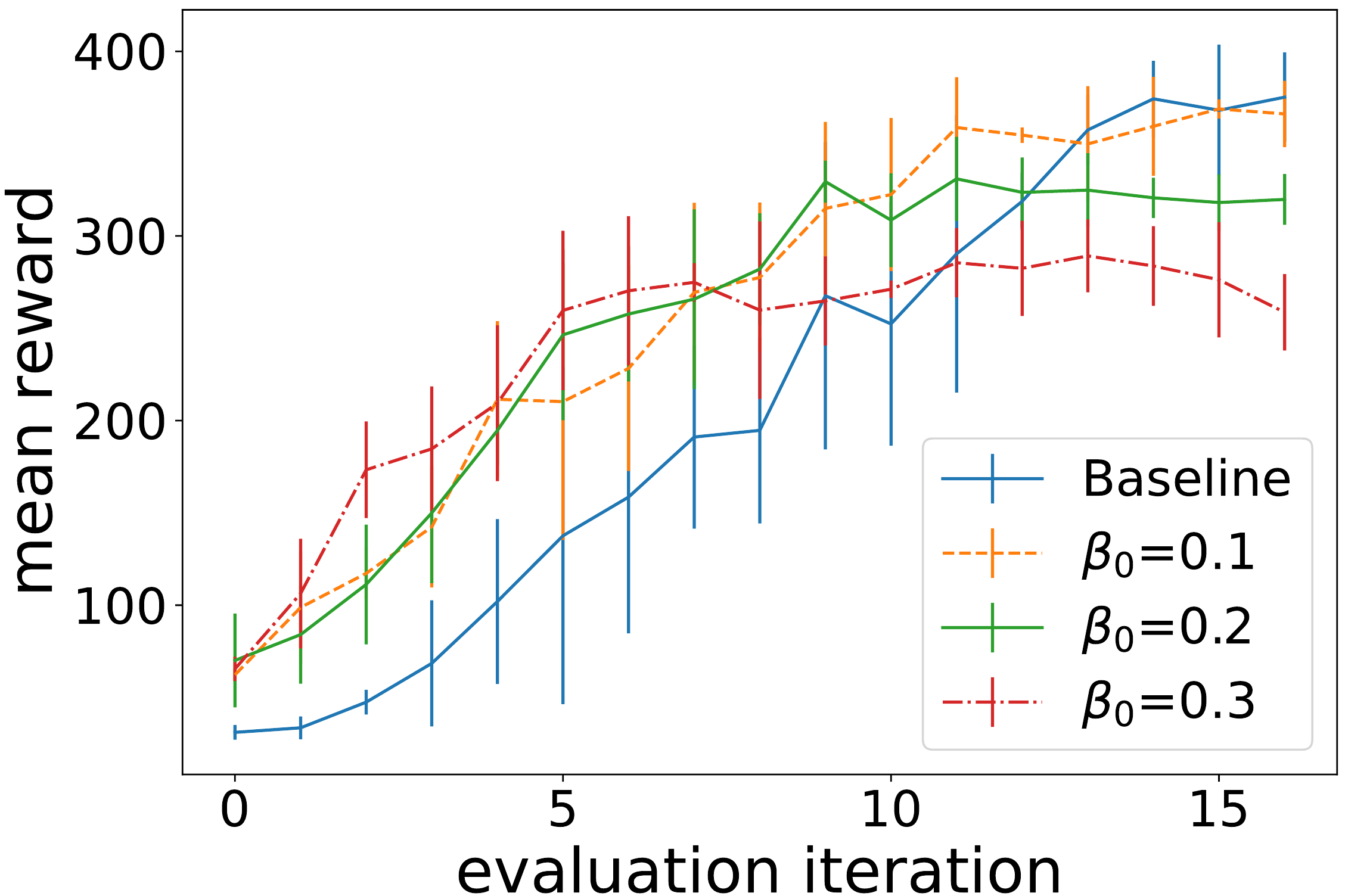}
        \caption{Outer-lane task with inner-lane teacher}
    \end{subfigure}
    \hspace{.3em}
    \begin{subfigure}{.45\linewidth}
        \centering
        \includegraphics[width=\linewidth]{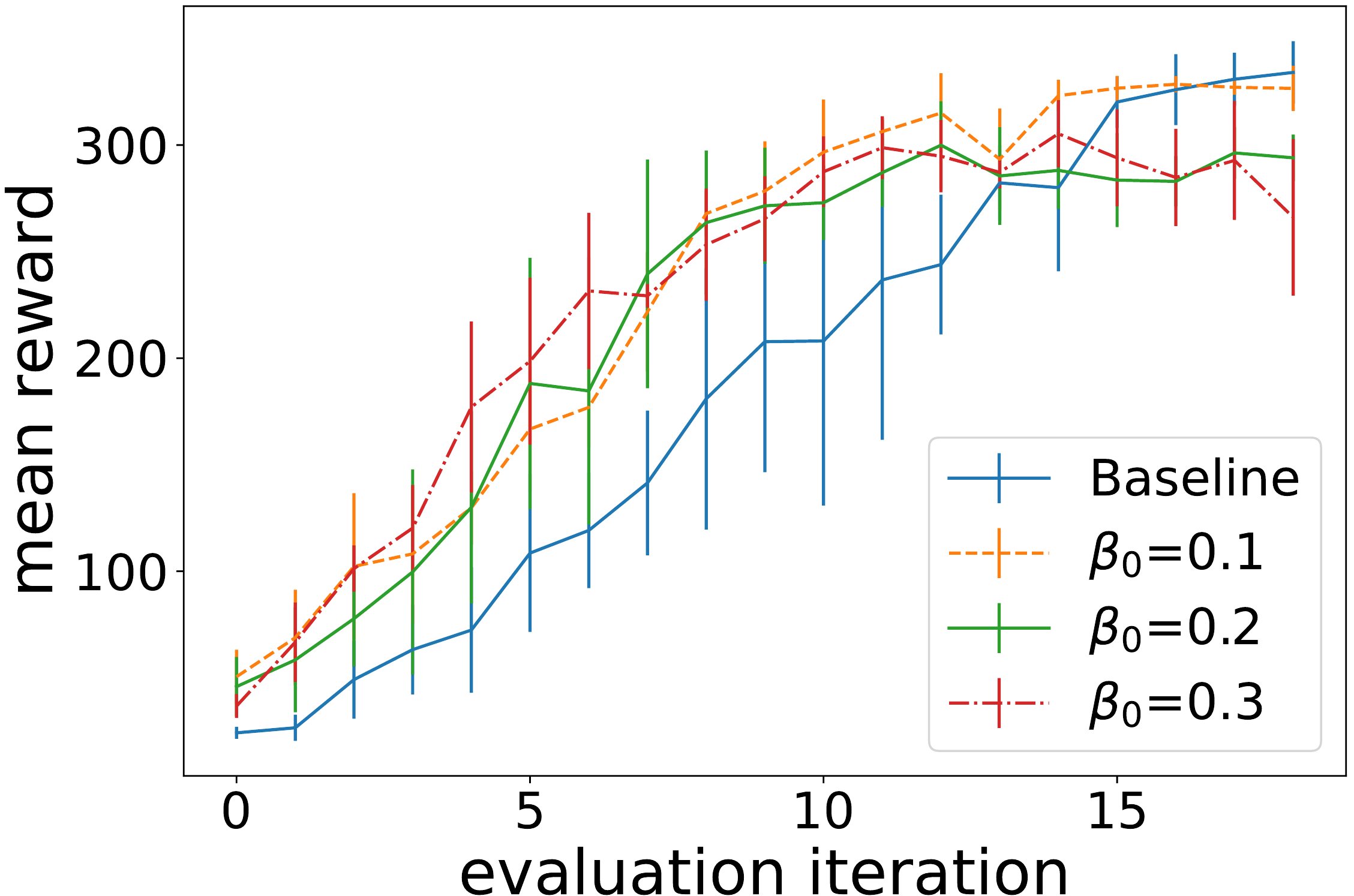}
        \caption{Inner-lane task with outer-lane teacher}
    \end{subfigure}
    \caption{Evaluation performance with respect to different initial $\beta_0$'s, averaged across five runs. Here we fix the $\beta$ update to be $\beta_{k+1}=0.95\beta_{k}$.}
    \label{fig:single-car-beta}
\end{figure}

\begin{figure}[!t]
    \centering
    \begin{subfigure}{.45\linewidth}
        \centering
        \includegraphics[width=\linewidth]{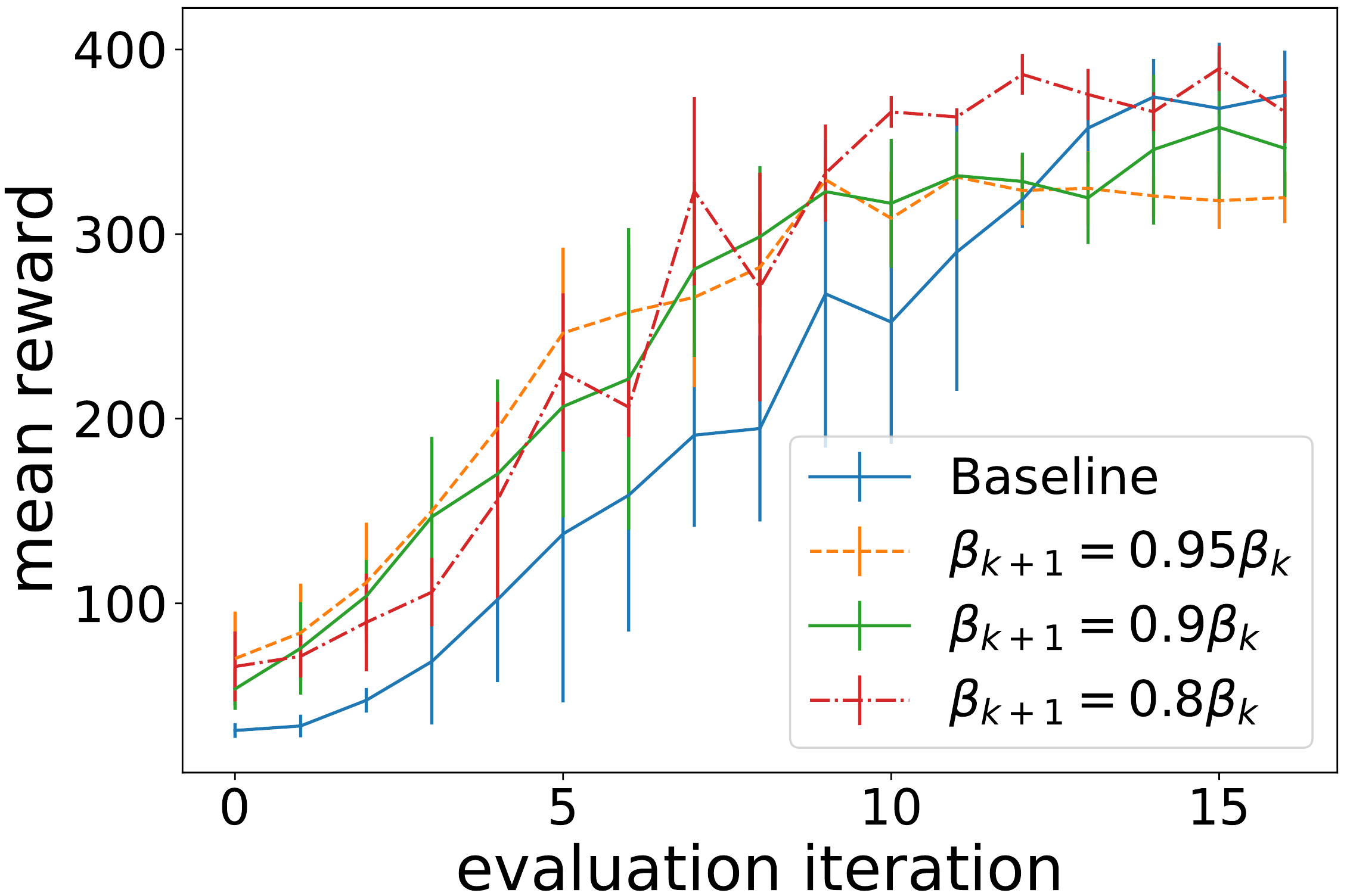}
        \caption{Outer-lane task with inner-lane teacher}
    \end{subfigure}
    \hspace{.3em}
    \begin{subfigure}{.45\linewidth}
        \centering
        \includegraphics[width=\linewidth]{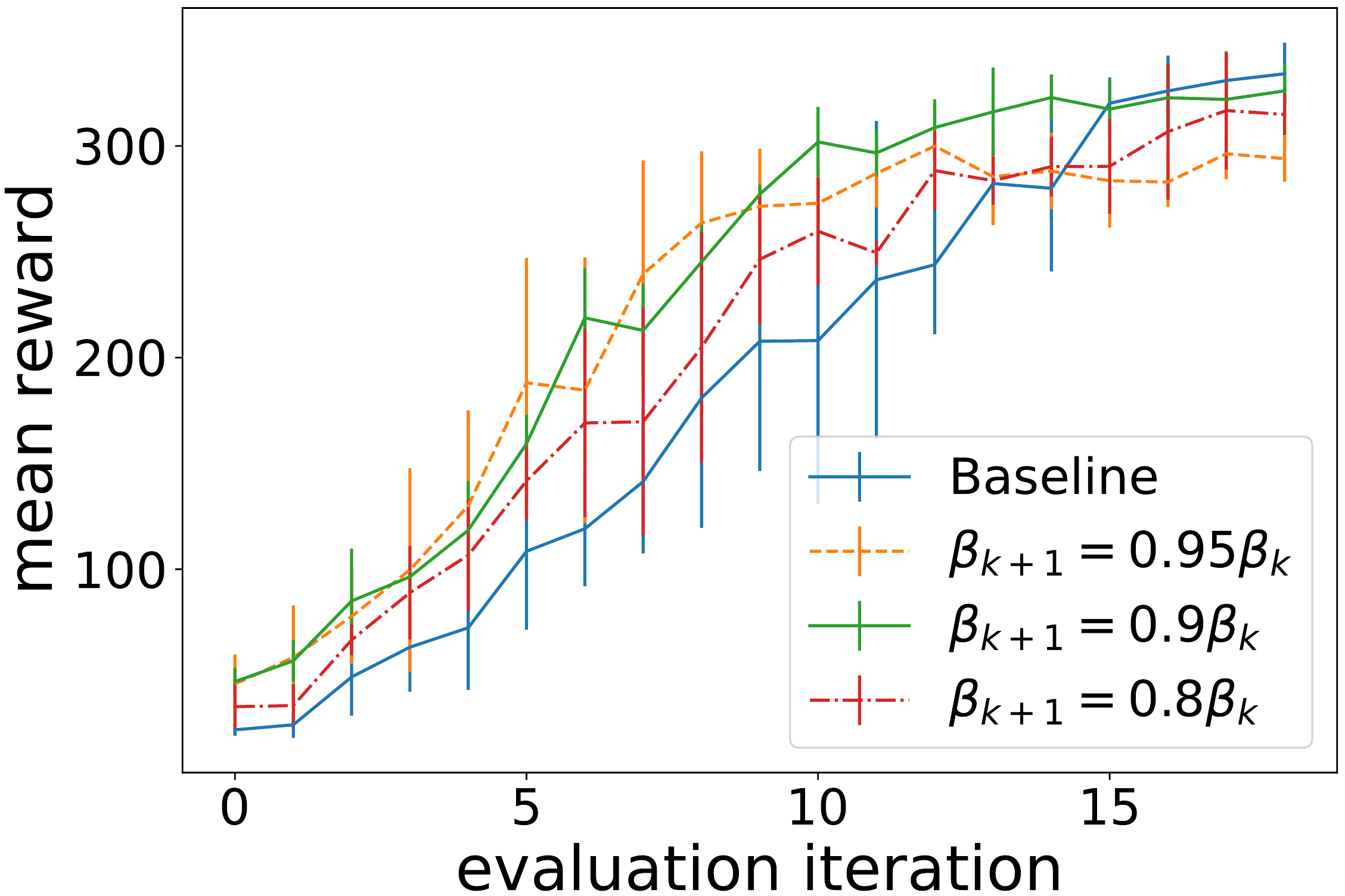}
        \caption{Inner-lane task with outer-lane teacher}
    \end{subfigure}
    \caption{Evaluation performance with respect to different $\beta$ schedules, averaged across five runs.}
    \label{fig:single-car-beta2}
\end{figure}

\begin{table*}[!t]
\caption{Summary of wall-clock time of experiments.}
\label{table:moresummary}
\begin{center}
\begin{small}
\begin{tabular}{cccccccc}
\toprule
\multirow{2}{*}{Env.} & Training & Teacher & Target & $\text{T}_\text{Baseline}$ & $\text{T}_\text{KS}$ & $\text{T}_\text{IT}$ & $\text{T}_\text{REPAINT}$ \\
& hardware & type & score & (hrs) & (pct. reduced) & (pct. reduced) & (pct. reduced) \\
\midrule
 \multirow{2}{*}{Reacher} & \multirow{2}{*}{laptop} & similar & \multirow{2}{*}{-7.4} & \multirow{2}{*}{$2.1$} & $0.6\ (71.4\%)$ & $1.1\ (47.6\%)$ & $0.4\ (81.0\%)$ \\
 &  & different & & & $0.9\ (57.1\%)$ & $1.4\ (33.3\%)$ & $0.6\ (71.4\%)$ \\
 \midrule
 Ant & laptop & similar & 3685 & $19.1$ & $8.0\ (58.1\%)$ & $12.8\ (33.0\%)$ & $7.5\ (60.7\%)$ \\
 \midrule
 \multirow{2}{*}{Single-car} & AWS, p2 & different & 394 & $2.2$ & Not achieved & Not achieved &$1.5\ (31.8\%)$ \\
 & AWS, p2 & different & 345 & $2.3$ & Not achieved & Not achieved &$1.5\ (34.8\%)$ \\
 \midrule
 \multirow{2}{*}{Multi-car} & AWS, p2 & sub-task & 1481 & $16.4$ & $4.8\ (70.7\%)$ & $12.6\ (23.2\%)$ & $4.5\ (72.6\%)$ \\
 & AWS, p2 & diff/sub-task & 2.7 & $9.6$ & $9.3\ (3.1\%)$ & $8.3\ (13.5\%)$ & $3.7\ (61.5\%)$ \\
\bottomrule
\end{tabular}
\end{small}
\end{center}
\end{table*}

We then study the performance of training with different cross-entropy loss weights $\beta_k$. First, we fix the diminishing factor to be 0.95, namely, $\beta_{k+1}=0.95\beta_{k}$, and test different $\beta_0$'s. From Figure~\ref{fig:single-car-beta}, we can see that training with all $\beta_0$ values can improve the initial performance compared to the baseline. However, when the teacher task is different from the target task, larger $\beta_0$ values, like 0.3, may reduce the agent's asymptotic performance since the agent overshoots learning from teacher policy. In addition, we then fix $\beta_0=0.2$ and test different $\beta_k$ schedules. The results are shown in Figure~\ref{fig:single-car-beta2}. We can observe some trade-offs between training convergence time and final performance. By reducing the $\beta$ values faster, one can improve the final performance but increase the training time that needed to achieve some certain performance level. It is of interest to automatically determine the best $\beta_k$ values during training, which needs further investigation. We leave it as another future work.

\begin{figure}[!t]
    \begin{subfigure}{.9\linewidth}
        \centering
        \includegraphics[width=\linewidth]{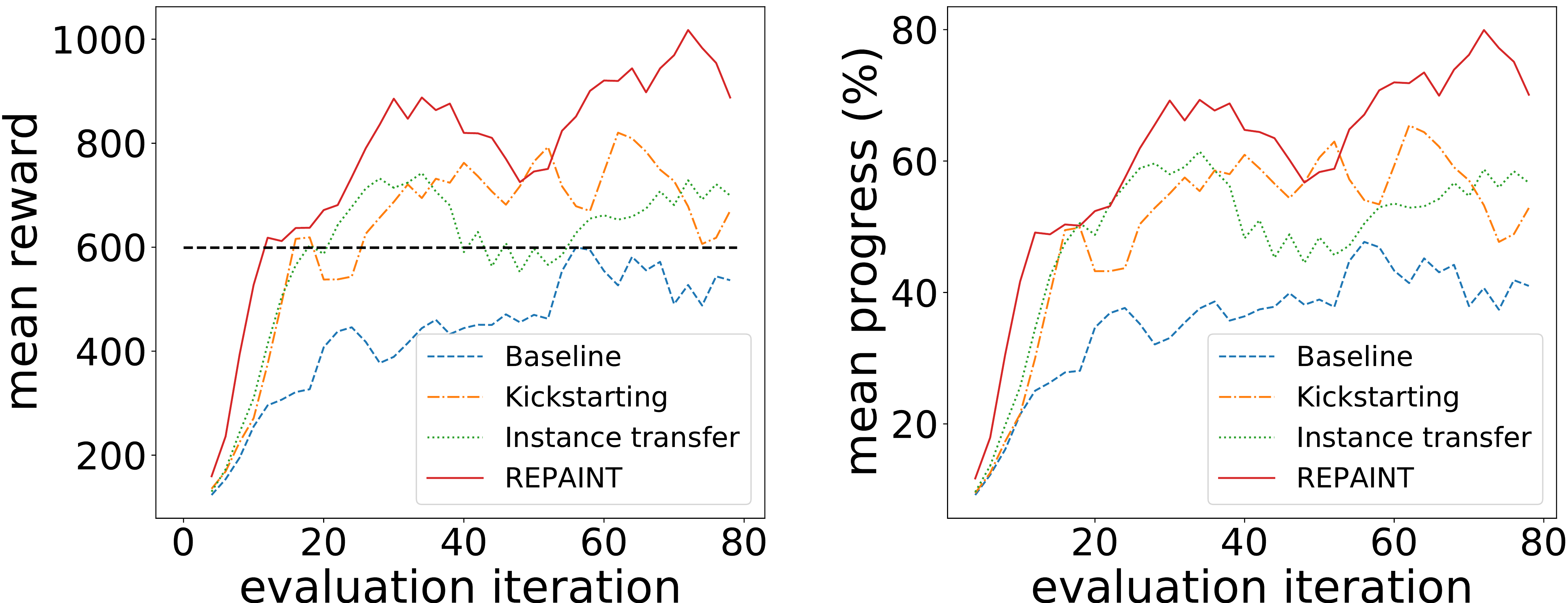}
        \caption{Task with advanced reward}
    \end{subfigure}
    \begin{subfigure}{.9\linewidth}
        \centering
        \includegraphics[width=\linewidth]{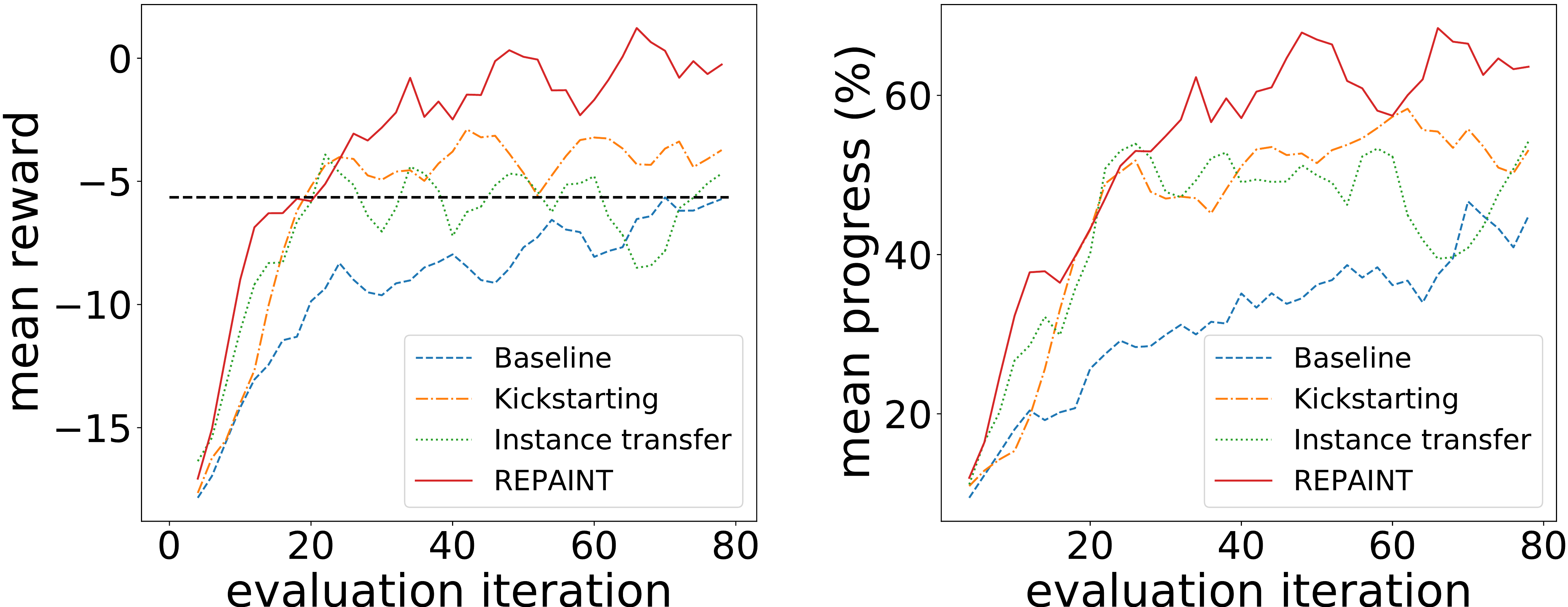}
        \caption{Task with progress-based reward}
    \end{subfigure}
    \caption{Evaluation performance for DeepRacer multi-car racing against bot cars, using 4-layer CNN. The plots are smoothed for visibility.}
    \label{fig:multi-car-normal}
\end{figure}

\subsection{Neural Network Architectures}\label{sec:more-exp-nn}
For completeness of the experiments, we also provide some results regarding different neural network architectures in this section. Take the DeepRacer task of multi-car racing against bot cars as an example, we have used three-layer CNN as the default architecture in experiments. Here, we present the comparison of REPAINT against other baselines with the evaluation performance using four-layer CNN (Figure~\ref{fig:multi-car-normal}) and five-layer CNN (Figure~\ref{fig:multi-car-deep}).

\begin{figure}[!t]
    \begin{subfigure}{.9\linewidth}
        \centering
        \includegraphics[width=\linewidth]{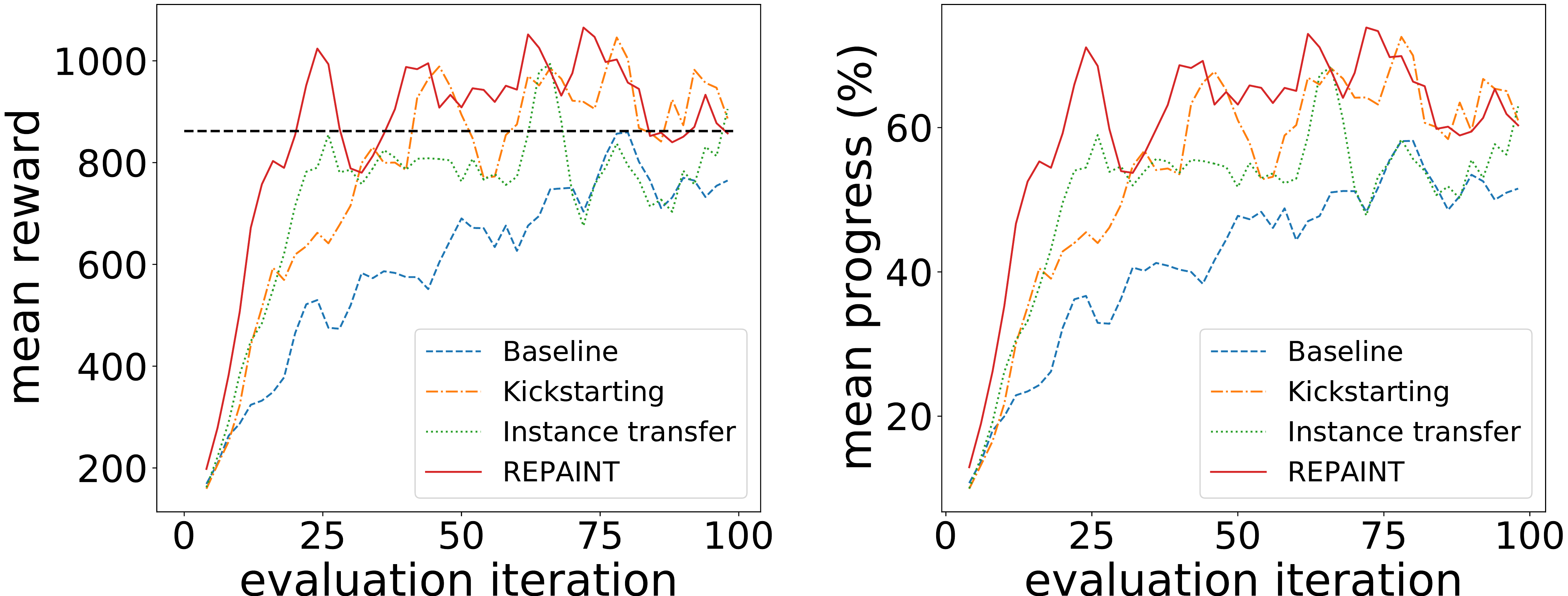}
        \caption{Task with advanced reward}
    \end{subfigure}
    \begin{subfigure}{.9\linewidth}
        \centering
        \includegraphics[width=\linewidth]{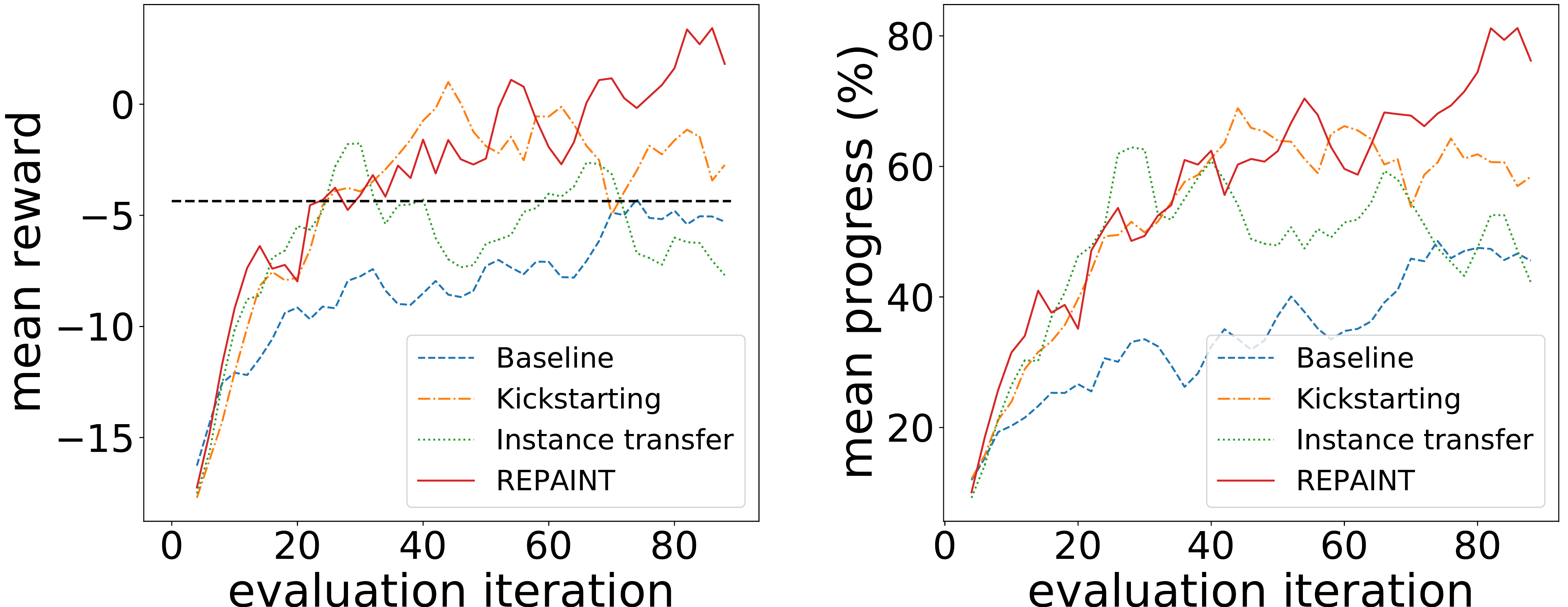}
        \caption{Task with progress-based reward}
    \end{subfigure}
    \caption{Evaluation performance for DeepRacer multi-car racing against bot cars, using 5-layer CNN. The plots are smoothed for visibility.}
    \label{fig:multi-car-deep}
\end{figure}

\subsection{Summary of Wall-Clock Training Time}\label{sec:moresum}
In addition to the summary of reduction performance with respect to number of training iterations presented in Table~\ref{table:summary}, we also provide the data of wall-clock time in Table~\ref{table:moresummary}. Note that we run StarCraft II experiments using different laptops, the comparison might not be convincing, and hence is omitted here. Again, we can see a significant reduction by training with REPAINT, which reaches at least 60\% besides the DeepRacer single-car time trial. The kickstarted training performs well when a similar teacher policy is used. Although training with only instance transfer cannot boost the initial performance, it still reduces the training cost to achieve some specific performance level.

\section{Convergence of Off-policy Instance Transfer}\label{sec:conv-appendix}
In order to apply the two time-scale stochastic approximation theory \cite{bhatnagar2009natural,karmakar2018two} for the convergence proof, the off-policy instance transfer learning is required to satisfy Assumptions (A1)-(A7) in \citet{holzleitner2020convergence}. We now discuss what assumptions we need to impose and how our instance transfer meets those properties.

First of all, regarding Assumptions (A1) and (A7), we can add some regularization terms in practice. For example, in our experiments for this paper, we have added weight decay, entropy regularization, and KL divergence terms.

Similar to \citet{holzleitner2020convergence}, in order to satisfy Assumptions (A2) and (A6), we need to make assumptions on the loss functions for actor and critic, i.e., Assumptions (L1)-(L3) in \citet{holzleitner2020convergence}. As mentioned before, the actor loss is denoted by $J_\text{ins}(\theta)$. Since the Q-function estimates in $J_\text{ins}$ involves the critic function, we denote the loss by $J_\text{ins}(\theta,\nu)$. We also denote the critic loss by $J_\text{critic}(\theta, \nu)$. Since the actor $\pi_\theta$ and the critic $V_\nu$ are approximated by deep neural networks, they are considered to be sufficiently smooth. Moreover, we should also assume sufficient smoothness for the two loss functions.
\begin{assumption}\label{ass:smooth1}
The loss functions $J_\text{ins}(\theta,\nu)$ and $J_\text{critic}(\theta,\nu)$ have compact support and are at least three times continuously differentiable with respect to $\theta$ and $\nu$. 
\end{assumption}
Next, for each starting point $(\theta_0, \nu_0)$, we want to find a neighborhood such that it contains only one critical point. Therefore, we further make the following two assumptions.
\begin{assumption}\label{ass:smooth2}
For each $\theta$, all critical points of $J_\text{critic}(\theta,\nu)$ are isolated local minima and there are only finitely many. The local minima $\{\lambda_i(\theta)\}_{i=1}^{k(\theta)}$ can be expressed locally as at least twice continuously differentiable functions with associated domains of definitions $\{W_{\lambda_i(\theta)}\}_{i=1}^{k(\theta)}$. 
\end{assumption}

\begin{assumption}\label{ass:smooth3}
Locally in $W_{\lambda_i(\theta)}$, $J_\text{ins}(\theta,\lambda_i(\theta))$ has only one local minimum.
\end{assumption}

Based on the above assumptions, for a fixed starting point $(\theta_0, \nu_0)$, we can construct a neighborhood $W_0\times U_0$, which contains unique local minimum. Assumption (A3) is not related to the instance transfer developed here, hence is omitted. We can either make the assumption explicitly for the update process, or follow the treatment mentioned in \citet{holzleitner2020convergence}, e.g., using online stochastic gradient descent (SGD) for update.

The next assumption we need to make is on the learning rates, i.e., Assumption (A4) in \citet{holzleitner2020convergence}. Denote the learning rates for actor and critic by $a_k$ and $b_k$, respectively. 
\begin{assumption}\label{ass:lr}
The learning rates $a_k$ and $b_k$ should satisfy:
\begin{align*}
\sum_k a_k = \infty,\quad \sum_k a_k^2 <\infty, \\
\sum_k b_k = \infty,\quad \sum_k b_k^2 <\infty, 
\end{align*}
and $\lim_{k\to\infty}a_k/b_k = 0$. Moreover, $a_k$ and $b_k$ are non-increasing for all $k\ge 0$.
\end{assumption}

At last, Assumption (A5) is satisfied as long as the transition kernels for the MDPs are continuous with respect to the weak topology in the space of probability measures. Therefore, after imposing Assumptions \ref{ass:smooth1}-\ref{ass:lr}, we can directly follow the two time-scale stochastic approximation theory \citep{karmakar2018two} and get that our proposed off-policy instance transfer can converge to some local optimum almost surely under the assumptions.

\section{Convergence Rate and Sample Complexity for REPAINT}\label{sec:rate-appendix}
The analysis and proof in this section is adapted from \citet{kumar2019sample}. Without loss of generality, we first assume that the teacher (source) task and student (target) task share the same state and action spaces $\mathcal{S}\times\mathcal{A}$. Then we make the following assumptions on the regularity of the student task and the parameterized student policy $\pi_\theta$.


\begin{assumption}
The reward function for student task is uniformly bounded. Namely, denote the reward by $R_\text{student}$. Then there exists a positive constant $U_\text{student}$, such that $R_\text{student}(s, a) \in [0, U_\text{student}]$ for any $(s, a) \in {\mathcal S} \times {\mathcal A}$. 
\end{assumption}

Since the policy (actor) is parameterized by neural networks, it is easy to see that $\pi_\theta$ is differentiable. In addition, we make an assumption on the corresponding score function.

\begin{assumption}
The score function $\nabla \log \pi_\theta (a|s)$ is Lipschitz continuous and has bounded norm, namely, for any $(s,a)\in\mathcal{S}\times\mathcal{A}$, there exist positive constants $L_\Theta$ and $B_\Theta$, such that
\begin{equation}
    \norm{\nabla \log \pi_{\theta_1} (a|s) - \nabla \log \pi_{\theta_2} (a|s)} \leq L_\Theta \|\theta_1-\theta_2\|, \forall \theta_1, \theta_2 ,
\end{equation}
and
\begin{equation}
    \norm{\nabla \log \pi_\theta (a|s)} \leq B_\Theta, \forall \theta.
\end{equation}
\end{assumption}

Note that by the above two assumptions, one can also obtain that the corresponding Q-function and objective function are also absolutely upper bounded. In order to prove our theorem, we also need the following i.i.d. assumption.

\begin{assumption}
In both teacher and student tasks, the random tuples $(s_t, a_t, s_t', a_t'), t=0, 1, \dots$ are drawn from the stationary distribution of the Markov reward process independently across time.
\end{assumption}

In practice, the i.i.d assumption does not hold \cite{dalal2017concentration}. But it is common when dealing with the convergence bounds in RL.

To ensure the Q-function evaluation and the stochastic estimate of the gradient unbiased, we consider the case where the Q-function admits a linear parameterization of the form $\hat{Q}^{\pi_{\theta}} (s, a) = \xi^T \varphi (s, a)$ where $\xi$ is a finite vector of real numbers of size $p$ and $\varphi: {\mathcal S} \times {\mathcal A} \rightarrow \mathbb{R}^p$ is a nonlinear feature map. In practice, we normalize the feature representation to guarantee the feature norm is bounded. Therefore, we can assume the norm boundedness of the feature map.

\begin{assumption}
For any state-action pair $(s, a) \in \mathcal{S} \times \mathcal{A}$, the norm of the student's feature representation $\varphi (s, a)$ is bounded by a constant $C_\text{student}$.
\end{assumption}

To simplify the proofs, we will consider that the experiences are not filtered
. The experience filtering results in possibly biased estimate of the gradient and impacts the variance bounds \cite{Greensmith2004variance}.
Next, we will assume that the update of the critic (Q-function) converges by some rate.
\begin{assumption}\label{assm:criticrate}
The expected error of the critic parameter for the student task is bounded by $O(k^{-b})$ for some $b\in (0, 1]$, i.e., there exists a positive constant $L_1$, such that
\begin{equation}
    \mathbb{E}(\norm{\xi_k - \xi_{\star}}) \leq L_1 k ^{-b}.
\end{equation}
\end{assumption}

Now we consider the update for actor in REPAINT. Assume that the learning rates $\alpha_1$ and $\alpha_2$ are also iteration dependent. Then we rewrite the actor update in Algorithm~\ref{alg:repaint} as
\begin{equation}
    \theta_{k+1} = \theta_k + \alpha_{1,k}\nabla_\theta J_\text{rep}(\theta_k) + \alpha_{2,k}\nabla_\theta J_\text{ins}(\theta_k). 
\end{equation}

For on-policy representation transfer, the gradient is defined by
\begin{equation}
    \nabla J_\text{rep}(\theta) = \nabla J_\text{RL}(\theta) - \beta_k \nabla J_\text{aux}(\theta).
\end{equation}
More specifically,
\begin{equation}
    \nabla J_\text{RL}(\theta) = \mathbb{E}_{\substack{s\sim d_\text{student}\\a\sim\pi_\text{student}}} [ \nabla \log \pi_{\theta} (a|s) Q^{\pi_{\theta}} (s, a) ],
\end{equation}
where $d_\text{student}$ is the limiting distribution of states under $\pi_\theta$, and
\begin{equation}
    \nabla J_\text{aux}(\theta) = \nabla H (\pi_\text{teacher} || \pi_{\theta})  = - \mathbb{E}_{\pi_\text{teacher}} [\nabla_{\theta} \log \pi_{\theta} (a|s)].
\end{equation}



For simplicity, we assume $\beta_k=0$ for all $k\ge 0$. Namely, we ignore the cross-entropy term in the proof of our theorem. However, we will present the extension of $\beta_k>0$ cases later. Then the stochastic estimate of the gradient is unbiased when the Q-function evaluation is unbiased, and is given by
\begin{equation}
    \hat{\nabla} J_\text{rep}(\theta) = \hat{Q}^{\pi_{\theta}} (s_T, a_T) \nabla \log \pi_{\theta} (a_T|s_T),
\end{equation}
where $s_T, a_T$ is the state-action pair collected following the student policy $\pi_\theta$ with some time step $T$.

The derivation of the instance transfer gradient estimate is similar to the off-policy actor-critic \cite{degris12:offpolicyAC}, 
which is defined as
\begin{equation}
    \hat{\nabla} J_\text{ins}(\theta) = \frac{\pi_{\theta} (\tilde{a}_{T}|\tilde{s}_{T})}{\pi_\text{teacher}} \hat{Q}^{\pi_{\theta}} (\tilde{s}_{T}, \tilde{a}_{T}) \nabla \log \pi_{\theta} (\tilde{a}_{T}|\tilde{s}_{T}),
\end{equation}
with some sample $\tilde{s}_T, \tilde{a}_T$ collected following the teacher policy $\pi_\text{teacher}$.




In summary, when updating the actor network. We collect rollouts following both teacher policy and student policy, and randomly select two samples for the following online update:
\begin{equation}
\begin{aligned}
&\theta_{k+1} - \theta_k = \alpha_{1, k} \hat{Q}^{\pi_{\theta}} (s_{T_k}, a_{T_k}) \nabla \log \pi_{\theta} (a_{T_k}|s_{T_k})  \\
& + \alpha_{2, k} \frac{\pi_{\theta} (\tilde{a}_{T_k}|\tilde{s}_{T_k})}{\pi_\text{teacher}(\tilde{a}_{T_k}|\tilde{s}_{T_k})} \hat{Q}^{\pi_{\theta}} (\tilde{s}_{T_k}, \tilde{a}_{T_k}) \nabla \log \pi_{\theta} (\tilde{a}_{T_k}|\tilde{s}_{T_k}).  
\end{aligned}
\end{equation}

Next, we assume that the estimate of the objectives' gradient conditioned on some filtration is bounded by some finite variance.
\begin{assumption}
Let $\hat{\nabla} L_\text{rep} (\theta)$ and $\hat{\nabla} L_\text{ins} (\theta)$ be the estimators of $\nabla L_\text{rep} (\theta)$ and $\nabla L_\text{ins} (\theta)$, respectively. Then, there exist finite $\sigma_\text{rep}$ and $\sigma_\text{ins}$ such that
\begin{eqnarray}
    &\mathbb{E}(\norm{\hat{\nabla} L_\text{rep} (\theta) }^2| {\mathcal F}_k) \leq \frac{\sigma_\text{rep}^2}{4},&\\ 
    &\mathbb{E}(\norm{\hat{\nabla} L_\text{ins} (\theta) }^2| {\mathcal F}_k) \leq \frac{\sigma_\text{ins}^2}{4}.&
\end{eqnarray}
\end{assumption}

Since the teacher policy is a deterministic policy distribution. It is also common to assume some boundedness for $\pi_\text{teacher}$ \cite{degris12:offpolicyAC}.
\begin{assumption}
The teacher policy has a minimum positive value $b_\text{min} \in (0, 1]$, such that $\pi_\text{teacher} (a|s) \geq b_\text{min}$ for all $(s, a) \in \mathcal{S} \times \mathcal{A}$.
\end{assumption}


Now we have stated all assumptions that are needed for deriving the convergence rate and sample complexity. Next, we introduce the proofs of two lemmas. The first lemma is on the Lipschitz continuity of the objective gradients. The proof can be found in, e.g., \citet{zhang2020global}.

\begin{lemma}
The objective gradients $\nabla J_\text{rep}$ and $\nabla J_\text{ins}$ are Lipschitz continuous, namely, there exist constants $L_\text{rep}$ and $L_\text{ins}$, such that for any $\theta_1$ and $\theta_2$,
\begin{eqnarray}
    &\norm{\nabla J_\text{rep} (\theta_1) - \nabla J_\text{rep} (\theta_2)} \leq  L_\text{rep} \norm{\theta_1 - \theta_2},&\\
    &\norm{\nabla J_\text{ins} (\theta_1) - \nabla J_\text{ins} (\theta_2)} \leq  L_\text{ins} \norm{\theta_1 - \theta_2}.&
\end{eqnarray}
\end{lemma}
For simplicity, we can let $L:=\max(L_\text{rep},L_\text{ins})$, so $L$ is the Lipschitz constant for both inequalities above. Next we will derive an approximate ascent lemma for a random variable $W_k$ defined by
\begin{equation}
\begin{aligned}
    W_k= &J_\text{rep}(\theta_k) + J_\text{ins}(\theta_k)  \\
    &- L \left( \sigma_\text{rep}^2 \sum_{j=k}^{\infty} \alpha_{1, j}^2 + \sigma_\text{ins}^2 \sum_{j=k}^{\infty} \alpha_{2, j}^2 \right).
\end{aligned}
\end{equation}

Since the rewards and score functions are bounded above (see Assumptions E.1 and E.2), then we can also get there exist constants $C_\text{rep}$ and $C_\text{ins}$, such that 
\begin{equation}
    \norm{\nabla J_\text{rep}} \leq C_\text{rep} \quad \text{and}\quad \norm{\nabla J_\text{ins}} \leq C_\text{ins}.
\end{equation}

\begin{lemma} \label{lemma:Wk}
The sequence ${W_k}$ defined above satisfies the inequality
\begin{multline}
\mathbb{E} [W_{k+1}|{\mathcal F}_k] \geq W_k \\
- (C_\text{rep}+C_\text{ins}) C_\text{student} B_{\Theta} (\alpha_{1, k} + \frac{\alpha_{2, k}}{b_\text{min}} ) \mathbb{E}[\norm{\xi_k - \xi_*}| {\mathcal F}_k]  \\ + \alpha_{1, k} \norm{\nabla J_\text{rep} (\theta_k)}^2   + \alpha_{2, k} \norm{\nabla J_\text{ins} (\theta_k)}^2  \\ + (\alpha_{1, k}+\alpha_{2, k}) \nabla J_\text{rep} (\theta_k)^\top \nabla J_\text{ins} (\theta_k)  
\end{multline}
\end{lemma}

\begin{proof}
By definition, we can write
\begin{equation}
\begin{aligned}
    W_{k+1}= &J_\text{rep}(\theta_{k+1}) + J_\text{ins}(\theta_{k+1})  \\
    &- L \left( \sigma_\text{rep}^2 \sum_{j=k+1}^{\infty} \alpha_{1, j}^2 + \sigma_\text{ins}^2 \sum_{j=k+1}^{\infty} \alpha_{2, j}^2 \right).
\end{aligned}
\end{equation}
By the Mean Value Theorem, there exists $\tilde{\theta}_k\in [\theta_k, \theta_{k+1}]$, such that
\begin{equation}
    J_\text{rep}(\theta_{k+1}) = J_\text{rep}(\theta_k) + (\theta_{k+1}-\theta_k)^\top \nabla J_\text{rep}(\tilde{\theta}_k).
\end{equation}
By Cauchy Schwartz inequality, we have
\begin{equation}
\begin{aligned}
    &(\theta_{k+1}-\theta_k)^\top(\nabla J_\text{rep}(\tilde{\theta}_k)-\nabla J_\text{rep}(\theta_k)) \\
    \ge & -\|\theta_{k+1}-\theta_k\|\|\nabla J_\text{rep}(\tilde{\theta}_k)-\nabla J_\text{rep}(\theta_k)\| \\
    \ge & -L_\text{rep} \|\theta_{k+1}-\theta_k\|^2 \\
    \ge & -L \|\theta_{k+1}-\theta_k\|^2.
\end{aligned}
\end{equation}
After similar treatment for $J_\text{ins}(\theta)$, we can get
\begin{equation}
\begin{aligned}
W_{k+1} \ge &W_k + (\theta_{k+1}-\theta_k)^\top(\nabla J_\text{rep}(\theta_k) + J_\text{ins}(\theta_k))\\
&- 2L\|\theta_{k+1}-\theta_{k}\|^2.
\end{aligned}
\end{equation}

Take the expectation with respect to the filtration $\mathcal{F}_k$ and substitue the definition for the actor updat. Since
\begin{equation}
\begin{aligned}
    &\mathbb{E}[\|\theta_{k+1}-\theta_k\|^2|{\mathcal F}_k] \\ &=\mathbb{E}[\|\alpha_{1,k}\hat{\nabla}J_\text{rep}(\theta_k)+\alpha_{2,k}\hat{\nabla}J_\text{ins}(\theta_k)\|^2|{\mathcal F}_k] \\
    & \le 2(\mathbb{E}[\|\alpha_{1,k}\hat{\nabla}J_\text{rep}(\theta_k)\|^2|{\mathcal F}_k] + \mathbb{E}[\|\alpha_{2,k}\hat{\nabla}J_\text{ins}(\theta_k)\|^2|{\mathcal F}_k]) \\
    & \le \frac{1}{2}(\alpha_{1,k}^2\sigma_\text{rep}^2 + \alpha_{2,k}^2\sigma_\text{ins}^2),
\end{aligned}
\end{equation}
we can get
\begin{equation}
\begin{aligned}
    \mathbb{E} [W_{k+1}|{\mathcal F}_k]  \geq & W_k +  \mathbb{E}[\theta_{k+1} - \theta_k |{\mathcal F}_k]^\top \nabla J_\text{rep}(\theta_k) \\
    & + \mathbb{E}[\theta_{k+1} - \theta_k |{\mathcal F}_k]^\top \nabla J_\text{ins}(\theta_k).
\end{aligned}
\end{equation}



Plug in the linear parameterized Q-function to the actor update, we can get
\begin{equation}
\begin{aligned}
\theta_{k+1} &- \theta_k = \alpha_{1, k} \xi_k^T \varphi (s_{T_k}, a_{T_k}) \nabla \log \pi_{\theta} (a_{T_k}|s_{T_k})  \\
& + \alpha_{2, k} \frac{\pi_{\theta} (\tilde{a}_{T_k}|\tilde{s}_{T_k})}{\pi_\text{teacher}}  \xi_k^T \varphi (\tilde{s}_{T_k}, \tilde{a}_{T_k}) \nabla \log \pi_{\theta} (\tilde{a}_{T_k}|\tilde{s}_{T_k}).
\end{aligned}
\end{equation}

To simplify the notation, let's denote
\begin{equation*}
    Z_\text{rep}^k(\theta) = \alpha_{1, k} (\xi_k^T - \xi_*) \varphi (s_{T_k}, a_{T_k}) \nabla \log \pi_{\theta} (a_{T_k}|s_{T_k}), 
\end{equation*}
\begin{equation*}
    Z_\text{ins}^k(\theta) = \alpha_{2, k} (\xi_k^T - \xi_*) \varphi (\tilde{s}_{T_k}, \tilde{a}_{T_k}) \nabla \log \pi_{\theta} (\tilde{a}_{T_k}|\tilde{s}_{T_k}).
\end{equation*}

Take expectation conditioned on the filtration and get
\begin{equation}
\begin{aligned}
\mathbb{E}[\theta_{k+1} & - \theta_k | {\mathcal F}_k] =  \mathbb{E}[Z_\text{rep}^k(\theta)| {\mathcal F}_k] + \alpha_{1, k} \nabla J_\text{rep} (\theta_k) \\
&  + \mathbb{E}[Z_\text{ins}^k(\theta, k) \frac{\pi_{\theta} (\tilde{a}_{T_k}|\tilde{s}_{T_k})}{\pi_\text{teacher}}| {\mathcal F}_k]  + \alpha_{2, k} \nabla J_\text{ins} (\theta_k)
\end{aligned}
\end{equation}

Then on both sides, take the inner product with $\nabla J_\text{rep} (\theta_k)$: 
\begin{equation}
\begin{aligned}
\mathbb{E}[\theta_{k+1} - \theta_k &| {\mathcal F}_k]^\top \nabla J_\text{rep} (\theta_k) =  \mathbb{E}[Z_\text{rep}^k(\theta)| {\mathcal F}_k]^\top \nabla J_\text{rep} (\theta_k)  \\ 
& + \alpha_{1, k} \|\nabla J_\text{rep} (\theta_k)\|^2 \\
& + \mathbb{E}[Z_\text{ins}^k(\theta) \frac{\pi_{\theta} (\tilde{a}_{T_k}|\tilde{s}_{T_k})}{\pi_\text{teacher}}| {\mathcal F}_k]^\top \nabla J_\text{rep} (\theta_k)  \\
& + \alpha_{2, k} \nabla J_\text{ins} (\theta_k)^\top \nabla J_\text{rep} (\theta_k) \\
\ge & -|\mathbb{E}[Z_\text{rep}^k(\theta)| {\mathcal F}_k]^\top \nabla J_\text{rep} (\theta_k)| \\
& + \alpha_{1, k} \|\nabla J_\text{rep} (\theta_k)\|^2 \\
& - |\mathbb{E}[Z_\text{ins}^k(\theta) \frac{\pi_{\theta} (\tilde{a}_{T_k}|\tilde{s}_{T_k})}{\pi_\text{teacher}}| {\mathcal F}_k]^\top \nabla J_\text{rep} (\theta_k)|  \\
& + \alpha_{2, k} \nabla J_\text{ins} (\theta_k)^\top \nabla J_\text{rep} (\theta_k).
\end{aligned}
\end{equation}




By the assumptions, we can get following bounds. 
\begin{equation}
\begin{aligned}
    \norm{\varphi (s_{T_k}, a_{T_k})} \cdot \norm{\nabla \log \pi_{\theta} (a_{T_k}|s_{T_k})} \cdot &\norm{\nabla J_\text{rep}(\theta_k)}  \\ 
     \leq C_\text{student} B_{\Theta}  C_\text{rep},
\end{aligned}
\end{equation}
\begin{eqnarray}
    & & \norm{\frac{\pi_{\theta} (\tilde{a}_{T_k}|\tilde{s}_{T_k})}{\pi_\text{teacher}}} \cdot \norm{ \varphi (\tilde{s}_{T_k}, \tilde{a}_{T_k})} \cdot \norm{ \nabla \log \pi_{\theta} (\tilde{a}_{T_k}|\tilde{s}_{T_k})}  \nonumber \\ & & \cdot \norm{\nabla J_\text{rep}(\theta_k)}  \leq C_\text{student} B_{\Theta}  C_\text{rep} / b_\text{min}.
\end{eqnarray}

Therefore, replace the bounds and we can get
\begin{eqnarray}
& & \mathbb{E}[\theta_{k+1} - \theta_k | {\mathcal F}_k]^\top \nabla J_\text{rep} (\theta_k) \geq  \nonumber \\
& & - C_\text{student} B_{\Theta}  C_\text{rep} \alpha_{1, k} \mathbb{E}[\norm{\xi_k - \xi_*}| {\mathcal F}_k] \nonumber \\ 
& & + \alpha_{1, k} \norm{\nabla J_\text{rep} (\theta_k)}^2 \nonumber \\
& & - C_\text{student} B_{\Theta}  \frac{C_\text{rep}}{b_{min}} \alpha_{2, k} \mathbb{E}[\norm{\xi_k - \xi_*}| {\mathcal F}_k]\nonumber \\
& &   + \alpha_{2, k} \nabla J_\text{ins} (\theta_k)^T \nabla J_{rep} (\theta_k). 
\end{eqnarray}

Similarly, for the objectives corresponding to the instance transfer, we can get
\begin{eqnarray}
& & \mathbb{E}[\theta_{k+1} - \theta_k | {\mathcal F}_k]^T \nabla J_\text{ins} (\theta_k) \geq  \nonumber \\
& & - C_\text{student} B_{\Theta} C_\text{ins} \alpha_{1, k} \mathbb{E}[\norm{\xi_k - \xi_*}| {\mathcal F}_k] \nonumber \\ 
& & + \alpha_{1, k} \nabla J_\text{rep} (\theta_k)^\top \nabla J_\text{ins} (\theta_k) \nonumber \\
& & - C_\text{student} B_{\Theta} \frac{C_\text{ins}}{b_\text{min}} \alpha_{2, k} \mathbb{E}[\norm{\xi_k - \xi_*}| {\mathcal F}_k]\nonumber \\
& &   + \alpha_{2, k} \norm{\nabla J_\text{ins} (\theta_k)}^2. 
\end{eqnarray}

Now we add them together. Let $C = C_\text{rep} + C_\text{ins}$, then
\begin{multline}
\mathbb{E} [W_{k+1}|{\mathcal F}_k] \geq W_k \\
- CC_\text{student} B_{\Theta} (\alpha_{1, k} + \frac{\alpha_{2, k}}{b_\text{min}} ) \mathbb{E}[\norm{\xi_k - \xi_*}| {\mathcal F}_k]  \\ + \alpha_{1, k} \norm{\nabla J_\text{rep} (\theta_k)}^2   + \alpha_{2, k} \norm{\nabla J_\text{ins} (\theta_k)}^2  \\ + (\alpha_{1, k}+\alpha_{2, k}) \nabla J_\text{rep} (\theta_k)^\top \nabla J_\text{ins} (\theta_k).  
\end{multline}
\end{proof}

We now present the main result which is the convergence rate of Q-value-based REPAINT. 
Let $K_{\epsilon}$ be the smallest number of updates $k$ required to attain a function gradient smaller than $\epsilon$,
\begin{equation}
    K_{\epsilon} = \min \{k : \underset{0\leq m\leq k}{\inf} \mathcal{F}(\theta_m) < \epsilon \},
\end{equation}
where $A_k = \alpha_{2, k}/\alpha_{1, k}$ and 
\begin{multline}
    \mathcal{F}(\theta_m) = \norm{\nabla J_\text{rep} (\theta_m)}^2 + A \norm{\nabla J_\text{ins} (\theta_m)}^2 \\ + (1+A_k) \nabla J_\text{rep}(\theta_m)^\top \nabla J_\text{ins} (\theta_m).
\end{multline}

\begin{theorem}\label{thm:convergencerate}
Suppose the representation transfer step size satisfies $\alpha_{1, k} = k^{-a}$ for $a>0$ and the critic update satisfies Assumption~\ref{assm:criticrate}. The instance transfer step size satisfies $\alpha_{2, k} = A_k \alpha_{1, k}$ for $A_k \in \mathbb{R}^+$. When the critic bias converges to null as $\mathcal{O}(k^{-1})$ ($b=1$), then $T_C(k)=k+1$ critic updates occur per actor update. Alternatively, if the critic bias converges to null more slowly as $\mathcal{O}(k^{-b})$ with $b\in (0, 1)$ in Assumption~\ref{assm:criticrate}, then $T_C(k)=k$ critic updates per actor update are chosen. Then the actor sequence defined in Algorithm~\ref{alg:repaint} satisfies
\begin{equation}
    K_{\epsilon} \leq \mathcal{O} (\epsilon^{-1/l}),
\end{equation}
where $l = \min\{a, 1-a, b\}$. Moreover, minimizing over $a$, the resulting sample complexity depends on the attenuation $b$ of the critic bias as
\begin{equation}
    K_{\epsilon} \leq \left\{
    \begin{array}{cc}
        \mathcal{O} (\epsilon^{-1/b}) & b\in (0, 1/2) \\
        \mathcal{O} (\epsilon^{-2}) & b\in (1/2, 1] 
    \end{array}   
    \right.
\end{equation}
\end{theorem}

\begin{proof}
Substitute for $W_k$ in Lemma~\ref{lemma:Wk},
\begin{multline}
\mathbb{E}[J_\text{rep}(\theta_{k+1})|{\mathcal F}_k]  + \mathbb{E}[J_\text{ins}(\theta_{k+1})|{\mathcal F}_k] \\
- L ( \sigma_\text{rep}^2 \sum_{j=k+1}^{\infty} \alpha_{1, j}^2 + \sigma_\text{ins}^2 \sum_{j=k+1}^{\infty} \alpha_{2, j}^2 ) \\ \geq  J_\text{rep}(\theta_k) + J_\text{ins}(\theta_k) \\ - L ( \sigma_\text{rep}^2 \sum_{j=k}^{\infty} \alpha_{1, j}^2 + \sigma_\text{ins}^2 \sum_{j=k}^{\infty} \alpha_{2, j}^2 ) \\
- CC_\text{student} B_{\Theta} (\alpha_{1, k} + \frac{\alpha_{2, k}}{b_\text{min}} ) \mathbb{E}[\norm{\xi_k - \xi_*}| {\mathcal F}_k]  \\ + \alpha_{1, k} \norm{\nabla J_\text{rep} (\theta_k)}^2   + \alpha_{2, k} \norm{\nabla J_\text{ins} (\theta_k)}^2  \\ + (\alpha_{1, k}+\alpha_{2, k}) \nabla J_\text{rep} (\theta_k)^\top \nabla J_\text{ins} (\theta_k).
\end{multline}

Cancel some common terms from both sides and take the total expectation, we can get
\begin{multline}
\mathbb{E}[J_\text{rep}(\theta_{k+1})]  + \mathbb{E}[J_\text{ins}(\theta_{k+1})] \\ \geq  \mathbb{E}[J_\text{rep}(\theta_k)] + \mathbb{E}[J_\text{ins}(\theta_k)] \\ - L ( \sigma_\text{rep}^2  \alpha_{1, k}^2 + \sigma_\text{ins}^2  \alpha_{2, k}^2) \\
- CC_\text{student} B_{\Theta} (\alpha_{1, k} + \frac{\alpha_{2, k}}{b_\text{min}} ) \mathbb{E}[\norm{\xi_k - \xi_*}]  \\ 
+ \mathbb{E}[\alpha_{1, k} \norm{\nabla J_\text{rep} (\theta_k)}^2   + \alpha_{2, k} \norm{\nabla J_\text{ins} (\theta_k)}^2  \\ + (\alpha_{1, k}+\alpha_{2, k}) \nabla J_\text{rep} (\theta_k)^\top \nabla J_\text{ins} (\theta_k).
\end{multline}
Rearrange the terms and get
\begin{multline}
 \mathbb{E}[\alpha_{1, k} \norm{\nabla J_\text{rep} (\theta_k)}^2   + \alpha_{2, k} \norm{\nabla J_\text{ins} (\theta_k)}^2  \\ + (\alpha_{1, k}+\alpha_{2, k}) \nabla J_\text{rep} (\theta_k)^\top \nabla J_\text{ins} (\theta_k)]\\
 \leq \mathbb{E}[J_\text{rep}(\theta_{k+1})] - \mathbb{E}[J_\text{rep}(\theta_k)] \\
 +\mathbb{E}[J_\text{ins}(\theta_{k+1})] - \mathbb{E}[J_\text{ins}(\theta_k)] \\
 + CC_\text{student} B_{\Theta} (\alpha_{1, k} + \frac{\alpha_{2, k}}{b_\text{min}} ) \mathbb{E}[\norm{\xi_k - \xi_*}] \\
 +L ( \sigma_\text{rep}^2  \alpha_{1, k}^2 + \sigma_\text{ins}^2  \alpha_{2, k}^2).
\end{multline}

Denote by LHS and RHS the left hand side and right hand side of the above equation.
Define $U_k = J(\theta^*) - J(\theta_k)$ for both rep and ins where $\theta^*$ is the optimal parameters. Then
\begin{multline}
\text{RHS} = \mathbb{E}[U_{k, \text{rep}}] - \mathbb{E}[U_{k+1, \text{rep}}] \\ 
 + \mathbb{E}[U_{k, \text{ins}}] - \mathbb{E}[U_{k+1, \text{ins}}] \\
 + CC_\text{student} B_{\Theta} (\alpha_{1, k} + \frac{\alpha_{2, k}}{b_\text{min}} ) \mathbb{E}[\norm{\xi_k - \xi_*}] \\
 +L ( \sigma_\text{rep}^2  \alpha_{1, k}^2 + \sigma_\text{ins}^2  \alpha_{2, k}^2).
\end{multline}
Let $A_k \alpha_{2, k}/\alpha_{1, k} $, then
\begin{multline}
\text{LHS} =  \alpha_{1, k}\mathbb{E} [ \norm{\nabla J_\text{rep} (\theta_k)}^2   + A_k \norm{\nabla J_\text{ins} (\theta_k)}^2  \\ + (1+A_k) \nabla J_\text{rep}^T (\theta_k) \nabla J_\text{ins} (\theta_k)],
\end{multline}
and
\begin{multline}
 \text{RHS} = \mathbb{E}[U_{k, \text{rep}}] - \mathbb{E}[U_{k+1, \text{rep}}] + \\ 
 \mathbb{E}[U_{k, \text{ins}}] - \mathbb{E}[U_{k+1, \text{ins}}] + \\
 + CC_\text{student} B_{\Theta} (\alpha_{1, k} + \frac{A_k\alpha_{1, k}}{b_\text{min}} ) \mathbb{E}[\norm{\xi_k - \xi_*}] \\
 +L ( \sigma_\text{rep}^2  \alpha_{1, k}^2 + \sigma_\text{ins}^2  A_k^2 \alpha_{1, k}^2).
\end{multline}

Divide both sides by $\alpha_{1, k}$ and take the sum over $\{k-N, \ldots, k\}$ for some integer $1 < N < k$. Then we have
\begin{multline}
\text{newLHS} = \sum_{j=k-N}^k \mathbb{E}\left[ \norm{\nabla J_\text{rep} (\theta_j)}^2   + A_k \norm{\nabla J_\text{ins} (\theta_j)}^2 \right.  \\ \left. +  (1+A_k) \nabla J_\text{rep} (\theta_j)^\top \nabla J_\text{ins} (\theta_j) \right],
\end{multline}
and
\begin{multline}
 \text{newRHS} = \sum_{j=k-N}^k \frac{1}{\alpha_{1, j}} (\mathbb{E}[U_{j, \text{rep}}] - \mathbb{E}[U_{j+1, \text{rep}}]) \\ 
 + \sum_{j=k-N}^k \frac{1}{\alpha_{1, j}} (\mathbb{E}[U_{j, \text{ins}}] - \mathbb{E}[U_{j+1, \text{ins}}]) \\
 + CC_\text{student} B_{\Theta} (1 + \frac{A_k}{b_\text{min}} ) \sum_{j=k-N}^k \mathbb{E}[\norm{\xi_j - \xi_*}] \\
 + L ( \sigma_\text{rep}^2+\sigma_\text{ins}^2  A_k^2) \sum_{j=k-N}^k \alpha_{1, j}.
\end{multline}

Rearrange the first two terms and get
\begin{multline}
 \text{newRHS} = \sum_{j=k-N}^k (\frac{1}{\alpha_{1, j}} - \frac{1}{\alpha_{1, j-1}}) \mathbb{E}[U_{j, \text{rep}}]\\  - \frac{1}{\alpha_{1, k}} \mathbb{E}[U_{k+1, \text{rep}}] + \frac{1}{\alpha_{1, k-N-1}} \mathbb{E}[U_{k-N, \text{rep}}] \\ 
 + \sum_{j=k-N}^k (\frac{1}{\alpha_{1, j}} - \frac{1}{\alpha_{1, j-1}}) \mathbb{E}[U_{j, \text{ins}}]  \\ - \frac{1}{\alpha_{1, k}} \mathbb{E}[U_{k+1, \text{ins}}] + \frac{1}{\alpha_{1, k-N-1}} \mathbb{E}[U_{k-N, \text{ins}}] \\
 + CC_\text{student} B_{\Theta} (1 + \frac{A_k}{b_\text{min}} ) \sum_{j=k-N}^k \mathbb{E}[\norm{\xi_j - \xi_*}] \\
 + L ( \sigma_\text{rep}^2+\sigma_\text{ins}^2  A_k^2) \sum_{j=k-N}^k \alpha_{1, j}.
\end{multline}

Since $\mathbb{E}[U_{k+1, \text{rep}}] \ge 0$ and $\mathbb{E}[U_{k+1, \text{ins}}] \ge 0$, we have
\begin{multline}
 \text{newRHS} \le \sum_{j=k-N}^k (\frac{1}{\alpha_{1, j}} - \frac{1}{\alpha_{1, j-1}}) \mathbb{E}[U_{j, \text{rep}}]\\  + \frac{1}{\alpha_{1, k-N-1}} \mathbb{E}[U_{k-N, \text{rep}}] \\ 
 + \sum_{j=k-N}^k (\frac{1}{\alpha_{1, j}} - \frac{1}{\alpha_{1, j-1}}) \mathbb{E}[U_{j, \text{ins}}]  \\  + \frac{1}{\alpha_{1, k-N-1}} \mathbb{E}[U_{k-N, \text{ins}}] \\
 + CC_\text{student} B_{\Theta} (1 + \frac{A_k}{b_\text{min}} ) \sum_{j=k-N}^k \mathbb{E}[\norm{\xi_j - \xi_*}] \\
 + L ( \sigma_\text{rep}^2+\sigma_\text{ins}^2  A_k^2) \sum_{j=k-N}^k \alpha_{1, j}.
\end{multline}

Since we have rewards and score functions bounded above, we can find two positive constants $D_{rep}$ and $D_{ins}$, such that $U_{k, \text{rep}} \leq D_\text{rep}$ and $U_{k, \text{ins}} \leq D_\text{ins}$ for all $k$. Substituting for these bounds and get
\begin{multline}
 \text{newRHS} \leq \sum_{j=k-N}^k (\frac{1}{\alpha_{1, j}} - \frac{1}{\alpha_{1, j-1}}) D_\text{rep} + \frac{1}{\alpha_{1, k-N-1}}D_\text{rep} \\ 
 + \sum_{j=k-N}^k (\frac{1}{\alpha_{1, j}} - \frac{1}{\alpha_{1, j-1}}) D_\text{ins} + \frac{1}{\alpha_{1, k-N-1}}D_\text{ins} \\
 + CC_\text{student} B_{\Theta} (1 + \frac{A_k}{b_\text{min}} ) \sum_{j=k-N}^k \mathbb{E}[\norm{\xi_j - \xi_*}] \\
 + L ( \sigma_\text{rep}^2+\sigma_\text{ins}^2  A_k^2) \sum_{j=k-N}^k \alpha_{1, j}.
\end{multline}
Unravel the telescoping sums:
\begin{multline}
 \text{newRHS} \leq \frac{D_\text{rep} + D_\text{ins}}{\alpha_{1, k}}  \\
 + CC_\text{student} B_{\Theta} (1 + \frac{A_k}{b_\text{min}} ) \sum_{j=k-N}^k \mathbb{E}[\norm{\xi_j - \xi_*}] \\
 + L ( \sigma_\text{rep}^2+\sigma_\text{ins}^2  A_k^2) \sum_{j=k-N}^k \alpha_{1, j}.
\end{multline}

Plug in the assumption that $\alpha_{1, k} = k^{-a}$, and the convergence rate of the critic (replace $\xi_k$ by $\xi_{T_C(k)}$):
\begin{multline}
 \text{newRHS} \leq (D_\text{rep} + D_\text{ins}) k^a \\ 
 + CC_\text{student} B_{\Theta} (1 + \frac{A_k}{b_\text{min}} ) \sum_{j=k-N}^k L_1 T_C(j)^{-b} \\
 + L ( \sigma_\text{rep}^2+\sigma_\text{ins}^2  A_k^2) \sum_{j=k-N}^k j^{-a}.
\end{multline}

Let
\begin{multline}
    \mathcal{F}(\theta_j) := \norm{\nabla J_\text{rep} (\theta_j)}^2   + A_k \norm{\nabla J_\text{ins} (\theta_j)}^2 \\ +  (1+A_k)\nabla J_\text{rep} (\theta_j)^\top \nabla J_\text{ins} (\theta_j).
\end{multline}

When $b\in(0,1)$, we set $T_C(k)=k$. Since $\sum_{j=k-N}^k j^{-a} \le (k^{1-a} - (k-N-1)^{1-a})/(1-a)$, we have
\begin{multline}
\text{newLHS} \le (D_\text{rep} + D_\text{ins}) k^a \\ 
+ CC_\text{student} B_{\Theta} (1 + \frac{A_k}{b_\text{min}} ) \frac{L_1}{1-b}(k^{1-b} - (k-N-1)^{1-b}) \\
+ ( \sigma_\text{rep}^2+\sigma_\text{ins}^2  A_k^2)\frac{L}{1-a} (k^{1-a} - (k-N-1)^{1-a}).
\end{multline}

Divide both sides by $k$ and set $N=k-1$,
\begin{multline}
\frac{1}{k}\sum_{j=1}^k \mathbb{E}[\mathcal{F}(\theta_j)] \leq (D_\text{rep} + D_\text{ins}) k^{a-1}  \\
+\frac{L ( \sigma_\text{rep}^2  + \sigma_\text{ins}^2  A_k^2)}{1 - a} k^{-a} \\
+ CC_\text{student} B_{\Theta} (1 + \frac{A_k}{b_\text{min}} )\frac{L_1}{1-b} k^{-b}.
\end{multline}
By definition of $K_{\epsilon}$
we have that $ \mathbb{E}[\mathcal{F}(\theta_j)] > \epsilon$ for $j=1, \ldots, K_{\epsilon}$ so
\begin{multline}
\epsilon \leq \frac{1}{K_{\epsilon}}\sum_{j=1}^{K_{\epsilon}} \mathbb{E}[{\mathcal F}(\theta_j)] \leq 
\mathcal{O} (K_{\epsilon} ^{a-1} + K_{\epsilon}^{-a} + K_{\epsilon}^{-b})
\end{multline}
Defining $l = \min\{a, 1-a, b\}$ and inverting, we have
\begin{equation}
    K_{\epsilon} \leq \mathcal{O} (\epsilon^{-1/l})
\end{equation}

When $b=1$, we set $T_C(k)=k+1$. Similarly, we can get
\begin{multline}
\text{newLHS} \le (D_\text{rep} + D_\text{ins}) k^a \\ 
+ CC_\text{student} B_{\Theta} (1 + \frac{A_k}{b_\text{min}} ) L_1 (\log(k+1) - \log(k-N)) \\
+ ( \sigma_\text{rep}^2+\sigma_\text{ins}^2  A_k^2)\frac{L}{1-a} (k^{1-a} - (k-N-1)^{1-a}).
\end{multline}
Divide both sides by $k$ and set $N=k-1$,
\begin{multline}
\frac{1}{k}\sum_{j=1}^k \mathbb{E}[\mathcal{F}(\theta_j)] \leq (D_\text{rep} + D_\text{ins}) k^{a-1}  \\
+\frac{L ( \sigma_\text{rep}^2  + \sigma_\text{ins}^2  A_k^2)}{1 - a} k^{-a} \\
+ CC_\text{student} B_{\Theta} (1 + \frac{A_k}{b_\text{min}} )L_1 \frac{\log(k+1)}{k}.
\end{multline}

Again, we can get

\begin{multline}
\epsilon \leq \frac{1}{K_{\epsilon}}\sum_{j=1}^{K_{\epsilon}} \mathbb{E}[{\mathcal F}(\theta_j)] \leq \\
\mathcal{O} (K_{\epsilon} ^{a-1} + K_{\epsilon}^{-a} + \frac{\log(K_{\epsilon} + 1)}{K_{\epsilon}}).
\end{multline}

Optimizing over $a$, we can get $\epsilon \le \mathcal{O}(K_\epsilon^{-\frac{1}{2}})$ when $b>\frac{1}{2}$, and $\epsilon \le \mathcal{O}(K_\epsilon^{-b})$ when $b\le\frac{1}{2}$.
\end{proof}

\subsection{Extension: Adding the Cross-entropy Term}
In this section, we want to demonstrate that when the auxiliary cross-entropy is involved in $J_\text{rep}$, our results do not change.


By definition:
\begin{multline}
    \nabla J_\text{aux}(\theta) = \nabla H (\pi_\text{teacher} || \pi_{\theta})  = - \mathbb{E}_{\pi_\text{teacher}} [\nabla_{\theta} \log \pi_{\theta} (a|s)] \\
    =- \mathbb{E}_{\pi_\theta} \left[ \frac{\pi_\text{teacher}(a| s)}{\pi_{\theta}(a|s)}\nabla_{\theta} \log \pi_{\theta} (a|s)\right].
\end{multline}

We now make an extra assumption on the student policy $\pi_\theta$.
\begin{assumption}\label{assmp:limitonstudentpolicy}
The student policy has a minimum positive value ${d}_\text{min} \in (0, 1]: \pi_\theta (a|s) \geq d_\text{min}$ for all $(s, a) \in {\mathcal S} \times {\mathcal A}$.
\end{assumption}


We now have an unbiased estimate for $\nabla J_\text{rep}$ with cross-entropy regularization
\begin{multline}
    \hat{\nabla} J_\text{rep}(\theta) = \hat{Q}^{\pi_{\theta}} (s_T, a_T) \nabla \log \pi_{\theta} (a_T|s_T) \\
    + \beta_k \frac{\pi_\text{teacher}(a_T| s_T)}{\pi_{\theta}(a_T|s_T)} \nabla_{\theta} \log \pi_{\theta} (a_T|s_T). 
\end{multline}
where $s_T, a_T$ is the state-action pair collected following the student policy.

Then the actor update obeys 
\begin{eqnarray}
& & \theta_{k+1} - \theta_k = \alpha_{1, k} \hat{Q}^{\pi_{\theta}} (s_{T_k}, a_{T_k}) \nabla \log \pi_{\theta} (a_{T_k}|s_{T_k})   \nonumber \\
& & + \alpha_{1, k} \beta_k  \frac{\pi_\text{teacher}(a_{T_k}| s_{T_k})}{\pi_{\theta}(a_{T_k}|s_{T_k})} \nabla_{\theta} \log \pi_{\theta}(a_{T_k}|s_{T_k}) \nonumber \\
& & + \alpha_{2, k} \frac{\pi_{\theta} (\tilde{a}_{T_k}|\tilde{s}_{T_k})}{\pi_\text{teacher}(\tilde{a}_{T_k}|\tilde{s}_{T_k})} \hat{Q}_{\pi_{\theta}} (\tilde{s}_{T_k}, \tilde{a}_{T_k}) \nabla \log \pi_{\theta} (\tilde{a}_{T_k}|\tilde{s}_{T_k}).   \nonumber \\
\end{eqnarray}

This results in having an additional term in $Z_\text{rep}^k(\theta)$ in the proof of Lemma~\ref{lemma:Wk}.
\begin{multline}
    Z_\text{rep}^k(\theta) = \alpha_{1, k} (\xi_k^T - \xi_*) \varphi (s_{T_k}, a_{T_k}) \nabla \log \pi_{\theta} (a_{T_k}|s_{T_k}) \\ + \alpha_{1, k} \beta_k  \frac{\pi_\text{teacher}(a_{T_k}| s_{T_k})}{\pi_{\theta}(a_{T_k}|s_{T_k})} \nabla_{\theta} \log \pi_{\theta}(a_{T_k}|s_{T_k}).     
\end{multline}

Since
\begin{equation}
    \norm{\frac{\pi_\text{teacher}}{\pi_{\theta}}} . \norm{\nabla \log \pi_{\theta} (a_{T_k}|s_{T_k})} \leq \frac{B_{\Theta}}{d_\text{min}},
\end{equation}
it will only change the parameter in the proof of Lemma~\ref{lemma:Wk} but not affect the final conclusion.

\end{document}